%% file: gauss_icml2026.tex
\documentclass{article}

\usepackage{microtype}
\usepackage{graphicx}
\usepackage{subcaption}
\usepackage{booktabs}
\usepackage{hyperref}
\hypersetup{colorlinks=true,allcolors=blue}

\usepackage[accepted]{icml2026}

\usepackage{amsmath}
\usepackage{amssymb}
\usepackage{mathtools}
\usepackage{amsthm}
\usepackage{bm}

\usepackage[utf8]{inputenc}
\usepackage[T1]{fontenc}
\usepackage{url}
\usepackage{amsfonts}
\usepackage[table,dvipsnames]{xcolor}  
\usepackage{array}  

\usepackage{tikz}
\usepackage{pgfplots}
\pgfplotsset{compat=1.18}
\usetikzlibrary{decorations.pathreplacing}
\usetikzlibrary{decorations.pathmorphing}
\usetikzlibrary{bayesnet}
\usetikzlibrary{calc}
\usetikzlibrary{positioning}
\usetikzlibrary{fit}

\input{math_commands.tex}

\input{Figures/tikz_plots}

\usepackage[most]{tcolorbox}
\usepackage{etoc}
\usepackage{mdframed}
\usepackage{physics}
\usepackage{pifont}
\usepackage{makecell}
\usepackage{tabularx}
\usepackage{float}
\usepackage{empheq}
\usepackage{wrapfig}
\usepackage{enumitem}
\usepackage{listings}
\usepackage{multirow}
\usepackage{arydshln}
\usepackage{pgfmath}
\usepackage{xspace}
\usepackage{listings}
\usepackage{verbatim}
\usepackage{algorithm}
\usepackage{algorithmic}

\newif\ifincludecomment
\includecommenttrue 

\ifincludecomment
  
\else
  \NewEnviron{guidance}{}{}
\fi
\usepackage[colorinlistoftodos,textsize=tiny, textwidth=18mm]{todonotes}
\ifincludecomment
\newcommand{\maybecomment}[1]{\todo[color=olive!40]{#1}} 
\newcommand{\maybetohere}[1]{\todo[color=red!40]{#1}} 
\newcommand{\maybedelete}[1]{\todo[color=blue!40]{#1}} 

\else
  \newcommand{\maybecomment}[1]{}
  
\newcommand{\maybedelete}[1]{} 
\fi
\newcommand{\amostohere}[1]{{\color{black}\maybetohere{AMOS HERE}}}

\usepackage{array}

\definecolor{HeaderBlue}{HTML}{1F4E79}
\definecolor{BoxBg}{HTML}{F7F9FC}
\definecolor{UncTint}{HTML}{FFF3E0}   
\definecolor{UncText}{HTML}{8A3B12}
\definecolor{SoftGray}{HTML}{6B6B6B}
\definecolor{UncBlue}{HTML}{0066CC}   

\usepackage[capitalize,noabbrev]{cleveref}
\crefname{section}{Section}{Sections}
\crefname{algorithm}{Algorithm}{Algorithms}
\crefname{appendix}{Appendix}{Appendices}
\crefname{definition}{Definition}{Definitions}
\crefname{table}{Table}{Tables}

\DeclareMathOperator*{\logsumexp}{logsumexp}
\newcommand{\unc}[1]{\textcolor{UncBlue}{#1}}
\definecolor{ForgetOrange}{HTML}{E07B00}   
\definecolor{WriteGreen}{HTML}{1B7837}     
\newcommand{\fg}[1]{\textcolor{ForgetOrange}{#1}}
\newcommand{\wg}[1]{\textcolor{WriteGreen}{#1}}
\newcommand{\kl}[1]{\textcolor{UncBlue}{#1}}

\usepackage[textsize=tiny]{todonotes}

\newtheorem{theorem}{Theorem}

\newtheorem*{remark}{Remark}

\usepackage{thm-restate}

\definecolor{heatpurple}{RGB}{186, 85, 211}
\definecolor{darkgreen}{rgb}{0.0, 0.5, 0.0}
\definecolor{codegreen}{rgb}{0,0.6,0}
\definecolor{codegray}{rgb}{0.5,0.5,0.5}
\definecolor{codepurple}{rgb}{0.58,0,0.82}
\definecolor{backcolour}{rgb}{0.95,0.95,0.92}
\definecolor{forgetred}{RGB}{220, 20, 60}

\lstdefinestyle{python}{
    backgroundcolor=\color{backcolour},
    commentstyle=\color{codegreen},
    keywordstyle=\color{blue},
    numberstyle=\tiny\color{codegray},
    stringstyle=\color{codepurple},
    basicstyle=\ttfamily\scriptsize,
    breaklines=true,
    numbers=left,
    numbersep=5pt,
    showstringspaces=false,
    frame=single,
    language=Python
}
\newcommand{\best}[1]{\textbf{#1}}
\newcommand{\second}[1]{\underline{#1}}

\newcommand{\gmark}{\textcolor{green!60!black}{\ding{51}}}
\newcommand{\rmark}{\textcolor{red!70!black}{\ding{55}}}

\newcommand{\our}{\textsc{KLA}\xspace}

\newcommand{\shape}[1]{\hfill $\triangleright$ \small{\textcolor{gray}{#1}}}

\icmltitlerunning{Parallel, Scalable and Efficient Bayesian Filters for Language Modelling}

\begin{document}

\twocolumn[
\icmltitle{Kalman Linear Attention: Parallel Bayesian Filtering For Efficient Language Modelling and State Tracking}

\begin{icmlauthorlist}
\icmlauthor{Vaisakh Shaj}{}
\icmlauthor{Cameron Barker}{}
\icmlauthor{Aidan Scannell}{}
\icmlauthor{Andras Szecsenyi}{}
\icmlauthor{Elliot J. Crowley}{}
\icmlauthor{Amos Storkey}{}
\end{icmlauthorlist}

\vspace{0.1in}
\centerline{\textit{University of Edinburgh, United Kingdom}}

\icmlcorrespondingauthor{Vaisakh Shaj}{vaisakhs.shaj@gmail.com}

\icmlkeywords{State Space Models, Kalman Filtering, Language Modelling, Uncertainty Quantification}

\vskip 0.3in
]

\printAffiliationsAndNotice{}

\begin{abstract}
State-space language models such as Mamba and gated linear attention (GLA) offer linear-complexity, parallelisable alternatives to transformers, but their \emph{linear} state updates limit expressivity and robust state tracking. We close this gap from a probabilistic angle, casting sequence mixing as exact Bayesian filtering with the Kalman filter as the core primitive. Classical Kalman filters give principled state and uncertainty estimates but are viewed as inherently sequential; we show that reparameterising them in \emph{information form} turns their updates into an associative scan - so the per-token recurrent update is \emph{non-linear} (a M\"obius/precision recursion) yet remains \emph{temporally parallel}. The resulting Kalman Linear Attention (KLA) layer is a drop-in sequence mixer that performs time-parallel probabilistic inference, carries an explicit belief-state uncertainty, and is strictly more expressive than GLA-style linear updates at the same computational cost. This expressivity translates directly into stronger state tracking: KLA solves permutation-composition ($A_5$) tasks that linear SSMs and attention cannot, while staying scan-parallel. As a drop-in primitive it also matches or improves on modern SSMs and GLAs across synthetic token-manipulation and zero-shot commonsense benchmarks, and is among the first stacked Bayesian-filtering primitives trained at the billion-token scale.
\end{abstract}

\begin{figure*}[t]
  \centering
  \includegraphics[width=0.5\textwidth]{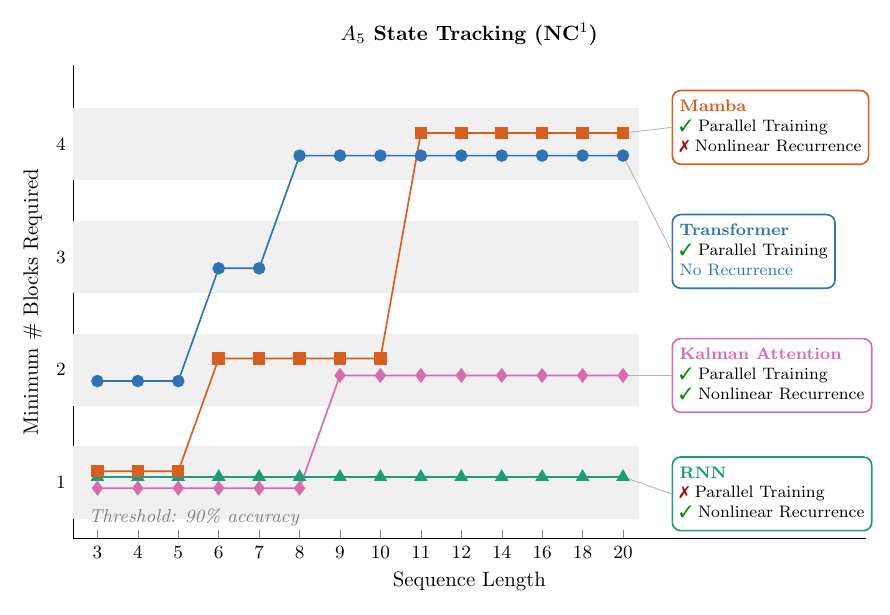}\hfill
  \includegraphics[width=0.47\textwidth]{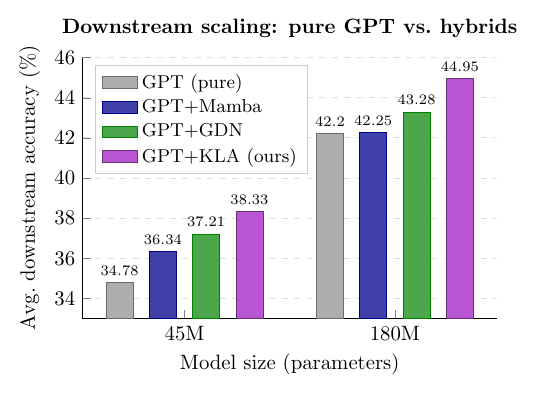}
  \caption{\textbf{(a) Expressivity.} Minimum number of blocks required to solve the $A_5$ (alternating group on 5 elements) permutation-composition task~\citep{merrill2024illusion}. \our's fractional-linear (M\"obius) updates fall between a fully nonlinear RNN and linear SSMs/transformers, solving the task at constant depth (1--2 blocks) where linear SSMs/attention require depth growing with sequence length, while remaining parallel-trainable.
  \textbf{(b) Downstream scaling.} Zero-shot accuracy (averaged over eight commonsense benchmarks) of softmax-attention GPT vs.\ GPT$+$SSM hybrids at 45M and 180M parameters. Replacing a single final attention layer with \our (GPT$+$\our) gives the strongest complement to attention.}
  \label{fig:a5-min-layers}
  \vspace{-0.1in}
\end{figure*}

\section{Introduction}
\label{sec:Intro}

\begin{table}[t]
\centering
\caption{High-level comparison of sequence-mixing primitives. \our combines the efficiency of SSMs with nonlinear (M\"obius) updates and explicit belief-state uncertainty.}
\label{tab:comparison}
\footnotesize
\setlength{\tabcolsep}{5pt}
\renewcommand{\arraystretch}{1.1}
\resizebox{\columnwidth}{!}{%
\begin{tabular}{@{}lccc@{}}
\toprule
& \textbf{Softmax Attention} & \textbf{SSMs / GLA} & \textbf{\our} \\
\midrule
Expressivity   & Nonlinear & Linear & Fractional lin.\ (M\"obius) \\
Training eff.  & $\mathcal{O}(T^2)$ & $\mathcal{O}(T)$ & $\mathcal{O}(T)$ \\
Inference eff. & $\mathcal{O}(T)$ & $\mathcal{O}(1)$ & $\mathcal{O}(1)$ \\
Seq.\ uncertainty & \rmark & \rmark & \gmark \\
Parallel training & \gmark & \gmark & \gmark \\
\bottomrule
\end{tabular}%
}
\vspace{-1.0em}
\end{table}

Scaling sequence modelling beyond quadratic attention~\citep{vaswani2017attention} is a central goal in large language and foundation model research.\footnote{Code: \href{https://github.com/vaisakh-shaj/kalman-linear-attention}{github.com/vaisakh-shaj/\allowbreak kalman-linear-attention}}
This has motivated interest in RNN-like architectures that support linear-time sequence modelling with efficient parallel training.
Recent state-space models (SSMs), including S4/S5~\citep{gu2021efficiently,smith2022simplified}, Mamba~\citep{gu2023mamba,dao2024transformers}, and their successors, 
achieve this with $\mathcal{O}(\log T)$ depth, $\mathcal{O}(T)$ work, and linear or sublinear memory. 
This is essential for handling long contexts, on-device deployment and energy efficiency.\looseness-1

Despite their efficiency, adoption of SSMs is not as widespread as attention-based transformers.
Recent work reinterprets modern SSMs and linear RNNs, including Mamba, as gated variants of linear attention~(GLA)~\citep{yang2023gated}, where performance largely depends on how gates are defined. In all of these, the underlying hidden state updates are linear or affine, which limits their expressivity compared to softmax attention, whose normalisation induces nonlinear interactions between tokens.
Furthermore, none of these represent state uncertainty explicitly. This contrasts with longstanding probabilistic linear state-space formalisms whose Kalman-filter updates carry explicit state uncertainty and update it through nonlinear precision recursions.

In this paper, we ask: {\em Can we overcome the linear update constraints common to current GLA models to develop a state-space block that efficiently implements exact Kalman filter updates?}

We introduce \textbf{Kalman Linear Attention (\our)}, which formulates sequence modelling as a Bayesian filtering problem. \our models two sources of uncertainty: process noise, which captures uncertainty in state evolution, and observation noise, which captures uncertainty in the information provided by each input token.
Crucially, this uncertainty is not merely an output: it directly controls how new information is gated in state spaces.\looseness-1

At first glance, such Bayesian updates seem inherently sequential. Our key insight is that the information form of the Kalman filter admits a fractional-linear / M\"obius associative structure that can be implemented using parallel scan-based algorithms, \emph{despite its nonlinear update computation}.
This allows Kalman-style updates to achieve the same $\mathcal{O}(\log T)$ parallel depth as models such as Mamba. The result is the \textbf{\our} layer: a drop-in replacement for standard SSM or attention layers.\looseness-1

It is not enough to implement a new primitive without establishing that the extra expressiveness provides a benefit. We study the suitability of this continuous-state probabilistic model for discrete language modelling both theoretically and empirically, demonstrating important gains. 

\paragraph{Contributions.} Our contributions are as follows:
\begin{itemize}
    \item[\textbf{C1}] \textbf{Associative reparameterisation of Kalman filtering:} we reparameterise the diagonal linear-Gaussian filter in information form, showing the precision recursion is a M\"obius (fractional-linear) map that composes associatively, enabling parallel prefix scans.
    \item[\textbf{C2}] \textbf{Nonlinear gating from uncertainty:} the precision-ratio gates are history-dependent and nonlinear, going beyond linear/affine SSM and gated linear-attention updates while preserving linear-time scan structure.
    \item[\textbf{C3}] \textbf{Kalman Linear Attention layer:} we introduce \textsc{KLA}, a drop-in sequence mixer for modern language-modelling pipelines that produces explicit belief-state uncertainty.
    \item[\textbf{C4}] \textbf{Scaling and empirical validation:} our scan-based implementation scales efficiently with sequence length and matches or outperforms modern SSMs and GLAs on state-tracking, associative recall, and zero-shot commonsense reasoning, and is among the first stacked Bayesian filters pretrained at the billion-token scale.\looseness-1
\end{itemize}
\paragraph{Conflict of Interest Disclosure.}
The authors declare no financial or other conflicts of interest; this is purely academic work.

\section{Related Work}

\paragraph{Attention as filtering.}
The idea that intelligent systems must filter information predates modern self-attention: in cognitive science, early theories framed attention as a mechanism that suppresses irrelevant signals and routes task-relevant ones under limited processing capacity
\citep{broadbent1957mechanical,treisman1969strategies,deutsch1963attention, nielsen2009statistical}.
The core theme across these views is that attention allocates scarce processing by selecting what matters in context.\looseness-2

In modern language models, transformer self-attention can be viewed as a form of context-based filtering \citep{vaswani2017attention}.
Tokens compete via similarity and the softmax suppresses most inputs whilst allowing a small subset to dominate the update.
But this selection is \emph{exemplar-based}: it requires retaining and comparing against all past keys/values rather than compressing them into a fixed-size state, incurring $\mathcal{O}(T^2)$ cost at long context lengths.

\paragraph{Bayesian filtering and precision-control as attention.}
Bayesian filters perform filtering as posterior inference: prior beliefs are updated with new observations weighted by their precision (inverse variance), so more reliable observations exert greater influence whilst unreliable evidence is downweighted.
This yields a principled, interpretable gating mechanism where selection follows from uncertainty propagation.

Precision-weighted prediction errors have also been proposed as a computational account of attention in neuroscience and predictive coding \citep{rao1999predictive,feldman2010attention}.
In this view, attention modulation is based upon
confidence in sensory signals as opposed to deterministic reweighting.
We adopt this view and show that Bayesian filters can be integrated into modern sequence modelling.

\paragraph{Sequence models as test-time in-context online learners.}
Concurrently, recent sub-quadratic mixers have been unified as online learners that solve a per-token regularised regression in closed form, a view first articulated for SSMs by Longhorn~\citep{liu2025longhorn,wang2025test} and extended in linear-attention and delta-rule variants~\citep{yang2023gated,yang2024gated,hatamizadeh2026gated}; we situate \our within this view in \cref{sec:parallel_inference} (with the full template in \cref{sec:online-learning}), where its posterior-mean update emerges as a gated RNN update with an adaptive learning rate from filtering. The closest layers to \our, MesaNet~\citep{von2025mesanet} and Gated KalmaNet~\citep{peng2026gated}, solve a regularised least-squares objective much like ours, but (i) assume a static, deterministic latent state - no transition dynamics, unlike a Bayesian filter - and (ii) require expensive iterative solvers (conjugate gradient and Chebyshev, respectively). \our instead retains a transition and process noise, giving a closed-form, scan-parallel fractional-linear (M\"obius) posterior-mean update (\Cref{Th:Mobius,thm:eta-affine}) that needs neither a steady-state assumption nor an inner solver. More broadly, \our shows that exact Bayesian filtering is itself scan-parallelisable, with no approximations, steady-state assumptions, or test-time solvers. Unlike recent mixers, it moreover admits a strictly more expressive M\"obius update, which translates into measurable gains on state tracking and common-sense reasoning (\cref{fig:a5-min-layers,tab:lm-main}).

\paragraph{Temporal parallelisation of Bayesian filters.}
\citet{sarkka2020temporal} pioneered temporal parallelisation of Bayesian filtering and smoothing using prefix-sum algorithms, requiring specialised parameterisations to ensure associativity.
To achieve associativity, they construct it by lifting each filtering step into a 5-tuple augmented representation, in the linear case. We show that no such augmentation is needed: Kalman updates in the information form are already M\"{o}bius maps, so precision composes by $2 \times 2$ matrix multiplication, a minimal, GPU-friendly structure (Theorems~\ref{Th:Mobius} and \ref{thm:eta-affine}). Beyond parallelisation, we integrate these probabilistic primitives into learnable neural sequence models for language, where token-dependent observation likelihoods and dynamics are learned from data rather than specified a priori.

\paragraph{Deep Kalman filters and probabilistic state-space models.}
Kalman filtering \citep{kalman1960new} has a long history in RL, control, and world modelling \citep{haarnoja2016rl,watter2015e2c,hafner2019planet,hafner2020dreamer,shaj2023multi}. Deep Kalman filters embed linear-Gaussian SSMs in neural networks, trained by exact or variational inference \citep{krishnan2015deepkf,krishnan2017structured,karl2016dsvae,fraccaro2017rssm,becker2019rkn,shaj2021action}, but their recurrent update limits scalability to long contexts. Recent deep-RL work pairs deterministic SSM backbones with Kalman-style dynamics via S\"{a}rkk\"{a}'s temporal parallelisation \citep{becker2024kalmamba}, or replaces them with parallelised Kalman layers for uncertainty under partial observability \citep{luis2024uncertainty}. In contrast, we adopt \textsc{\our} as a \emph{standalone}, closed-form probabilistic primitive for language modelling, competing directly with modern SSM/GLA mixers without an extra deterministic scaffold. Crucially, whereas these prior approaches typically employ a single Kalman layer, we hierarchically stack \our across many layers, training deep multi-layer Bayesian-filtering architectures for language modelling at the billion-token scale.\looseness-1

\section{Background and Preliminaries}
\label{sec:background}

Modern language models can be abstracted as the composition of two functions~\citep{alonso2025state}: a shared
representation (compression) map $f(\cdot;\theta)$ and a task-specific head $g$. Given a
prompt $x_{1:T}$ of sequence length $T$, the backbone (the pretrained model) produces a compact latent
representation $z_{1:T}=f(x_{1:T};\theta)$ that captures structure in the language; the
head then maps $z_{1:T}$ to outputs $y_{1:T}$ (e.g., next-token prediction, tagging,
regression).
Concrete parameterisations for $f$ balance expressivity with computational tractability: transformers use softmax attention as the core primitive, whereas state-space models (SSMs) such as Mamba adopt structured recurrence with linear complexity.

\paragraph{Deterministic state-space models (SSMs) and Mamba.}
Popular SSMs such as Mamba originate from structured
state-space models~\citep{gu2021efficiently}.
Building on their diagonal simplification~\citep{gupta2022simplifying},
a linear SSM maps a scalar input signal $x(t)$ to an output $y(t)$
through a latent state $h(t)\in\mathbb{R}^{N}$:
\begin{equation}
h'(t)=\mathbf{A}h(t)+\mathbf{B}x(t),\qquad
y(t)=\mathbf{C}h(t),
\label{eq:ssm-ct}
\end{equation}
with $x(t),y(t)\!\in\!\mathbb{R}$,
$\mathbf{A}=\mathrm{diag}(a_1,\dots,a_N)$, $\mathbf{B}\in\mathbb{R}^{N}$,
and $\mathbf{C}\in\mathbb{R}^{N}$.
After discretisation, one processes a
$D$-dimensional sequence $\mathbf{X}\in\mathbb{R}^{T\times D}$ by running $D$
independent copies of this recurrence, one per feature coordinate, or
channel, each with its own scalar input $x(t)$.

\textbf{Mamba (S6)}~\cite{gu2023mamba} makes this recurrence
input-dependent (selective).
For a single channel the update is
\begin{equation}
\mathbf{h}_t
  =\bar{\mathbf{a}}_t\odot\mathbf{h}_{t-1}
  +\bar{\mathbf{b}}_t\, x_t,
\qquad
y_t=\mathbf{c}_t^{\top}\mathbf{h}_t,
\label{eq:mamba-rec}
\end{equation}
where $x_t,y_t\!\in\!\mathbb{R}$; this system can be applied
independently to each feature channel of a $D$-dimensional sequence.
Here $\mathbf{h}_t\!\in\!\mathbb{R}^{N}$ is the latent state,
$\bar{\mathbf{a}}_t,\bar{\mathbf{b}}_t\!\in\!\mathbb{R}^{N}$ are discretised
diagonal dynamics, and
$\mathbf{c}_t\!\in\!\mathbb{R}^{N}$ is an input-dependent readout
(the selective counterpart of
$\mathbf{C}$ in \cref{eq:ssm-ct}).
The key property is \emph{selectivity}: all parameters
$(\bar{\mathbf{a}}_t,\bar{\mathbf{b}}_t,\mathbf{c}_t)$ are computed from the
current input $x_t$, allowing the model to selectively filter which
information to retain or discard at each step, while preserving
linear-time scanning.
This input-dependent gating can be seen as a learned form of
filtering; in \cref{par:ssm} we formalise this intuition through
Bayesian filtering with explicit uncertainty.

\paragraph{State expansion.}
The dimension $N$ in \cref{eq:mamba-rec} is a design choice known as
\emph{state expansion}~\cite{gupta2022simplifying,poli2024mad,yu2025block}.
With $N{=}1$, each feature channel maintains a single scalar recurrence and the
state is $\mathbf{h}_t\in\mathbb{R}^{D}$.
Setting $N{>}1$ equips every channel with $N$ parallel memory slots, expanding
the state to $\mathbf{H}_t\in\mathbb{R}^{N\times D}$.
Stacking the per-channel dynamics accordingly, we write
$\bar{\mathbf{A}}_t,\bar{\mathbf{B}}_t\!\in\!\mathbb{R}^{N\times D}$
and $\mathbf{C}_t\!\in\!\mathbb{R}^{1\times N}$.
How the dynamics are allocated across these slots is itself an algorithmic
decision: $\bar{\mathbf{A}}_t$ and $\bar{\mathbf{B}}_t$ may be \emph{distinct}
per slot or \emph{shared} (tied) across channels, depending on the model variant.
State expansion trades a larger memory footprint ($N{\times}D$ vs.\ $D$) for
richer per-channel recurrence histories, as each slot can capture a different
effective time-scale.

\paragraph{Attention from the SSM perspective.}
Ignoring scaling constants, auto-regressive softmax self-attention computes
\begin{align}
\mathbf{Y} = \mathrm{softmax}(\mathbf{Q}\mathbf{K}^\top + \mathbf{M})\mathbf{V},
\end{align}
where $\mathbf{Q},\mathbf{K},\mathbf{V}=\mathbf{X}\mathbf{W}_q,\mathbf{X}\mathbf{W}_k,\mathbf{X}\mathbf{W}_v$ for $\mathbf{X},\mathbf{Q},\mathbf{K},\mathbf{V}\in\mathbb{R}^{T\times D}$, $\mathbf{W}_q, \mathbf{W}_k, \mathbf{W}_v \in \mathbb{R}^{D \times D}$ and $\mathbf{M} \in (- \infty, 1)^{T \times T}$ is an auto-regressive mask.

Unlike the linear recurrence in \cref{eq:mamba-rec}, softmax attention maintains a distinct memory of each $\mathbf{K}_t$ and $\mathbf{V}_t$, yielding $\mathcal{O}(T^2)$ time and memory complexity.   
\looseness-1

By removing softmax we obtain Linear Attention
\\
$ \mathbf{y}_t = \mathbf{q}_t \sum_{i=0}^t  \mathbf{k}_i^\top \mathbf{v}_i$
which permits an alternate recurrent form~\cite{katharopoulos2020transformers} that can be calculated in $\mathcal{O}(T)$ time and $\mathcal{O}(1)$ memory complexity, given by
\begin{equation}
\mathbf{H}_t=\mathbf{G}_t\odot\mathbf{H}_{t-1}+\mathbf{k}_t^\top\mathbf{v}_t
\qquad\text{and}\qquad
\mathbf{y}_t=\mathbf{q}_t\mathbf{H}_t.
\label{eq:gla-rec}
\end{equation}
where $\mathbf{G}_t \in \mathbb{R}^{N \times D}$ is an optional gating term.

Identifying $\mathbf{G}_t \equiv \mathbf{\bar A}_t$, $\mathbf{k}_t^\top \equiv \mathbf{\bar B}_t$ and $\mathbf{q}_t \equiv \mathbf{C}_t$ 
shows that GLA~\citep{yang2023gated} matches Mamba when $\mathbf{W}_v = \mathbf{I}$.

\paragraph{Gaussian state-space models and Bayesian filtering.}
\label{par:ssm}
A classical probabilistic lens on sequential data is \emph{state estimation}: we posit an (unobserved) latent state
$\mathbf{z}_t \in \mathbb{R}^d$ that evolves over time, and assume the observed sequence
$\mathbf{o}_t \in \mathbb{R}^p$ provides noisy evidence about that state.
The goal of \emph{Bayesian filtering} is to maintain the posterior belief
$p(\mathbf{z}_t \mid \mathbf{o}_{1:t})$ online, recursively combining a \emph{dynamics prior}
(predict) with an \emph{observation likelihood} (update). 

When the dynamics and observations are linear with Gaussian noise, we obtain a linear-Gaussian state-space model:\looseness-1
\begin{align}
\mathbf{z}_t &= \mathbf{A}_t \mathbf{z}_{t-1} + \mathbf{w}_t,
\qquad \mathbf{w}_t \sim \mathcal{N}(\mathbf{0}, \mathbf{P}_t), \label{eq:lgssm_dyn}\\
\mathbf{o}_t &= \mathbf{C}_t \mathbf{z}_t + \mathbf{v}_t,
\qquad \mathbf{v}_t \sim \mathcal{N}(\mathbf{0}, \bm\Sigma^{\mathrm{obs}}_t), \label{eq:lgssm_obs}
\end{align}
in which the filtering posterior remains Gaussian, $
p(\mathbf{z}_t \mid \mathbf{o}_{1:t}) = \mathcal{N}(\bm\mu_t, \bm\Sigma_t)$
and can be computed in closed form via the \emph{Kalman filter} (KF), which is optimal in the minimum mean-squared error (MMSE) sense under these assumptions.

\paragraph{Control versus observation: two ways to read a token sequence.}
This filtering view contrasts with the common ``controlled dynamics'' interpretation of deterministic mixers (e.g.\ SSM/CDE-style models)~\citep{muca2024theoretical}, where the input is an exogenous control driving the hidden state forward; Bayesian filtering instead treats it as noisy observations of a stochastic latent process. This swap, control versus observation, changes what is random and what computation is appropriate, forward simulation versus posterior inference (\cref{fig:control-obs-view}). We adopt the filtering perspective throughout, yielding an attention-like mixer with explicit uncertainty propagation.\looseness-1

\paragraph{Notational conventions.}
We index token sequences by $t\in\{1,\dots,T\}$. Bold uppercase terms denote matrices (or collections over time), lowercase denotes scalars, and boldface denotes vectors/sequences (e.g., $\mathbf{x}_t$).
In the diagonal (per-coordinate) parameterisation, we identify diagonal matrices with their diagonal terms and apply scalar recursions elementwise. An overbar denotes the discretised counterpart of a continuous-time parameter, e.g.\ $\bar a_t$ is obtained by discretising a continuous-time decay parameter.\looseness-1

For completeness, Appendix \ref{app:ex-background} includes additional background material on related mathematical constructs such as M\"{o}bius (Fractional Linear) Transforms, the Information Form of Gaussian distributions, and Information Filters.

\begin{figure}[t]
  \centering
  \resizebox{0.92\columnwidth}{!}{\tikzOUHybridTall}
  \caption{From OU dynamics to parallel inference.
  \textbf{(Top)} Continuous-time OU prior.
  \textbf{(Middle)} Discrete linear-Gaussian SSM.
  \textbf{(Bottom)} M\"{o}bius scan for parallel posterior state estimation.\vspace{-0.1in}}
  \label{fig:ou-hybrid-tall}
\end{figure}
\section{Method}
\label{sec:method}

We introduce \emph{Kalman Linear Attention (KLA)} as a \emph{probabilistic} sequence mixer. Instead of deterministically updating a hidden state as in modern SSM mixers, KLA maintains a belief state over a latent representation, consisting of a posterior mean and an explicit uncertainty (precision/covariance). This probabilistic view turns sequence mixing into state estimation from noisy token evidence, and provides a principled mechanism for adaptivity and a gating mechanism derived from relative uncertainties of state and observation.

Our method is organised around two ideas.
\textbf{(i) A probabilistic SSM view of sequence mixing.}
We cast token processing as inference in a linear-Gaussian state-space model. Each token contributes a (noisy) observation $v_t$ of the latent state through an observation operator $k_t$, while the latent evolves under a continuous-time stochastic prior that we discretise for sequences. 

\textbf{(ii) Parallelisable filtering via information form.}
We address the main obstacle to using Bayesian filters as sequence mixers: classical filtering is recursive and appears inherently sequential. We show that by working in \emph{information form} (precisions and natural parameters), the posterior updates acquire a parallelisation-friendly M\"{o}bius (fractional-linear) structure. This yields a scalable implementation of Bayesian sequence mixing with the computational profile of modern linear-time mixers, while preserving clear probabilistic semantics.\looseness-1

\paragraph{Notation and scope.}
Throughout this section, \unc{blue terms} denote uncertainty-aware quantities absent in deterministic SSMs. For readability, the remainder of this section fixes the state-expansion factor to $N{=}1$ and presents the \emph{diagonal} (per-channel) Kalman filter, so that every recursion reduces to scalar (elementwise) operations. The state-expanded ($N{>}1$) form - in which the state becomes a matrix $\mathbf{H}_t\in\mathbb{R}^{N\times D}$ - is given in \cref{tab:online-learning} and \cref{alg:kla}. We stress that this matrix arises \emph{solely} from state expansion (parallel memory slots per channel) and is \emph{not} a full covariance: \our's covariance/precision remains diagonal throughout.\footnote{\label{fn:notation}\our uses two equivalent state parameterisations related by the precision: $\boldsymbol{\eta}_t,\boldsymbol{\mu}_t$ are the per-channel ($N{=}1$) \emph{information}- and \emph{moment}-form means, and $\mathbf{H}_t,\mathbf{S}_t\!\in\!\mathbb{R}^{N\times D}$ their state-expanded ($N{>}1$) counterparts, related by $\mathbf{H}_t=\boldsymbol{\Lambda}_t\odot\mathbf{S}_t$ (i.e.\ $\boldsymbol{\eta}_t=\boldsymbol{\lambda}_t\odot\boldsymbol{\mu}_t$). The parallel scan runs in the information form ($\mathbf{H}_t$, \cref{alg:kla}); the moment form ($\mathbf{S}_t$, \cref{tab:online-learning}) is an equivalent representation used by the query readout. The explicit state-expanded recursions are given in \cref{eq:kla-prec-mat,eq:kla-mean-mat}.}
\subsection{Stochastic Dynamics and Token Likelihood}
\label{sec:stochastic_dynamics}
\paragraph{Dynamics prior (predict step) via OU discretisation.}
Modern SSM mixers~\cite{gu2021efficiently,smith2022simplified} are often motivated from continuous-time dynamics and then discretised to operate on token sequences. In this view, a simple linear ODE yields a discrete-time recurrence in which the discretised decay $\bar{\mathbf{a}}_t$ acts as a forget factor, and the step size $\Delta t$ controls the effective time scale. This provides an appealing interpretation of SSM channels as a spectrum of memories parameterised by their decay rates.

KLA adopts the same continuous-to-discrete modelling pipeline, but makes one conceptual change: we treat latent evolution as inherently uncertain. Concretely, we specify a continuous-time \emph{stochastic} dynamics prior and discretise it exactly. We use an Ornstein-Uhlenbeck (OU) process as the prior because it is the canonical mean-reverting diffusion (and the continuous-time analogue of a stable AR(1)): it preserves exponential forgetting while introducing an explicit process-noise term that captures unmodelled variability and uncertainty in the latent dynamics.
Under exact discretisation, the OU prior induces a Gaussian transition
\begin{align}
\mathbf{z}_t \mid \mathbf{z}_{t-1} &\sim \mathcal{N}\!\bigl(\bar{\mathbf{a}}_t \odot \mathbf{z}_{t-1},\, \unc{\bar{\mathbf{p}}_t}\bigr),
\label{eq:ssm_dyn} \\
\label{eq:discretisation}
\bar{\mathbf{a}}_t &= e^{-\mathbf{a}\,\Delta t},
\qquad
\unc{\bar{\mathbf{p}}_t = \frac{\mathbf{p}^2}{2\mathbf{a}}\odot\Bigl(1-e^{-2\mathbf{a}\,\Delta t}\Bigr)},
\end{align}
where $\bar{\mathbf{a}}_t$ denotes the decay and $\bar{\mathbf{p}}_t$ the process-noise variance.

\paragraph{Multi-channel specialisation.} Importantly, $\bar{\mathbf{p}}_t$ is coupled to $\bar{\mathbf{a}}_t$ by construction: the same decay and time-scale parameters $\mathbf{a}$ that control decay also determine how uncertainty accumulates between observations.
This yields per-channel specialisation along two linked axes: memory decay and drift.
Different coordinates can learn different decay rates (short- vs.\ long-memory channels), and their process noise $\bar{\mathbf{p}}_t$ sets how freely each channel can drift away from its predicted trajectory between tokens. This is different from Mamba, where the timescale/discretisation parameter is token-dependent and undertakes the role of selection/filtering at each timestep, whereas in our case selection/filtering is achieved purely via uncertainties without overloading $\Delta_t$ with additional roles.\looseness-1

\paragraph{Input token likelihood as noisy evidence (update step).}
To connect this dynamics prior to input token sequences, we model each token as providing noisy evidence about the latent state:
\begin{equation}
\mathbf{v}_t \mid \mathbf{z}_t \sim \mathcal{N}\!\bigl(\mathbf{k}_t \odot \mathbf{z}_t,\, \unc{(\boldsymbol{\Lambda}_t^{\mathrm{v}})^{-1}}\bigr),
\label{eq:ssm_obs}
\end{equation}
where $\mathbf{v}_t$ is token-derived content, $\mathbf{k}_t$ is an observation operator, and $\boldsymbol{\Lambda}_t^{\mathrm{v}}$ is the value precision representing confidence in the token evidence. Combining~\cref{eq:ssm_dyn} with~\cref{eq:ssm_obs} yields a linear-Gaussian state-space model, and \emph{sequence mixing corresponds to posterior inference} of the latent state given token evidence.

\begin{table}[t]
\centering
\caption{Notation alignment across attention, SSMs, and \our. Precisions ($\boldsymbol{\lambda}$, $\boldsymbol{\Lambda}^{\mathrm v}$, $\boldsymbol{\Lambda}^{\mathrm p}$) are inverse variances encoding confidence: higher precision means lower uncertainty.}
\label{tab:notation-alignment}
\footnotesize
\renewcommand{\arraystretch}{1.15}
\setlength{\tabcolsep}{4pt}
\resizebox{\columnwidth}{!}{%
\begin{tabular}{@{}p{0.95in}p{1.1in}p{1.1in}@{}}
\toprule
\textbf{Attention} & \textbf{Deterministic SSM / GLA} & \textbf{\our (probabilistic)} \\
\midrule
Query operator $\mathbf{q}_t$ & Output/readout map $\mathbf{C}_t$ & Query/readout operator $\mathbf{q}_t$ \\
Key operator $\mathbf{k}_t$ & Input to latent map $\mathbf{b}_t$ & Observation model/operator $\mathbf{k}_t$ \\
Value $\mathbf{v}_t$ & Token input/control $\mathbf{u}_t$ & Noisy observation $\mathbf{v}_t$ \\
Output $\mathbf{y}_t$ & Output $\mathbf{y}_t$ & Output $\mathbf{y}_t$ \\
\midrule
\multicolumn{3}{@{}l}{\textit{Hidden state (implicit in attention, explicit in SSMs/\our):}} \\
\rowcolor{gray!8}
\textcolor{gray}{---} & Hidden state $\mathbf{h}_t$ & Posterior belief $\mathcal{N}(\boldsymbol{\mu}_t, \boldsymbol{\Lambda}_t^{-1})$ over $\mathbf{z}_t$ \\
\rowcolor{gray!8}
\textcolor{gray}{---} & State decay $\bar{\mathbf{a}}_t = e^{-\mathbf{a}(t)\Delta t}$ & State decay $\bar{\mathbf{a}}_t = e^{-\mathbf{a} \Delta t}$ \\
\midrule
\multicolumn{3}{@{}l}{\textit{Uncertainty terms (explicit in \our only):}} \\
\rowcolor{gray!8}
\textcolor{gray}{---} & \textcolor{gray}{---} & Value precision $\boldsymbol{\Lambda}_t^{\mathrm v}$ \\
\rowcolor{gray!8}
\textcolor{gray}{---} & \textcolor{gray}{---} & State precision $\boldsymbol{\lambda}_t$ \\
\rowcolor{gray!8}
\textcolor{gray}{---} & \textcolor{gray}{---} & Process noise $\bar{\mathbf{p}}_t = \tfrac{\mathbf{p}^2}{2\mathbf{a}}\odot(1-e^{-2\mathbf{a}\Delta t})$ \\
\bottomrule
\end{tabular}%
}
\vspace{-0.15in}
\end{table}
\paragraph{Output readout as a query/task-conditioned projection of the belief state.}
So far, $(\mathbf{k}_t,\mathbf{v}_t,\boldsymbol{\Lambda}_t^{\mathrm v})$ define how input token evidence updates the latent belief $p(\mathbf{z}_t\mid \mathbf{v}_{1:t})$.
A remaining design choice is how to use this belief to produce an output representation for downstream prediction.
Deterministic SSM mixers typically apply a readout $\mathbf{y}_t=\mathbf{C}_t\odot \mathbf{h}_t$, while transformers use a query to determine what information should be extracted from stored context.
KLA adopts an analogous view: the \emph{query} $\mathbf{q}_t$ specifies \emph{what we want to read out} from the inferred belief state at time $t$.
To make this explicit, we introduce a (linear-Gaussian) readout model
\begin{equation}
\mathbf{y}_t \mid \mathbf{z}_t \;\sim\; \mathcal{N}\!\bigl(\mathbf{q}_t \odot \mathbf{z}_t,\; (\boldsymbol{\Lambda}^{\mathrm{out}}_t)^{-1}\bigr),
\label{eq:readout_model}
\end{equation}
where $\boldsymbol{\Lambda}^{\mathrm{out}}_t$ is an \emph{output precision} (readout noise).
In this work we take the deterministic-readout limit $\boldsymbol{\Lambda}^{\mathrm{out}}_t \to \infty$ (equivalently, zero readout noise), so that the output is the posterior mean projection
\begin{equation}
\mathbf{y}_t \;=\; \mathbf{q}_t \odot \boldsymbol{\mu}_t,
\label{eq:readout}
\end{equation}
which mirrors the role of a query-conditioned readout in attention and a $\mathbf{C}_t$-projection in SSMs. An intuitive analogy is accumulating beliefs over time via posterior inference through an input modality $\mathbf{v}_{t}$ (say image sensors) via observation model $\mathbf{k}_{t}$ and converting the latent beliefs to an output modality (say proprioceptive sensors) $\mathbf{y}_{t}$ via query model $\mathbf{q}_t$.

\paragraph{Gaussian SSM view of $q$-$k$-$v$ interactions.}
We can summarise the attention-aligned probabilistic interpretation as a linear-Gaussian state-space model in which keys and values parameterise the \emph{likelihood} (token evidence), while queries parameterise a \emph{readout} applied after inference.
Let $\mathbf{z}_t$ denote the latent state and let $(\mathbf{k}_t,\mathbf{v}_t,\boldsymbol{\Lambda}_t^{\mathrm v})$ be token-dependent likelihood parameters produced from the input sequence.
The generative model is:
\begin{align}
& \textbf{Prior}\quad
&\mathbf{z}_t \mid \mathbf{z}_{t-1}
&\sim \mathcal{N}\!\bigl(\bar{\mathbf{a}}_t \odot \mathbf{z}_{t-1},\; \bar{\mathbf{p}}_t\bigr),
\label{eq:gen_dyn_qkv}\\
&\textbf{Evidence}\quad
&\mathbf{v}_t \mid \mathbf{z}_t
&\sim \mathcal{N}\!\bigl(\mathbf{k}_t \odot \mathbf{z}_t,\; (\boldsymbol{\Lambda}_t^{\mathrm v})^{-1}\bigr),
\label{eq:gen_obs_qkv}\\
&\textbf{Readout}\quad
&\mathbf{y}_t \mid \mathbf{z}_t
&\sim \mathcal{N}\!\bigl(\mathbf{q}_t \odot \mathbf{z}_t,\; (\boldsymbol{\Lambda}_t^{\mathrm{out}})^{-1}\bigr).
\label{eq:gen_readout_qkv}
\end{align}

In this view, $\mathbf{k}_t$ specifies the observation geometry (how token evidence constrains the latent), $\mathbf{v}_t$ is the observed token evidence, and $\boldsymbol{\Lambda}_t^{\mathrm v}$ is its reliability (precision).
The query $\mathbf{q}_t$ instead specifies what component of the inferred latent state is exposed as the output.

\subsection{Parallel Inference via M\"obius and Affine Scans}
\label{sec:parallel_inference}

A major obstacle to using Bayesian filters as sequence mixers is that classical filtering is \emph{recursive} and thus appears inherently sequential.
Our key observation is that, in \emph{information form}, the updates admit a \emph{compositional} structure that can be implemented with parallel scans.

\paragraph{Information form.}
We represent the Gaussian posterior by its precision and natural (canonical) parameter:
\begin{equation}
p(\mathbf{z}_t \mid \mathbf{v}_{1:t}) = \mathcal{N}(\boldsymbol{\mu}_t,\boldsymbol{\lambda}_t^{-1}),
\qquad
\boldsymbol{\eta}_t := \boldsymbol{\lambda}_t \odot \boldsymbol{\mu}_t.
\end{equation}
This parameterisation is convenient because, in linear-Gaussian models, incorporating token evidence corresponds to \emph{adding} canonical parameters (precisions and information means), while the remaining ``predict'' transformation induced by the dynamics takes a structured form.
As we show next, these structured updates \emph{compose associatively} across time, which is exactly what enables parallel prefix scans.
\begin{restatable}[Precision Update as a M\"obius Transformation]{theorem}{ThMobius}
\label{Th:Mobius}
Let $\unc{\boldsymbol{\lambda}_t}$ be the posterior precision at time $t$ in the diagonal linear-Gaussian model
\cref{eq:ssm_dyn}--\cref{eq:ssm_obs}. Define $\unc{\boldsymbol{\phi}_t \coloneqq \mathbf{k}_t^{2}\odot\boldsymbol{\Lambda}_t^{\mathrm v}}$.
Then the map $\boldsymbol{\lambda}_{t-1} \mapsto \boldsymbol{\lambda}_t$ is a linear-fractional (M\"obius) transform:
\begin{align}
\boldsymbol{\lambda}_t &= \mathbf{M}_t(\boldsymbol{\lambda}_{t-1})
=
\frac{\boldsymbol{\alpha}_t \odot \boldsymbol{\lambda}_{t-1} + \boldsymbol{\beta}_t}{\boldsymbol{\gamma}_t \odot \boldsymbol{\lambda}_{t-1} + \boldsymbol{\delta}_t}, \label{eq:mobius_main} \\[4pt]
\mathbf{M}_t &=
\begin{pmatrix}
\boldsymbol{\alpha}_t & \boldsymbol{\beta}_t\\
\boldsymbol{\gamma}_t & \boldsymbol{\delta}_t
\end{pmatrix}
=
\begin{pmatrix}
1 + \unc{\mathbf{\bar{p}}_t \odot \boldsymbol{\phi}_t} & \mathbf{\bar{a}}_t^{2}\odot\unc{\boldsymbol{\phi}_t}\\
\unc{\mathbf{\bar{p}}_t} & \mathbf{\bar{a}}_t^{2}
\end{pmatrix}.
\end{align}
\end{restatable}
\paragraph{Interpretation: precision track as uncertainty-driven gating.}
While \Cref{eq:mobius_main} is nonlinear in $\boldsymbol{\lambda}_{t-1}$, it admits a simple ``gate-like'' rearrangement:
\begin{align}
\unc{\boldsymbol{\lambda}_t}
&=
\underbrace{(\mathbf{\bar{a}}_t^2 + \unc{\mathbf{\bar{p}}_t \odot \boldsymbol{\lambda}_{t-1}})^{-1}\odot\unc{\boldsymbol{\lambda}_{t-1}}}_{\textit{(confidence history)}}
\;+\;
\underbrace{\unc{\boldsymbol{\phi}_t = \mathbf{k}_t^2\odot\boldsymbol{\Lambda}_t^{\mathrm v}}}_{\textit{(confidence in current token)}}.
\label{eq:lambda_gate_view}
\end{align}
The shared denominator $\bar{\mathbf{a}}_t^2 + \bar{\mathbf{p}}_t \odot \boldsymbol{\lambda}_{t-1}$ introduces \emph{history dependence}: as accumulated precision grows, the model naturally becomes more selective about incorporating new evidence.
We name its inverse $\unc{\boldsymbol{\rho}_t \coloneqq \mathbf{1}\oslash(\bar{\mathbf{a}}_t^2 + \bar{\mathbf{p}}_t \odot \boldsymbol{\lambda}_{t-1})}$, so the history term of \cref{eq:lambda_gate_view} is $\unc{\boldsymbol{\rho}_t\odot\boldsymbol{\lambda}_{t-1}}$.
The \emph{same} factor $\boldsymbol{\rho}_t$ reappears in the mean recursion below as the forget gate $\mathbf{f}_t = \boldsymbol{\rho}_t\odot\bar{\mathbf{a}}_t$ (\Cref{thm:eta-affine}), tying the precision and mean tracks together and inducing the input/forget-gate behaviour of the mean update (\Cref{eq:mu-update}).
\begin{restatable}[Precision Updates via Parallel Prefix Scan]{corollary}{CorParallelScan}
\label{cor:parallel-scan}
Given $\boldsymbol{\lambda}_0$ and matrices $\{\mathbf{M}_t\}_{t=1}^T$ from \Cref{Th:Mobius}, the posterior precisions
$\{\boldsymbol{\lambda}_t\}_{t=1}^T$ can be computed by a parallel prefix scan over $\{\mathbf{M}_t\}$ with
$\mathcal{O}(T)$ work and $\mathcal{O}(\log T)$ depth.
\end{restatable}
\begin{restatable}[Mean Update as Affine Transformations]{theorem}{ThMeanAffine}
\label{thm:eta-affine}
Let $\boldsymbol{\eta}_t \coloneqq \unc{\boldsymbol{\lambda}_t}\odot\boldsymbol{\mu}_t$ be the posterior information mean and $\unc{\boldsymbol{\Lambda}_t^{\mathrm v}}$ the value precision.
Given the precision path $\{\unc{\boldsymbol{\lambda}_t}\}$, the information mean evolves affinely:
\begin{align}
\boldsymbol{\eta}_t
&=
\underbrace{(\mathbf{\bar{a}}_t^2 + \unc{\mathbf{\bar{p}}_t \odot \boldsymbol{\lambda}_{t-1}})^{-1}\odot \mathbf{\bar{a}}_t}_{\mathbf{f}_t\;\;\text{(history-dependent forget gate)}}
\odot\boldsymbol{\eta}_{t-1}
\;+\;
\underbrace{\mathbf{k}_t \odot \unc{\boldsymbol{\Lambda}_t^{\mathrm v}} \odot \mathbf{v}_t}_{\text{token evidence}}.
\label{eq:mu-update}
\end{align}
\end{restatable}
\begin{restatable}[Mean Updates via Parallel Prefix Scan]{corollary}{CorMeanParallelScan}
\label{cor:mean-parallel-scan}
The posterior information means $\{\boldsymbol{\eta}_t\}_{t=1}^T$ are computable via parallel prefix scan over the affine transformations in Theorem~\ref{thm:eta-affine} in $\mathcal{O}(T)$ work and $\mathcal{O}(\log T)$ depth.
\end{restatable}
\begin{figure*}[t]
\centering
\begin{subfigure}[b]{0.15\textwidth}
  \centering
  \includegraphics[width=\linewidth,height=4.8cm,keepaspectratio]{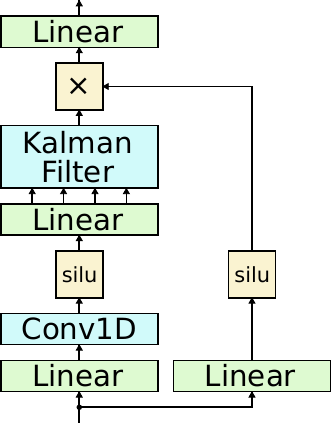}
  \caption{\our block.}
  \label{fig:block_arch}
\end{subfigure}\hspace{0.02\textwidth}%
\begin{subfigure}[b]{0.80\textwidth}
  \centering
  \includegraphics[width=\linewidth]{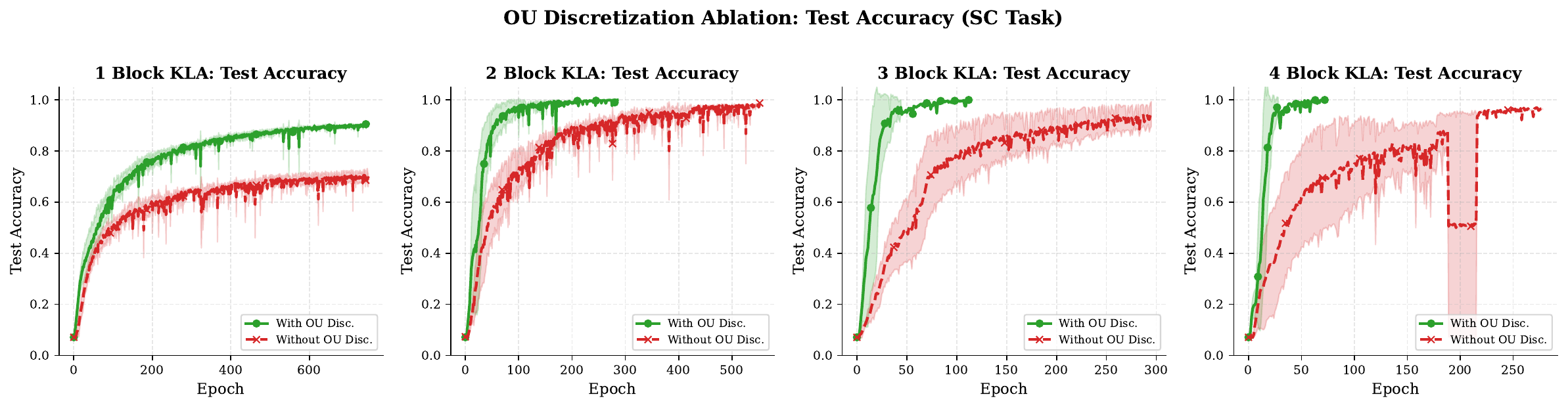}
  \caption{Importance of OU stochastic dynamics for filter stability.}
  \label{fig:ou-ablation}
\end{subfigure}
\caption{\textbf{(a)} The \our block follows Mamba's fused-MLP design, with the Kalman filter as a drop-in mixer (full architecture in \cref{app:architecture}). \textbf{(b)} Ablating the OU prior dynamics and discretisation (\cref{sec:stochastic_dynamics}): OU discretisation improves accuracy and learning stability, especially for deeper models.}
\label{fig:block-and-ablation}
\vspace{-0.1in}
\end{figure*}
\noindent\textbf{Special cases:} Under deterministic ($\mathbf{p}_t=0$) and linear time-invariant (LTI) settings, the KLA updates reduce to convolutions computable in $\mathcal{O}(\log T)$ time via FFT (see Theorem~\ref{thm:conv-form} in the Appendix).

\subsection{\our as an In-Context Online Learner: A Precision-Weighted Least-Squares View}
\label{sec:online-learning}

Recent sub-quadratic mixers are increasingly unified through an
\emph{online-learning} / fast-weight lens, first articulated for SSMs by
Longhorn~\citep{liu2025longhorn} and developed in follow-up linear-attention and
delta-rule work~\citep{schlag2021linear,yang2023gated,yang2024gated,katharopoulos2020transformers,dao2024transformers}:
each recurrent update is the closed-form minimiser of a per-token objective that
trades a \emph{proximal} term (stay close to the decayed memory) against a
\emph{fit} term (absorb the current token evidence).

We can show that \our's equivalent moment-form recursion admits an online update that is a \emph{gated RNN update} solving a local \emph{precision-weighted least-squares} objective. Reading the information-mean recursion of \Cref{thm:eta-affine} in moment
coordinates, $\boldsymbol{\mu}_t = \boldsymbol{\eta}_t\oslash\boldsymbol{\lambda}_t$, collapses the
predict/update pair into this single gated recurrence - the diagonal/scalar instance of the \our row of \cref{tab:online-learning}:
\begin{restatable}[Posterior-mean recursion as a gated RNN update]{corollary}{CorMomentForm}
\label{cor:moment-form}
The moment-form posterior mean $\boldsymbol{\mu}_t = \boldsymbol{\eta}_t\oslash\unc{\boldsymbol{\lambda}_t}$ of the information recursion in \Cref{thm:eta-affine} satisfies the gated recurrence
\begin{equation}
\boldsymbol{\mu}_t
= \bar{\mathbf{a}}_t\odot\Bigl(\mathbf{1} - \frac{\mathbf{k}_t^{2}\odot\unc{\boldsymbol{\Lambda}_t^{\mathrm v}}}{\unc{\boldsymbol{\lambda}_t}}\Bigr)\odot\boldsymbol{\mu}_{t-1}
\;+\; \frac{\mathbf{k}_t\odot\unc{\boldsymbol{\Lambda}_t^{\mathrm v}}\odot\mathbf{v}_t}{\unc{\boldsymbol{\lambda}_t}},
\label{eq:mu-delta}
\end{equation}
the diagonal/scalar instance of the \our state update in \cref{tab:online-learning}.
\end{restatable}
\noindent We prove \cref{cor:moment-form} in \cref{app:theorems-proofs}. Concretely, this gated update is the exact minimiser $\boldsymbol{\mu}_t = \arg\min_{\boldsymbol{\mu}} L_t$ of the
precision-weighted least-squares objective
\begin{equation}
L_t(\boldsymbol{\mu}) = \unc{\boldsymbol{\Lambda}_t^{\mathrm v}}\,\bigl\|\mathbf{v}_t - \mathbf{k}_t\odot\boldsymbol{\mu}\bigr\|^2
\;+\; \unc{\boldsymbol{\rho}_t\odot\boldsymbol{\lambda}_{t-1}}\,\bigl\|\boldsymbol{\mu} - \bar{\mathbf{a}}_t\odot\boldsymbol{\mu}_{t-1}\bigr\|^2,
\label{eq:kla-ridge}
\end{equation}
balancing a fit term (weighted by the observation precision $\boldsymbol{\Lambda}_t^{\mathrm v}$)
against a proximal term anchoring $\boldsymbol{\mu}$ to the propagated mean
$\bar{\mathbf{a}}_t\odot\boldsymbol{\mu}_{t-1}$ with weight the prior precision
$\unc{\boldsymbol{\lambda}_t^{\mathrm{prior}} = \boldsymbol{\rho}_t\odot\boldsymbol{\lambda}_{t-1}}$
 - the history-dependent factor $\boldsymbol{\rho}_t$ of \cref{eq:lambda_gate_view}, i.e.\ the
prior-to-previous precision ratio. \Cref{tab:online-learning} places \our in this
template alongside representative delta-rule and gated mixers. \our's
nonlinearity is supplied entirely by the process noise $\bar{\mathbf{p}}_t$, which enters
only through the state-dependent factor $\boldsymbol{\rho}_t$ of the M\"obius precision
recursion (\Cref{Th:Mobius}). Fixing $\bar{\mathbf{p}}_t=\mathbf{0}$ makes $\boldsymbol{\rho}_t$
constant and linearises the recursion, collapsing \our to a fixed-forgetting linear
recurrence. We isolate the effect of $\bar{\mathbf{p}}_t$
in \cref{sec:ablations}.\looseness-1

\begin{table*}[t]
\centering
\caption{Local online objectives and state updates across linear-attention mixers. \fg{Orange} marks the forget gate (coefficient of $\mathbf{S}_{t-1}$), \wg{green} the write gate (coefficient of $\mathbf{k}_t\mathbf{v}_t^\top$), and \kl{blue} the quantities unique to KLA - the precision weights in the objective (the only \emph{precision-weighted} least-squares in the table, as marked by the underbrace) and the M\"obius recursion supplying $\bm{\Lambda}_t$. Subscript $t$ denotes per-token learned quantities; symbols without subscript denote global parameters. KLA differs structurally by deriving its gating from a coupled recursion - the matrix-M\"obius precision update of Theorem~1 (\cref{eq:kla-prec-mat}) - whose solution $\bm{\Lambda}_t$ supplies the precision-ratio gain visible in both forget and write terms.}
\label{tab:online-learning}
\scriptsize
\setlength{\tabcolsep}{4pt}
\renewcommand{\arraystretch}{1.55}
\resizebox{\textwidth}{!}{%
\begin{tabular}{@{}llll@{}}
\toprule
\textbf{Method} & \textbf{Local online objective $\mathcal{L}_t(\mathbf{S})$} & \textbf{State update}$^{*}$ & \textbf{Gates from} \\
\midrule
\multicolumn{4}{@{}l}{\textit{Correlation write into a (decayed) state}}\\[2pt]
Linear Attn.   & $\|\mathbf{S}-\mathbf{S}_{t-1}\|_F^2 - 2\langle \mathbf{S}^\top\mathbf{k}_t,\mathbf{v}_t\rangle$
               & $\mathbf{S}_t = \mathbf{S}_{t-1} + \mathbf{k}_t\mathbf{v}_t^\top$
               & - \\
Mamba-1 (S6)   & $\|\mathbf{S} - \bar{\mathbf{A}}_t\mathbf{S}_{t-1}\|_F^2 - 2\langle \mathbf{S}^\top\mathbf{k}_t,\mathbf{v}_t\rangle$
               & $\mathbf{S}_t = \fg{\bar{\mathbf{A}}_t}\mathbf{S}_{t-1} + \mathbf{k}_t\mathbf{v}_t^\top$
               & $\mathbf{A}, \bar{\mathbf{A}}_t$ \\
Mamba-2        & $\|\mathbf{S} - \alpha_t\mathbf{S}_{t-1}\|_F^2 - 2\langle \mathbf{S}^\top\mathbf{k}_t,\mathbf{v}_t\rangle$
               & $\mathbf{S}_t = \fg{\alpha_t}\mathbf{S}_{t-1} + \mathbf{k}_t\mathbf{v}_t^\top$
               & $a, \alpha_t$ \\
\midrule
\multicolumn{4}{@{}l}{\textit{Delta (residual) write}}\\[2pt]
DeltaNet~\citep{schlag2021linear}       & $\|\mathbf{S}-\mathbf{S}_{t-1}\|_F^2 - 2\langle \mathbf{S}^\top\mathbf{k}_t,\,\beta_t(\mathbf{v}_t - \mathbf{S}_{t-1}^\top\mathbf{k}_t)\rangle$
               & $\mathbf{S}_t = \fg{(\mathbf{I}-\beta_t\mathbf{k}_t\mathbf{k}_t^\top)}\mathbf{S}_{t-1} + \wg{\beta_t}\mathbf{k}_t\mathbf{v}_t^\top$
               & $\beta_t$ \\
Gated DeltaNet & $\|\mathbf{S}-\alpha_t\mathbf{S}_{t-1}\|_F^2 - 2\langle \mathbf{S}^\top\mathbf{k}_t,\,\beta_t(\mathbf{v}_t - (\alpha_t\mathbf{S}_{t-1})^\top\mathbf{k}_t)\rangle$
               & $\mathbf{S}_t = \fg{\alpha_t(\mathbf{I}-\beta_t\mathbf{k}_t\mathbf{k}_t^\top)}\mathbf{S}_{t-1} + \wg{\beta_t}\mathbf{k}_t\mathbf{v}_t^\top$
               & $\alpha_t, \beta_t$ \\
KDA~\citep{team2025kimi} & $\|\mathbf{S}-\mathbf{D}_t\mathbf{S}_{t-1}\|_F^2 - 2\langle \mathbf{S}^\top\mathbf{k}_t,\,\beta_t(\mathbf{v}_t - (\mathbf{D}_t\mathbf{S}_{t-1})^\top\mathbf{k}_t)\rangle$
               & $\mathbf{S}_t = \fg{(\mathbf{I}-\beta_t\mathbf{k}_t\mathbf{k}_t^\top)\mathbf{D}_t}\mathbf{S}_{t-1} + \wg{\beta_t}\mathbf{k}_t\mathbf{v}_t^\top$
               & $\mathbf{D}_t, \beta_t$ \\
GDN-2~\citep{hatamizadeh2026gated} & $\|\mathbf{S}-\mathbf{D}_t\mathbf{S}_{t-1}\|_F^2 - 2\langle \mathbf{S}^\top\mathbf{k}_t,\,\mathbf{w}_t\mathbf{v}_t - (\mathbf{D}_t\mathbf{S}_{t-1})^\top\mathbf{b}_t\mathbf{k}_t\rangle$
               & $\mathbf{S}_t = \fg{(\mathbf{I}-\mathbf{k}_t(\mathbf{b}_t\mathbf{k}_t)^\top)\mathbf{D}_t}\mathbf{S}_{t-1} + \mathbf{k}_t(\wg{\mathbf{w}_t}\mathbf{v}_t)^\top$
               & $\mathbf{D}_t, \mathbf{b}_t, \mathbf{w}_t$ \\
\midrule
\multicolumn{4}{@{}l}{\textit{Bayesian filtering write}}\\[2pt]
\textbf{KLA (Ours)}
               & $\underbrace{\kl{\bm{\Lambda}_t^{\mathrm{prior}}}\|\mathbf{S}-\bar{\mathbf{A}}\mathbf{S}_{t-1}\|_F^2 + \kl{\bm{\Lambda}_t^v}\|\mathbf{S}^\top\mathbf{k}_t - \mathbf{v}_t\|^2}_{\kl{\textit{uncertainty- (precision-) weighted least squares}}}$
               & $\mathbf{S}_t = \fg{\bar{\mathbf{A}}\Bigl(\mathbf{I} - \dfrac{\mathbf{k}_t^{2}(\bm{\Lambda}_t^v)^\top}{\bm{\Lambda}_t}\Bigr)}\mathbf{S}_{t-1} + \dfrac{\mathbf{k}_t(\wg{\bm{\Lambda}_t^v}\mathbf{v}_t)^\top}{\wg{\bm{\Lambda}_t}}$
               & \makecell[l]{$\bar{\mathbf{A}}, \bar{\mathbf{P}}, \bm{\Lambda}_t^v,$ \\ \kl{\textbf{M\"obius recursion ($\bm{\Lambda}_t$)}}} \\
\bottomrule
\end{tabular}%
}
\vspace{2pt}
\begin{minipage}{\textwidth}
\footnotesize $^{*}$\,For two matrices $\mathbf{A}, \mathbf{B}$ of equal shape we overload the fraction $\mathbf{A}/\mathbf{B}$ to denote \emph{element-wise} (Hadamard) division; products between vectors of equal dimension (e.g.\ $\mathbf{b}_t\mathbf{k}_t, \mathbf{w}_t\mathbf{v}_t, \bm{\Lambda}_t^v\mathbf{v}_t$) are element-wise. Vector outer products such as $\mathbf{k}_t\mathbf{v}_t^\top$ and $\mathbf{k}_t^{2}(\bm{\Lambda}_t^v)^\top$ form $d_k\!\times\!d_v$ matrices in the standard way.
\end{minipage}
\end{table*}

\subsection{\our as a Drop-in Probabilistic Sequence Mixer}
\label{sec:ka_layer}

In order to instantiate the above parallel Bayesian filter as a neural network layer, we parametrise the observation value and precision ($\mathbf{v}_t, \boldsymbol{\Lambda}^{\mathrm{v}}_t$) as well as the observation and readout operators ($\mathbf{k}_t, \mathbf{q}_t$) in terms of the layer input sequence $x_{1:T}$.
The observation variance is parametrised such that it is positive. For the continuous stochastic prior (\cref{sec:stochastic_dynamics}), we treat $\mathbf{a}_t$ and $\mathbf{p}_t$ as learnable, time-invariant parameters
$\mathbf{a}$ and $\mathbf{p}$, respectively, in contrast to Mamba where $\mathbf{a}_t$ is token-dependent and time-varying.

The block architecture follows a fused MLP architecture as shown in \cref{fig:block_arch}. Following the Mamba design, we employ standard scaffolding components paired with the \our sequence mixer: a 1D causal convolution with kernel size 4 and SiLU activation, along with residual connections and simple nonlinearities. See \cref{fig:gauss_mamba_arc} and Appendix~\ref{par:arch} for more architecture and implementation details. We implement efficient Triton~\cite{tillet2019triton} kernels for the Linear and M\"{o}bius parallel scan primitives; however, we do not fuse the whole block into a single kernel, which has been shown to be necessary in order to minimise memory movement and achieve peak performance in other SSMs, especially for the backward pass. Additionally, the KLA posterior mean recurrence can be unrolled into an equivalent matrix multiplication with a lower-triangular ``attention'' matrix; we visualise this structure in \cref{fig:attn-matrix-expansion} (Appendix \ref{app:empirical}).

\subsection{Training Loss}
We use two training losses. \textbf{(i) Posterior mean:} logits $\boldsymbol{\ell}_t = g_\theta(\mathbf{y}_t)$ are read out from the posterior mean $\mathbf{y}_t=\mathbf{q}_t\odot \boldsymbol{\mu}_t$ and trained with cross-entropy. \textbf{(ii) Log marginal likelihood:} we minimise the negative log marginal likelihood of the output token under the latent posterior (a Monte-Carlo objective), which provides an additional training signal for the posterior variances; the equations and details are given in Appendix~\ref{app:architecture}.


\section{Experiments}
We evaluate \textsc{KLA} as a \emph{drop-in sequence mixer} for language modelling.
Although \textsc{KLA} is derived from continuous-time Gaussian filtering, our goal is pragmatic: to test whether a Bayesian filtering view of selection/attention can retain the efficiency and accuracy of modern deterministic mixers, while yielding interpretable \emph{uncertainty-driven gating} behaviour, and, crucially, whether such a probabilistic primitive can be \emph{stacked} into deep networks and trained at scale.

\noindent\textbf{Experimental questions.}
We address four questions:
\begin{itemize}[leftmargin=1.5em,itemsep=2pt,topsep=2pt,parsep=0pt]
\item \textbf{(Q1) Compute and scaling:} how does the parallel implementation of \textsc{KLA} scale with sequence length, and how does it compare to a recurrent (time-stepped) Kalman implementation in wall-clock and memory?
\item \textbf{(Q2) Model quality as a single-block primitive:} as a drop-in mixer block, does \textsc{KLA} match or improve upon deterministic sub-quadratic alternatives (SSMs / linear attention) on controlled synthetic language-modelling tasks that probe core LM skills (long-range dependence, retrieval, and compression)?
\item \textbf{(Q3) State-tracking expressivity:} does the fractional-linear structure induced by information-form filtering translate into stronger state-tracking capabilities (e.g., permutation composition) than linear SSM/attention baselines, while retaining scan-parallelisability?
\item \textbf{(Q4) Stacking and scale:} can \textsc{KLA} be stacked into deep architectures and pretrained at the billion-token scale, matching attention and SSM/GLA backbones as a \emph{standalone} mixer and \emph{complementing} attention as a hybrid on zero-shot commonsense benchmarks?
\end{itemize}
Finally, we ablate architectural choices within the \our block to isolate what enables stable deep stacking of a probabilistic primitive.

\subsection{Baselines}
\label{sec:baselines}
We compare \our against softmax attention and four recent state-of-the-art sub-quadratic mixers spanning the main architectural paradigms for efficient sequence modelling.
\textbf{GPT}~\citep{vaswani2017attention}: a standard softmax-attention transformer, included as a strong $\mathcal{O}(T^2)$ reference in the pretraining setting.
(i) \textbf{Recurrent architectures.}
\textbf{mLSTM}~\citep{beck2024xlstm}: a modern LSTM variant featuring multiplicative interactions between hidden states and inputs, exponential gating, and improved training parallelisation.
(ii) \textbf{Structured state-space models.}
\textbf{Mamba}~\citep{gu2023mamba}: a widely used selective SSM with input-dependent state-transition matrices and gating, discussed in \cref{sec:background}.
\textbf{GDN} (Gated DeltaNet)~\citep{yang2024gated}: a delta-rule-based state-space model with selective gating for memory updates, combining principles from associative memory with modern SSM architectures and representing the current state of the art in SSM-based language modelling.
(iii) \textbf{Linear attention.}
\textbf{GLA} (Gated Linear Attention)~\citep{yang2023gated}: a linear-attention variant that unifies several modern linear RNN/SSM formulations, including Mamba, under a common framework.
Throughout, models are matched on parameter count and effective state size; full configurations are in Appendix~\ref{app:hyperparams}.\looseness-1

\subsection{Compute Scaling of Parallel \textsc{KLA}}
\label{sec:compute-scaling}

\begin{figure}[!t]
\centering
\includegraphics[width=\columnwidth]{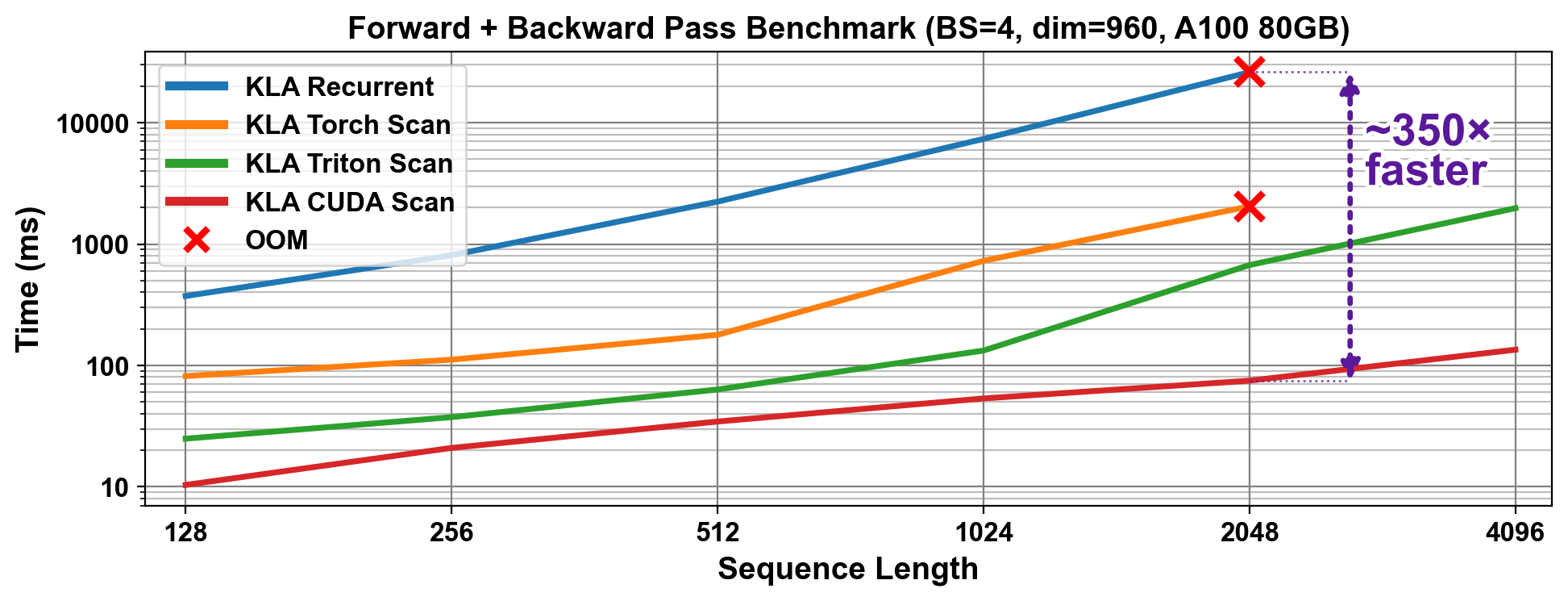}
\caption{\textbf{Compute scaling of parallel \textsc{KLA}.} Forward$+$backward (training) runtime vs.\ sequence length across the four \textsc{KLA} implementations: a naive recurrent (time-stepped) Kalman update, a Torch associative scan, a custom Triton associative-scan kernel, and a custom CUDA M\"{o}bius-scan kernel that avoids materialising expanded states in HBM. The scan-based variants match the $\mathcal{O}(\log T)$ profile of modern SSM/GLA mixers; the recurrent baseline OOMs at long sequence lengths. The forward-only (prompt-processing) curve and full setup are in Appendix~\ref{sec:fwd-scaling}.}
\label{fig:scaling}
\end{figure}

\begin{figure*}[!t]
\centering
\begin{subfigure}[c]{0.48\textwidth}
\centering
\setlength{\tabcolsep}{4pt}%
\renewcommand{\arraystretch}{1.2}%
\resizebox{\linewidth}{!}{%
\begin{tabular}{@{}lcccccc@{}}
\toprule
\textbf{Algorithm} & \makecell[c]{\textbf{Compre-}\\\textbf{ssion}} & \makecell[c]{\textbf{Memori-}\\\textbf{zation}} & \makecell[c]{\textbf{Context}\\\textbf{Recall}} & \makecell[c]{\textbf{Noisy}\\\textbf{Recall}} & \makecell[c]{\textbf{Fuzzy}\\\textbf{Recall}} & \makecell[c]{\textbf{Selective}\\\textbf{Copy}} \\
\midrule
GDN & \cellcolor{heatpurple!37}65.53 & \cellcolor{heatpurple!68}\best{99.93} & \cellcolor{heatpurple!70}\best{99.94} & \cellcolor{heatpurple!70}\best{99.95} & \cellcolor{heatpurple!53}37.04 & \cellcolor{heatpurple!59}90.30 \\
GLA & \cellcolor{heatpurple!15}49.45 & \cellcolor{heatpurple!70}\best{99.99} & \cellcolor{heatpurple!15}73.60 & \cellcolor{heatpurple!15}85.24 & \cellcolor{heatpurple!15}18.22 & \cellcolor{heatpurple!56}82.41 \\
Mamba & \cellcolor{heatpurple!55}78.35 & \cellcolor{heatpurple!68}\best{99.96} & \cellcolor{heatpurple!70}\best{99.92} & \cellcolor{heatpurple!70}\best{99.93} & \cellcolor{heatpurple!40}30.83 & \cellcolor{heatpurple!54}80.60 \\
mLSTM & \cellcolor{heatpurple!26}57.17 & \cellcolor{heatpurple!70}\best{99.98} & \cellcolor{heatpurple!70}\best{99.98} & \cellcolor{heatpurple!70}\best{99.99} & \cellcolor{heatpurple!30}25.43 & \cellcolor{heatpurple!15}36.61 \\
\textbf{KLA (Ours)} & \cellcolor{heatpurple!65}85.03 & \cellcolor{heatpurple!15}98.87 & \cellcolor{heatpurple!70}\best{99.95} & \cellcolor{heatpurple!70}\best{99.93} & \cellcolor{heatpurple!70}\best{45.70} & \cellcolor{heatpurple!59}90.67 \\
\textbf{KLA+ (Ours)} & \cellcolor{heatpurple!70}\best{88.87} & \cellcolor{heatpurple!68}\best{99.94} & \cellcolor{heatpurple!70}\best{99.94} & \cellcolor{heatpurple!70}\best{99.95} & \cellcolor{heatpurple!66}43.32 & \cellcolor{heatpurple!70}\best{91.45} \\
\bottomrule
\end{tabular}%
}
\caption{MAD-Lab accuracy (\%).}
\label{tab:accuracy-heatmap}
\end{subfigure}\hfill
\begin{subfigure}[c]{0.48\textwidth}
\centering
\includegraphics[width=\linewidth]{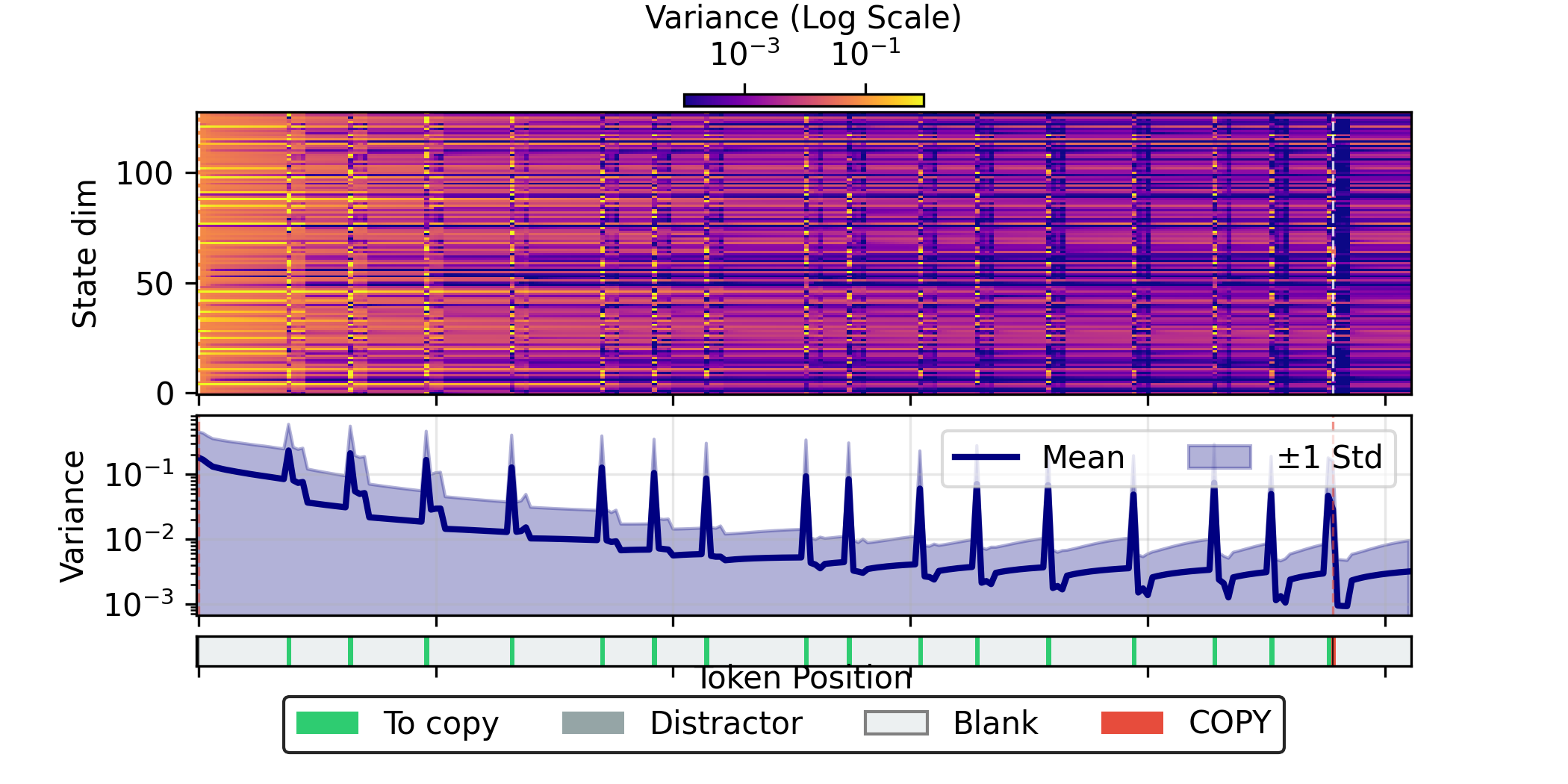}
\caption{Posterior variance on Selective Copy.}
\label{fig:posterior-var-sc}
\end{subfigure}
\caption{\textbf{\our as a single-block primitive: quality and uncertainty.} \textbf{(a)} Test accuracy (\%, $\uparrow$) on the MAD synthetic-LM suite for GDN, GLA, Mamba, mLSTM, \our, and probabilistic decoding \our+; all methods are matched on parameter count and effective state size, with darker shading $=$ higher accuracy. \textbf{(b)} Posterior variance on Selective Copy: variance decreases as evidence accumulates while sharp spikes align with copy-relevant tokens, an interpretable filtering/selection signal.}
\label{fig:madlab-quality}
\end{figure*}

Scaling probabilistic primitives to long contexts while retaining efficient training parallelism remains comparatively underexplored, and competitive throughput is won not by asymptotics alone but by mathematical reparameterisation and hardware-aware programming acting \emph{in sync}. To make this concrete, we benchmark four \textsc{KLA} implementations that progressively expose this interplay: (i) a naive recurrent (time-stepped) Kalman update; (ii) a mathematically parallel formulation via a Torch associative scan; (iii) a hardware-optimised custom Triton associative-scan kernel; and (iv) a custom CUDA kernel that implements the M\"{o}bius scan directly with a fused scan that \emph{avoids materialising the expanded (lifted) states in HBM}, keeping the recurrence in fast on-chip (SRAM) memory - the same kernel-fusion principle that underpins the selective-scan implementation of Mamba~\citep{gu2023mamba}. The scan-based implementations scale efficiently with sequence length compared to the recurrent counterpart, confirming that \emph{information-form filtering admits the same scan-parallel $\mathcal{O}(\log T)$ profile as modern SSM/GLA-style mixers}. Crucially, the Torch scan captures only the reparameterisation, whereas the CUDA kernel adds the memory-traffic savings on top: at $T{=}2048$ - the longest sequence the recurrent baseline fits before running out of memory - the CUDA kernel is roughly $350\times$ faster than the classic recurrent Kalman update, and this gap widens with $T$ (the recurrent update OOMs at $T{=}4096$). At pre-training scale such per-step gains compound across billions of tokens and thousands of optimisation steps into hours-to-days of saved GPU wall-clock, and correspondingly more efficient and sustainable use of expensive accelerator resources. The forward$+$backward (training) runtime curve is shown in \cref{fig:scaling}; the forward-only (prompt-processing) curve and full setup are in Appendix~\ref{sec:fwd-scaling}.

\subsection{\textsc{KLA} as a Single-Block Primitive on Synthetic Language Modelling}
\label{sec:synthetic-lm}

\paragraph{MAD synthetic LM suite.}
We evaluate \textsc{KLA} and the baselines as a single-block primitive on the MAD-Lab benchmark~\citep{poli2024mad}, a set of six synthetic token-manipulation tasks that test different LM skills (\cref{tab:skill-mapping-grouped}). We match parameter counts and effective state sizes and use a common encoder-decoder training protocol across all methods (dataset details in Appendix~\ref{app:datasets}; hyperparameters in Appendix~\ref{app:hyperparams}). Evaluating as a single block isolates the contribution of the \emph{update mechanism itself}, free of confounds from depth or memory budget.\looseness-1

\paragraph{Results.}
\cref{tab:accuracy-heatmap} shows that \textsc{KLA} is competitive across the suite and best-in-class precisely where \emph{selection under uncertainty} is required: it attains the top score on \textit{Fuzzy Recall} ($45.7$) and \textit{Selective Copy} ($90.7$), the latter a task popularised by Mamba~\citep{gu2023mamba}. Notably, whereas Mamba's selectivity is driven by \emph{token-dependent, time-varying} transition dynamics, \textsc{KLA} keeps its dynamics parameters ($a$, $\Delta$, $\mathbf{p}$) as time-invariant global parameters and instead derives selectivity from the \emph{uncertainty ratios} of the M\"{o}bius precision recurrence (\cref{eq:mobius_main}).
On \textit{Compression} - which measures how faithfully a model can reconstruct its context from a single final hidden-state vector - \textsc{KLA} ($85.0$) and \textsc{KLA}+ ($88.9$) lead all baselines, consistent with information-form filtering acting as a near least-squares-optimal context compressor.
The one mild exception is pure \textit{Memorization} ($98.9$ vs.\ ${\sim}99.9$ for the strongest baselines): we read this small gap as \textsc{KLA} favouring \emph{adaptive in-context filtering} over rote storage of a static key--value dictionary into weights, a sensible trade-off given its uncertainty-weighted inductive bias.
Finally, probabilistic decoding (\textsc{KLA}+, marginalisation under the learned posterior) further improves \textit{Selective Copy} and \textit{Compression} and closes the \textit{Memorization} gap, indicating that propagating the full predictive distribution - i.e.\ second-order (variance) statistics - through the loss in a principled manner yields measurable gains and is a promising avenue for future research.

\paragraph{Mechanistic behaviour: posterior uncertainty visualisations.}
To characterise how \textsc{KLA} filters information, we visualise the posterior variance on Selective Copy (\cref{fig:posterior-var-sc}). Rather than claiming calibrated uncertainty estimation, we use it as a diagnostic of the model's internal filtering dynamics. In Selective Copy, the posterior variance generally decreases over time as evidence accumulates, while sharp variance spikes align with copy-relevant positions, indicating timesteps where the model treats the input as especially informative for updating its belief.

\begin{figure*}[!t]
\centering
\begin{subfigure}[c]{0.42\textwidth}
  \centering
  \includegraphics[width=0.8\linewidth]{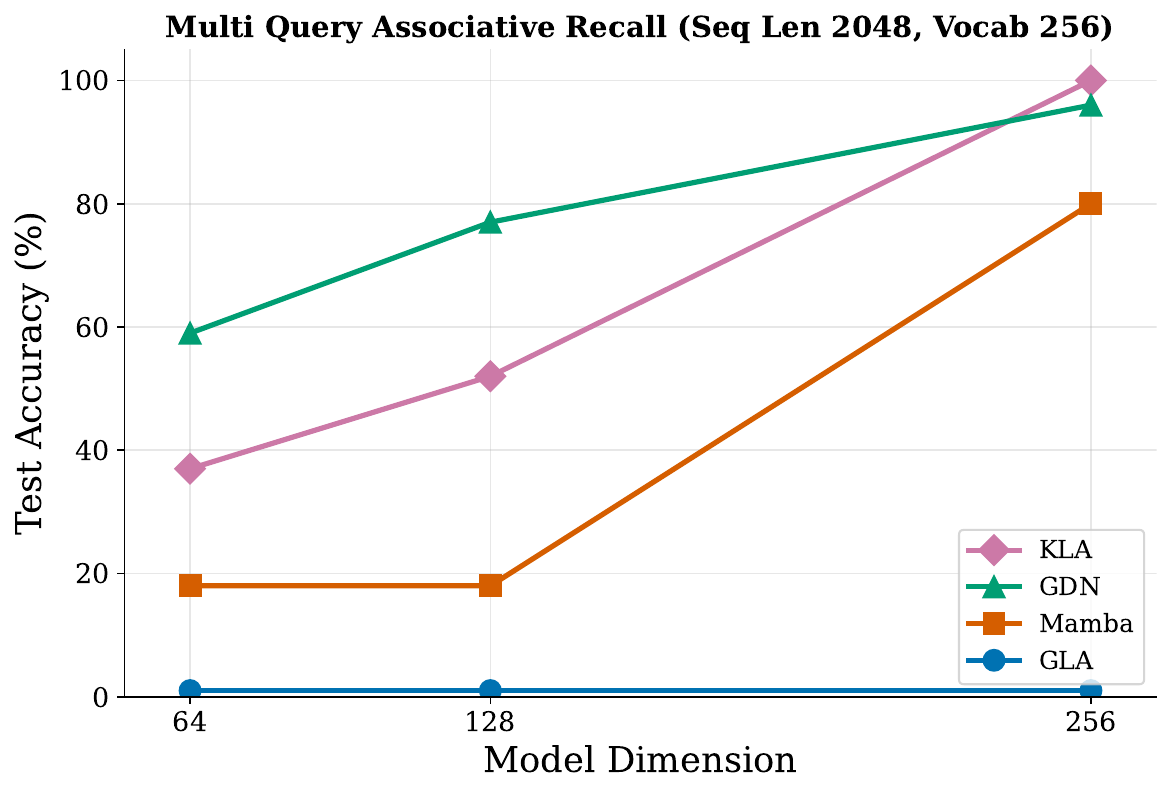}
  \caption{Long-context MQAR accuracy vs.\ dimension.}
  \label{fig:mqar-results}
\end{subfigure}\hspace{0.04\textwidth}%
\begin{subfigure}[c]{0.42\textwidth}
  \centering
  \includegraphics[width=\linewidth]{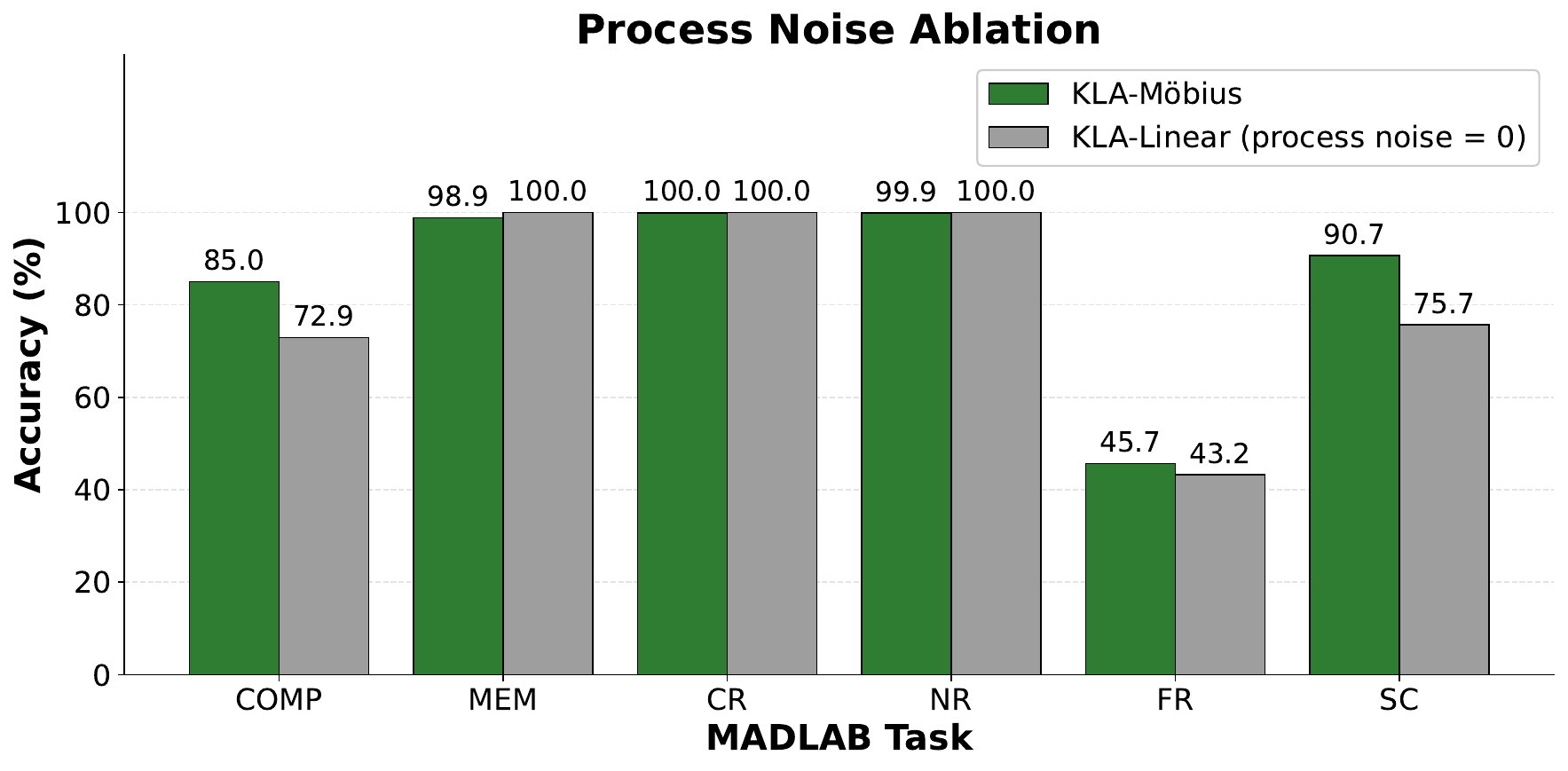}
  \caption{Process-noise ablation on MAD tasks.}
  \label{fig:process-noise-ablation}
\end{subfigure}
\caption{\textbf{Single-block analyses of \textsc{KLA}.} \textbf{(a)} Long-context MQAR ($T{=}2048$, $V{=}256$): \our reaches near-perfect accuracy ($>95\%$) at dimension $256$, outperforming GDN and Mamba, while GLA fails to learn the task in this challenging setting. \textbf{(b)} Process-noise ablation on MAD tasks: fixing $\mathbf{p}_t{=}0$ (deterministic dynamics) collapses the M\"{o}bius precision recurrence to a fixed-gate linear recurrence, degrading the selection-heavy tasks (Selective Copy, Compression) while leaving memorisation and exact recall unchanged (full per-task breakdown in \cref{tab:process-noise-ablation}, Appendix~\ref{sec:gauss-variants}).}
\label{fig:pnoise-mqar}
\end{figure*}

\paragraph{Long-context associative recall (MQAR).}
As a stress test of the single-block primitive under heavy memory load, we evaluate on Multi-Query Associative Recall (MQAR)~\citep{arora2023zoology}, which directly probes the capacity of a fixed-size recurrent state to store and retrieve many key-value associations.
Following the Zoology benchmark, we adopt a considerably harder configuration than typically studied - sequence length $T{=}2048$ with vocabulary size $V{=}256$ - which directly probes the \emph{storage-capacity} limits of fixed-size states.
\textsc{KLA} consistently outperforms Mamba across all dimensions and substantially outperforms GLA (which fails to learn the task in this extreme setting), reaching near-perfect accuracy ($>95\%$) at dimension $256$; GDN's delta rule is stronger at the lowest dimensions, where it also operates with a larger effective state.
We attribute \textsc{KLA}'s robustness to its M\"{o}bius precision updates (\cref{eq:lambda_gate_view}), which adaptively down-weight new observations once the state saturates --- an implicit, learned compression mechanism analogous to selective pattern storage in high-capacity associative memories~\citep{storkey1997hopfield}, without requiring explicit dimension scaling. The full dimension sweep is shown in \cref{fig:mqar-results}; the data configuration is given in Appendix~\ref{sec:longcontext-mqar-results}.\looseness-1

\subsection{State Tracking and Permutation-Group Expressivity}
\label{sec:a5-expressivity}

\paragraph{Why $A_5$ matters.}
A central open question is whether sub-quadratic sequence models can \emph{track state} beyond highly-parallel shortcut computation~\citep{merrill2023parallelism,merrill2024illusion}. Recent theory shows that common linear/diagonal SSMs (e.g.\ S4, Mamba), despite their recurrent form, share the bounded-depth expressivity limits of transformers, they are simulable by highly-parallel circuit classes (informally $\mathrm{TC}^0$) and cannot solve inherently sequential state-tracking whose difficulty grows with sequence length. A canonical example is the word problem for the alternating group $A_5$ (the smallest non-solvable subgroup of $S_5$), which is $\mathrm{NC}^1$-complete and serves as a minimal benchmark for \emph{hard state tracking}~\citep{merrill2024illusion}.\looseness-1

\paragraph{Result.}
Following \citet{merrill2024illusion}, \cref{fig:a5-min-layers}(a) shows that \textsc{KLA} solves the $A_5$ task at constant depth with only $1$-$2$ layers, whereas linear SSM/attention baselines require depth that grows with sequence length in this regime.
Among the parallel-trainable architectures we consider, \textsc{KLA} is the only one to solve the task at constant depth; concurrent work shows linear RNNs can also acquire state tracking through negative eigenvalues~\citep{grazzi2024unlocking}, whereas \textsc{KLA} achieves it through the nonlinearity of the M\"{o}bius precision recurrence.
This is consistent with \textsc{KLA}'s \emph{fractional-linear} (information-form) updates: each step stays scan-parallelisable, yet the state evolution is richer than a linear recurrence, placing \textsc{KLA} beyond $\mathrm{TC}^0$ while preserving the parallel-scan structure of efficient SSM/linear-attention models.\looseness-1

\subsection{Stacking \textsc{KLA} for Language Modelling at Scale}
\label{sec:pretraining}

The experiments above isolate \textsc{KLA} as a single ($1$-$2$ block) primitive. We now ask whether it can be \emph{stacked into deep networks and trained at scale}: we pretrain \textsc{KLA} and all baselines from scratch on FineWeb-Edu at \textbf{45M} ($1.8$B tokens) and \textbf{180M} ($10.9$B tokens) parameters, with matched optimiser and token budget, reporting zero-shot accuracy on eight commonsense benchmarks (full protocol in Appendix~\ref{app:hyperparams}). To our knowledge this is the first large-scale from-scratch pretraining of a stacked probabilistic-filtering primitive propagating belief states between blocks, and we regard stable training at this scale as a contribution in itself. Stacking also raises an interface question absent for deterministic mixers: \textsc{KLA} emits a full posterior at every step; as a first iteration we propagate only the \emph{mean} between blocks (a single posterior sample performs comparably), leaving full-posterior propagation to future work.

\begin{table*}[!t]
\centering
\caption{\textbf{Language modelling at two academic scales.} Zero-shot accuracy (\%) on eight commonsense benchmarks at 45M ($1.8$B tokens) and 180M ($10.9$B tokens). \best{Bold} = best within each size block, \second{underline} = second. Rows above the dashed line are standalone mixers; the row below it is the hybrid GPT$+$\textsc{KLA}, in which only the \emph{final} attention layer of GPT is replaced with a \textsc{KLA} block. Per-size comparisons against the analogous GPT$+$Mamba and GPT$+$GDN hybrids are in \cref{fig:a5-min-layers}(b).}
\label{tab:lm-main}
\small
\setlength{\tabcolsep}{4.5pt}
\renewcommand{\arraystretch}{1.15}
\begin{tabular}{@{}lccccccccc@{}}
\toprule
\textbf{Model} & \textbf{LAMB.} & \textbf{HellaS.} & \textbf{PIQA} & \textbf{Arc-E} & \textbf{Arc-C} & \textbf{WinoGr.} & \textbf{OBQA} & \textbf{BoolQ} & \textbf{Average} \\
 & acc\,$\uparrow$ & acc$_n$\,$\uparrow$ & acc\,$\uparrow$ & acc\,$\uparrow$ & acc$_n$\,$\uparrow$ & acc\,$\uparrow$ & acc\,$\uparrow$ & acc\,$\uparrow$ & acc\,$\uparrow$ \\
\midrule
GPT-45M             & \best{16.00} & 28.60 & 58.00 & 40.20 & 20.60 & 50.00 & \best{27.40} & 37.40 & 34.78 \\
Mamba-45M           & \second{9.20} & \best{29.20} & \best{59.80} & 37.20 & \best{25.40} & 48.80 & 25.40 & 61.60 & 37.08 \\
GDN-45M             & \best{16.00} & 28.20 & 58.40 & 39.00 & \second{23.40} & \best{52.60} & \second{27.20} & 45.20 & 36.25 \\
KLA-45M             & 7.60 & \second{28.80} & \second{58.60} & \second{40.60} & 23.20 & \second{52.20} & \best{27.40} & \best{63.00} & \second{37.68} \\
\hdashline\noalign{\vskip 1pt}
\textbf{GPT+KLA-45M} & \best{16.00} & 28.20 & \second{58.60} & \best{42.80} & 23.00 & 50.80 & 24.40 & \second{62.80} & \best{38.33} \\
\midrule
GPT-180M             & \best{28.80} & 32.20 & \second{63.40} & 51.00 & 27.00 & 49.20 & 31.20 & 54.80 & 42.20 \\
Mamba-180M           & 24.00 & 30.60 & 63.00 & 51.00 & 27.20 & \best{53.80} & \best{33.00} & 55.80 & 42.30 \\
GDN-180M             & 27.80 & 31.60 & 60.40 & 45.60 & 25.80 & 53.40 & 30.00 & \second{60.80} & 41.93 \\
KLA-180M             & 22.40 & \second{34.00} & \second{63.00} & 50.00 & 26.00 & \second{53.60} & \second{32.40} & \second{60.80} & \second{42.78} \\
\hdashline\noalign{\vskip 1pt}
\textbf{GPT+KLA-180M} & \best{30.20} & \best{34.40} & \best{64.20} & \best{55.80} & \best{29.00} & 51.80 & 31.00 & \best{63.20} & \best{44.95} \\
\bottomrule
\end{tabular}
\end{table*}

\paragraph{Standalone \textsc{KLA} is competitive.}
\cref{tab:lm-main} shows \textsc{KLA} is competitive with softmax attention, Mamba, and GDN at both scales (with per-task strengths on HellaSwag and PIQA). Its advantage is largest at shallow depth and narrows as baselines add layers (in our $24$-layer experiments), as predicted by our expressivity analysis (\cref{sec:a5-expressivity}). It underperforms on the \emph{clean} copying task LAMBADA, consistent with uncertainty-weighted updates favouring robust over verbatim recall (cf.\ \cref{sec:synthetic-lm}).\looseness-1

\paragraph{A single \textsc{KLA} layer improves a GPT.}
Unlike prior hybrids that interleave many layers, we replace only the \emph{final} layer of a softmax-attention GPT with an SSM-like block (\textsc{KLA}, Mamba, or GDN), motivated by evidence that uncertainty concentrates in later layers~\citep{agarwal2025bayesian}. At $180$M, GPT$+$\textsc{KLA} improves over pure GPT and over the analogous GPT$+$Mamba and GPT$+$GDN hybrids on average and most tasks (\cref{tab:lm-main}; \cref{fig:a5-min-layers}(b)), showing \textsc{KLA} is an effective drop-in \emph{complement} to attention, not merely a substitute.\looseness-1

\subsection{Ablations}
\label{sec:ablations}
Finally, we isolate the two modelling choices that matter most. Stable \emph{deep stacking} of a probabilistic primitive was possible \emph{only} because of the mean-reverting Ornstein-Uhlenbeck (OU) prior (\cref{sec:stochastic_dynamics}, \cref{fig:ou-hybrid-tall}): naive stacking was unstable and standard fixes (gradient clipping, lower learning rates) did not help at the root, whereas the OU prior keeps the latent dynamics bounded, letting \textsc{KLA} stack to $12$ and $24$ blocks stably with no gradient explosions (gradient clipping $3.0$ for all models; \cref{fig:ou-ablation} confirms improved stability and accuracy on Selective Copy, especially at depth) - to our knowledge the first stable backpropagation through stacked blocks that propagate belief states between layers and through time. A second ablation (\cref{fig:process-noise-ablation}, Appendix~\ref{sec:gauss-variants}) isolates the process-noise parameter $\mathbf{p}_t$ as the source of \textsc{KLA}'s nonlinear gating: $\mathbf{p}_t$ enters only through the M\"{o}bius precision recurrence (\Cref{Th:Mobius}), via the state-dependent factor $\boldsymbol{\rho}_t = \mathbf{1}\oslash(\bar{\mathbf{a}}_t^2 + \bar{\mathbf{p}}_t\odot\boldsymbol{\lambda}_{t-1})$, so setting $\mathbf{p}_t{=}0$ freezes $\boldsymbol{\rho}_t$ to a constant and collapses the recurrence to a fixed-gate linear update. This costs both expressivity (the gate no longer adapts to accumulated history) and stability: without injected process noise the posterior precision $\boldsymbol{\lambda}_t$ accumulates unbounded, making the filter overconfident, whereas a non-zero $\mathbf{p}_t$ caps it and maintains a fading memory. Removing process noise leaves memorisation and exact recall essentially unchanged but degrades the selection-heavy tasks (Selective Copy $-14.9$, Compression $-12.1$ points), a $4.7$-point average drop (\cref{tab:process-noise-ablation}).\looseness-1

\section{Conclusion}
In this paper, we propose a scalable probabilistic sequence-modelling primitive, the \our layer, which treats language modelling as a state estimation problem from noisy observations. We show that posterior inference in these models is tractable for very long sequences via parallel scans, and that the resulting gating mechanisms are strictly more expressive than those of modern deterministic linear SSMs and RNNs. Our experiments indicate the promise of \our as a scalable primitive within the language and probabilistic sequence-modelling toolbox.\looseness-1

\section{Limitations and Future Work}
\label{sec:limitations}
This paper is a \emph{first iteration} on a fundamental algorithmic family - exact Bayesian filters as parallelisable sequence-modelling primitives - and we deliberately start from the simplest possible assumptions: diagonal, linear time-invariant (LTI) dynamics with a linear-Gaussian observation model. Accordingly, the focus is to isolate what the Bayesian filtering update contributes in a minimal but impactful setting; most experiments use shallow configurations rather than deep, large-scale pretraining, and we do not claim state-of-the-art performance on web-scale corpora, very long-context regimes, or open-ended generative benchmarks.

This minimal starting point yields a clear capability profile. \textsc{KLA} is already directly usable where its strengths are structural: state tracking and permutation composition, associative recall (MQAR), and, as a single drop-in block, complementing attention in a hybrid with gains that carry over to zero-shot commonsense reasoning. Its current headroom lies in standalone long-context retrieval and raw next-token quality, which we attribute to the deliberately minimal, time-invariant dynamics of this first iteration; relaxing these assumptions is a natural direction for future members of the family.

We have also not explicitly exploited the posterior covariance for applications where uncertainty quantification may be critical (e.g., epistemic prompt uncertainty, hallucination detection, and out-of-distribution prompt detection), which are likely areas where \our could excel; we leave this for future work. Finally, because the underlying state-space formulation is inherently continuous, a promising direction is to move beyond the language domain to vision and multimodal foundation models over continuous modalities, where an explicit belief-state uncertainty may be especially natural.

\section*{Impact Statement}
This paper presents work whose goal is to advance the field of Machine Learning. There are many potential societal consequences of our work, none of which we believe must be specifically highlighted here. The \our layer provides principled uncertainty quantification, which could benefit applications requiring calibrated predictions, though, as with any language modelling primitive, deployment should consider potential misuse.

\section*{Acknowledgements}
We thank Henry Gouk, Nikolay Malkin, Thomas Lee, and Stephan Kostov for their valuable feedback on this draft at various stages. We also thank Wouter Boomsma, Frederikke Isa Marin, Jose Cano Reyes, and Jude Harris for the informative discussions over the course of this research. This project has received funding from the European Union’s Horizon Europe research and innovation programme under grant agreement No. 101120726. This work was funded by UK Research and Innovation (UKRI) under the UK government’s Horizon Europe funding Guarantee 10085198. We also gratefully acknowledge the support of the UK Engineering and Physical Sciences Research Council (EPSRC), under grant EP/W002876/1 and through PhD Studentships within the CDT in Machine Learning Systems, School of Informatics, University of Edinburgh (EP/Y03516X/1). The authors acknowledge the use of resources provided by the Isambard-AI National AI Research Resource (AIRR)~\citep{mcintosh2024isambard}. Isambard-AI is operated by the University of Bristol and is funded by the UK Government's Department for Science, Innovation and Technology (DSIT) via UK Research and Innovation; and the Science and Technology Facilities Council [ST/AIRR/I-A-I/1023].

\bibliography{iclr2026_conference}
\bibliographystyle{icml2026}

\newpage
\appendix
\onecolumn
\section*{Appendix}

\subsection*{Appendix outline}

\begin{itemize}
    \item \hyperref[app:architecture]{Appendix A: Architecture Details} (p.~\pageref{app:architecture})
    \begin{itemize}
        \item \hyperref[fig:gauss_mamba_arc]{Figure 5: Model Architecture Diagram} (p.~\pageref{fig:gauss_mamba_arc})
        \item \hyperref[alg:kla]{Algorithm 1: Kalman Linear Attention} (p.~\pageref{alg:kla})
        \item \hyperref[app:loss]{Training Loss}
    \end{itemize}

    \item \hyperref[app:notations]{Appendix B: Notation} (p.~\pageref{app:notations})

    \item \hyperref[app:ex-background]{Appendix C: Extended Background} (p.~\pageref{app:ex-background})
    \begin{itemize}
        \item \hyperref[sec:mobius]{Mobius / Fractional-Linear Transformations} (p.~\pageref{sec:mobius})
        \item \hyperref[sec:information-form]{Information Form of Gaussian Distributions} (p.~\pageref{sec:information-form})
        \item \hyperref[app:primer-nlp-kf]{Kalman and Information Filters} (p.~\pageref{app:primer-nlp-kf})
    \end{itemize}

    \item \hyperref[app:theorems-proofs]{Appendix D: Theorems and Proofs} (p.~\pageref{app:theorems-proofs})

    \item \hyperref[app:empirical]{Appendix E: Additional Empirical Results} (p.~\pageref{app:empirical})
    \begin{itemize}
        \item \hyperref[sec:gauss-variants]{Ablation On The Importance Of Process Noise Parameter} (p.~\pageref{sec:gauss-variants})
        \item \hyperref[sec:fwd-scaling]{Forward Pass Runtime Scaling} (p.~\pageref{sec:fwd-scaling})
        \item \hyperref[sec:attn-matrix]{Equivalent Attention Matrix} (p.~\pageref{sec:attn-matrix})
        \item \hyperref[sec:attention-viz]{Attention Map Visualisation} (p.~\pageref{sec:attention-viz})
    \end{itemize}

    \item \hyperref[app:datasets]{Appendix F: Datasets and Benchmarks Used} (p.~\pageref{app:datasets})
    \begin{itemize}
        \item \hyperref[sec:mad-descriptions]{MAD LM Suite} (p.~\pageref{sec:mad-descriptions})
        \item \hyperref[sec:longcontext-mqar]{Long-Context MQAR} (p.~\pageref{sec:longcontext-mqar})
    \end{itemize}

    \item \hyperref[app:hyperparams]{Appendix G: Hyperparameters Used} (p.~\pageref{app:hyperparams})
    \begin{itemize}
        \item \hyperref[sec:exp-protocol]{Experimental Protocol} (p.~\pageref{sec:exp-protocol})
        \item \hyperref[sec:training-hyperparams-sub]{Training Hyperparameters} (p.~\pageref{sec:training-hyperparams-sub})
        \item \hyperref[sec:madlab-hyperparams]{MAD-Lab Hyperparameters} (p.~\pageref{sec:madlab-hyperparams})
        \item \hyperref[sec:mqar-hyperparams]{MQAR Hyperparameters} (p.~\pageref{sec:mqar-hyperparams})
        \item \hyperref[sec:a5-hyperparams]{A5 State Tracking Hyperparameters} (p.~\pageref{sec:a5-hyperparams})
    \end{itemize}
\end{itemize}

\vspace{0.5cm}
\newpage

\section{Architecture Details}
\label{app:architecture}

\begin{figure}[h]
  \vspace{-8pt}
  \hspace{-.0cm}
  \includegraphics[width=1.0\linewidth]{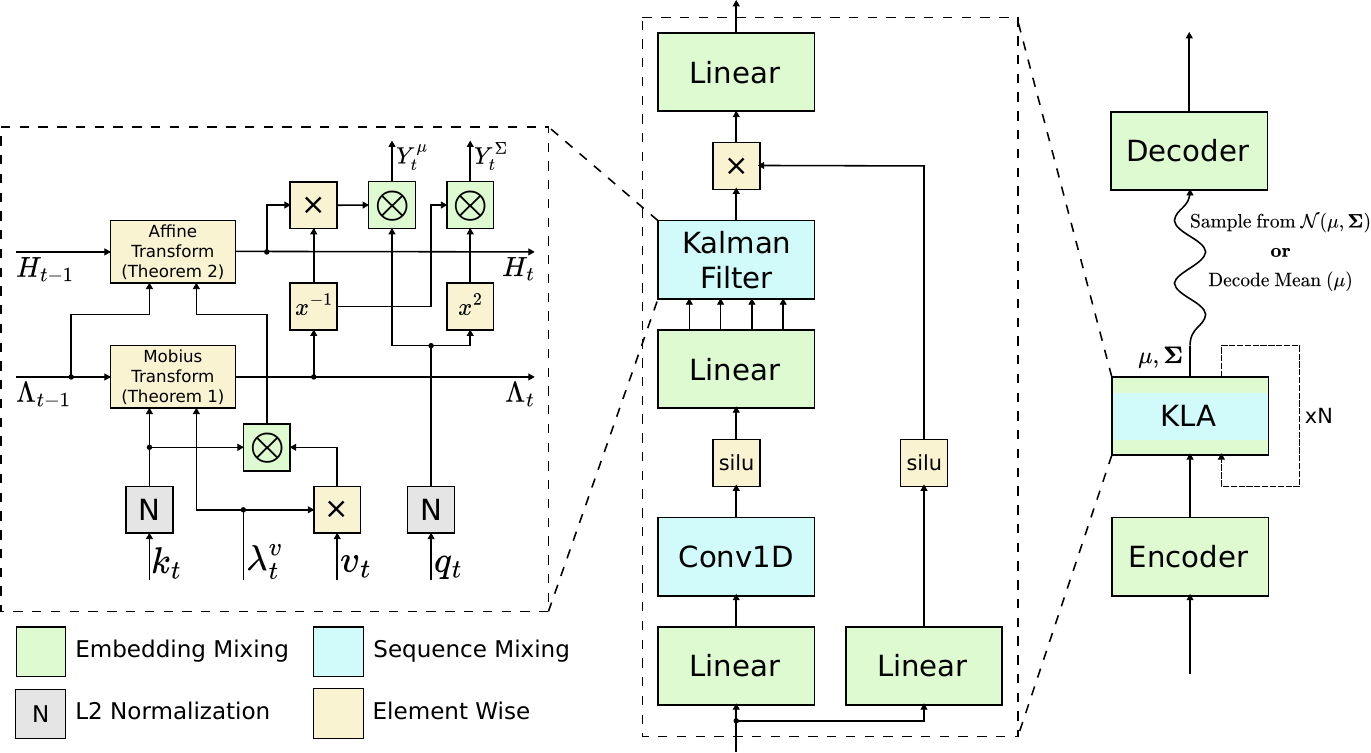}
  \caption{\textbf{Model architecture used in our experiments.}
  The KLA block is a drop-in replacement for Transformer or SSM blocks. We follow a similar block structure to Mamba, fusing the sequence mixing and gated MLP into a single block.
  We additionally employ QK-Norm and use an expansion factor of 1.
  Tokens are embedded and passed through a \emph{sequence mixer}, then decoded to logits.
  The sequence mixer is a drop-in primitive that can be instantiated as
  attention variants, or other SSM/Filter modules.
  Many mixers are paired with standard \emph{scaffolding} - e.g., 1D convolutions, residual/skip connections, and simple nonlinearities (elementwise multiplication or activations) - shown on the right.
  }
  \label{fig:gauss_mamba_arc}
  \vspace{-6pt}
\end{figure}

\begin{algorithm}[tb]
   \caption{Kalman Linear Attention}
   \label{alg:kla}
\begin{algorithmic}[1]
   \STATE {\bfseries Input:}  $\mathbf{X} : (\mathtt{B, T, 1, D})$\\
   \STATE {\bfseries Output} $ \mathbf{Y^\mu} : (\mathtt{B, T, 1, D}) $ \\
   \STATE {\bfseries Output (Optional)} $ \mathbf{Y^\Sigma} : (\mathtt{B, T, 1, D}) $
   \STATE $\mathbf{A, P, \Delta} : (\mathtt{N, D}) \gets$ \texttt{Parameters} \shape{A: Gating/decay, P: Process noise, $\Delta$: Discretisation step size}
   \STATE $\mathbf{K} : (\mathtt{B, T, N, 1}) \gets f_{\theta_{k}}(\mathbf{X})$ \shape{Observation operator}
   \STATE $\mathbf{Q} : (\mathtt{B, T, 1, N}) \gets f_{\theta_{q}}(\mathbf{X})$ \shape{Readout operator}
   \STATE $\mathbf{V, \Lambda^{\mathrm{v}}} : (\mathtt{B, T, 1, D}) \gets f_{\theta_{v,\lambda^{\mathrm{v}}}}(\mathbf{X})$ \shape{Observation mean and precision, $\lambda^{\mathrm{v}} \in \mathbb{R}^{+}$}
   \STATE $\mathbf{\bar A} : (\mathtt{N, D}) \gets \texttt{OU-Discretisation}(\mathbf{A, \Delta})$ \shape{See \cref{eq:discretisation}}
   \STATE $\mathbf{\bar P} : (\mathtt{N, D}) \gets \texttt{OU-Discretisation}(\mathbf{A, P, \Delta})$
   \STATE $\mathbf{\Lambda} : (\mathtt{B, T, N, D}) \gets \texttt{Mobius-Scan}(\mathbf{K, \Lambda^{\mathrm{v}}, \bar P, \bar A})$ \shape{matrix form: \cref{eq:kla-prec-mat}}
   \STATE $\mathbf{H} : (\mathtt{B, T, N, D}) \gets \texttt{Linear-Scan}(\mathbf{K, \Lambda^{\mathrm{v}}, V, \Lambda, \bar P, \bar A})$       \shape{information mean $\boldsymbol{\eta}$; matrix form: \cref{eq:kla-mean-mat}}
   \STATE $\mathbf{Y^\mu} : (\mathtt{B, T, 1, D}) \gets \mathbf{Q} (\mathbf{H} \mathbf{\Lambda}^{-1}) $
   \IF{\text{Decode Variance}}
   \STATE $\mathbf{Y^\Sigma} : (\mathtt{B, T, 1, D}) \gets \mathbf{Q}^2 \mathbf{\Lambda}^{-1} $
   \STATE {\bfseries Return:} $\mathbf{Y^\mu}, \mathbf{Y^\Sigma}$
   \ELSE
   \STATE {\bfseries Return:} $\mathbf{Y^\mu}$
   \ENDIF
\end{algorithmic}
\end{algorithm}

\paragraph{State-expanded matrix form of the scans.}
The \texttt{Mobius-Scan} and \texttt{Linear-Scan} steps are the state-expanded ($N{>}1$) matrix versions of the per-channel recursions \cref{eq:mobius_main,eq:mu-update}. Adopting the column-vector convention of \cref{tab:online-learning} ($\mathbf{k}_t\in\mathbb{R}^{N}$, $\mathbf{v}_t,\boldsymbol{\Lambda}_t^{\mathrm v}\in\mathbb{R}^{D}$, information-mean state $\mathbf{H}_t\!=\!\boldsymbol{\eta}_t$ and precision $\boldsymbol{\Lambda}_t$ in $\mathbb{R}^{N\times D}$; see footnote~\ref{fn:notation}), they read
\begin{align}
\boldsymbol{\Lambda}_t
&= \frac{(\mathbf{1}+\bar{\mathbf{P}}\odot\boldsymbol{\Phi}_t)\odot\boldsymbol{\Lambda}_{t-1} + \bar{\mathbf{A}}^2\odot\boldsymbol{\Phi}_t}{\bar{\mathbf{P}}\odot\boldsymbol{\Lambda}_{t-1} + \bar{\mathbf{A}}^2},
\qquad \boldsymbol{\Phi}_t = \mathbf{k}_t^2\,(\boldsymbol{\Lambda}_t^{\mathrm v})^{\!\top},
\label{eq:kla-prec-mat}\\[2pt]
\mathbf{H}_t
&= \mathbf{F}_t\odot\mathbf{H}_{t-1} + \mathbf{k}_t\,(\boldsymbol{\Lambda}_t^{\mathrm v}\odot\mathbf{v}_t)^{\!\top},
\qquad \mathbf{F}_t = \bar{\mathbf{A}}\oslash\bigl(\bar{\mathbf{A}}^2+\bar{\mathbf{P}}\odot\boldsymbol{\Lambda}_{t-1}\bigr),
\label{eq:kla-mean-mat}
\end{align}
where $\odot,\oslash$ act elementwise on $N\times D$ arrays, and the rank-one outer products $\mathbf{k}_t^2(\boldsymbol{\Lambda}_t^{\mathrm v})^{\!\top}$ and $\mathbf{k}_t(\boldsymbol{\Lambda}_t^{\mathrm v}\odot\mathbf{v}_t)^{\!\top}$ supply the per-slot evidence (mirroring the $\mathbf{k}_t\mathbf{v}_t^\top$ term of \cref{tab:online-learning}). Here $\mathbf{H}_t=\boldsymbol{\eta}_t$ is the \emph{information} (natural-parameter) mean; the \emph{moment-form} state of \cref{tab:online-learning} is $\mathbf{S}_t=\mathbf{H}_t\oslash\boldsymbol{\Lambda}_t$ (i.e.\ $\boldsymbol{\mu}_t=\boldsymbol{\eta}_t\oslash\boldsymbol{\lambda}_t$), and the posterior mean reads out as $\mathbf{y}_t=\mathbf{q}_t\mathbf{S}_t$. Setting $N{=}1$ recovers the scalar recursions \cref{eq:mobius_main,eq:mu-update}.

\paragraph{Architecture.}
\label{par:arch}
We use a consistent language modelling architecture across all models and baselines, as illustrated in \cref{fig:gauss_mamba_arc}. The architecture follows a standard pattern: input tokens are embedded, passed through a sequence mixer layer, and then decoded to output logits. The sequence mixer serves as a drop-in primitive that can be instantiated as \our, Mamba, GLA, or other SSM/attention variants.

Following the Mamba design, we employ standard scaffolding components paired with the sequence mixer: a 1D causal convolution with kernel size 4 and SiLU activation, along with residual connections and simple nonlinearities (elementwise multiplication). We use standard state expansion (see Section \ref{sec:background}); the state-expansion factor $N$ (sometimes called \texttt{d\_state} in SSM codebases) is set to 16 by default, following \citet{gu2023mamba}. More extensive exploration of scaffolding configurations optimally suited to the \our filtering mixer can be a direction for future work.

We implemented \our with PyTorch \citep{paszkePyTorchImperativeStyle2019} and used the Adam optimiser \citep{KingmaB14} for training the models.

\paragraph{Training loss.}
\label{app:loss}
We consider two training losses. Throughout, the decoder is applied to the query-conditioned readout
$\mathbf{y}_t=\mathbf{q}_t\odot \mathbf{z}_t$; when the readout noise is zero this is a deterministic linear map, so conditioning on $\mathbf{z}_t$ or $\mathbf{y}_t$ is equivalent up to this transformation.

\textbf{(i) Posterior mean:} logits $\boldsymbol{\ell}_t = g_\theta(\mathbf{y}_t)$ are obtained from the posterior-mean readout $\mathbf{y}_t=\mathbf{q}_t\odot \boldsymbol{\mu}_t$ and trained with cross-entropy.\looseness-1

\textbf{(ii) Log marginal likelihood:} minimises the negative log marginal likelihood of the output token $o_t$ under the latent posterior,
\begin{align}
-\log p_\theta(o_t \mid \mathbf{X})
&=
-\log \int \underbrace{p_\theta(o_t \mid \mathbf{y}_t)}_{\text{decoder}} \underbrace{p_\theta(\mathbf{z}_t \mid \mathbf{X})}_{\text{\our posterior}}\, d\mathbf{z}_t \nonumber\\
&\approx\;
-\log\!\left(\frac{1}{S}\sum_{s=1}^S p_\theta(o_t \mid \mathbf{y}_t^{(s)})\right),
\label{eq:nmll_integral}
\end{align}
with $\mathbf{z}_t^{(s)}\sim p_\theta(\mathbf{z}_t\mid \mathbf{X})$ and $\mathbf{y}_t^{(s)}=\mathbf{q}_t\odot \mathbf{z}_t^{(s)}$.
Equivalently, the estimator can be written as a $\logsumexp$:
\begin{align}
&-\log p_\theta(o_t \mid \mathbf{X})\notag \\
&\;\approx\;
-\log\!\left(\frac{1}{S}\sum_{s=1}^S p_\theta(o_t \mid \mathbf{y}_t^{(s)})\right) \notag \\
&\;=\;
-\logsumexp_{s=1}^S\!\Bigl(\log p_\theta(o_t \mid \mathbf{y}_t^{(s)})\Bigr) \;+\; \log S,
\label{eq:nmll_logsumexp}
\end{align}
which reduces to cross-entropy when $S{=}1$. We show this for a single time-step; it can be trivially extended across timesteps.

\newpage

\section{Notation}
\label{app:notations}
Throughout the paper we use the following convention:
\begin{tcolorbox}[colback=white,colframe=black!40,boxrule=0.4pt]
\begin{tabular}{@{}p{0.22\linewidth}p{0.73\linewidth}@{}}
\toprule
\textbf{Symbol} & \textbf{Meaning} \\
\midrule
$t$ & Time index \\
$z_t$ & Latent state at time $t$ \\
$v_t$ & Value (observed token/features) at time $t$ \\
$A_t$ & State transition matrix \\
$P_t$ & Process noise covariance (discretized: $\bar P_t$) \\
$C_t$ & Observation (likelihood) matrix \\
$\Lambda^{\mathrm{v}}_t$ & Value precision (inverse observation variance) \\
$\mathcal N(\mu,\Sigma)$ & Gaussian in \emph{moment} form \\
$\mathcal N(\eta,\Lambda)$ & Gaussian in \emph{canonical} form; $\;\eta=\Lambda\mu,\;\Lambda=\Sigma^{-1}$ \\
$\mu_t^{\mathrm{prior}},\Sigma_t^{\mathrm{prior}}$ & Predictive prior mean/cov at $t$ (canonical: $\eta_t^{\mathrm{prior}},\Lambda_t^{\mathrm{prior}}$) \\
$\mu_t,\Sigma_t$ & Posterior mean/cov at $t$ (canonical: $\eta_t,\Lambda_t$) \\
\midrule
\multicolumn{2}{@{}l@{}}{\textbf{Beliefs (linear--Gaussian case):}} \\
\multicolumn{2}{@{}l@{}}{$\displaystyle z_t \mid v_{1:t-1} \sim \mathcal N(\mu_t^{\mathrm{prior}},\,\Sigma_t^{\mathrm{prior}})\quad$ (predictive prior)} \\
\multicolumn{2}{@{}l@{}}{$\displaystyle z_t \mid v_{1:t} \sim \mathcal N(\mu_t,\,\Sigma_t)\quad$ (posterior)} \\
\bottomrule
\end{tabular}
\end{tcolorbox}

\noindent\textbf{Diagonal (per-dimension) specialisation.}
When matrices are diagonal, we identify them with their diagonal vectors and switch to bold lowercase:
$A_t\!\equiv\!\mathbf{a}_t$, $B_t\!\equiv\!\mathbf{b}_t$, $K_t\!\equiv\!\mathbf{k}_t$ (key), $P_t\!\equiv\!\mathbf{p}_t$ (process noise); value precision $\Lambda_t^{\mathrm{v}}$ is treated as a scalar.
All algebra becomes \emph{elementwise}: products use the Hadamard operator $\odot$, superscripts like $\mathbf{a}_t^{2}$ denote elementwise squares, and $(\cdot)^{-1}$ on lowercase vectors means elementwise inverse.
\newcommand{\had}{\odot}

\begin{center}
\small
\begin{tabular}{@{}lll@{}}
\toprule
\textbf{Full (matrix)} & \textbf{Diagonal notation (this paper)} & \textbf{Implication (elementwise)}\\
\midrule
$A_t$ & $\mathbf{a}_t$ & $(A_t \mathbf{x})=\mathbf{a}_t\had \mathbf{x}$;\quad $A_t^2 \Rightarrow \mathbf{a}_t^{2}$ \\
$K_t$ & $\mathbf{k}_t$ & $K_t \Sigma K_t^\top \Rightarrow \mathbf{k}_t^{2}\had \Sigma$ (key) \\
$P_t$ & $\mathbf{p}_t$ & Process noise cov.; $\bar P_t \Rightarrow \mathbf{\bar p}_t$ (discretized) \\
$\Lambda^{\mathrm{v}}_t$ & $\Lambda_t^{\mathrm{v}}$ & Scalar value precision (variance $= (\Lambda_t^{\mathrm{v}})^{-1}$) \\
\bottomrule
\end{tabular}
\end{center}

\section{Extended Background}
\label{app:ex-background}
\subsection{M\"{o}bius / Fractional-Linear Transformations.}
\label{sec:mobius}
A M\"{o}bius (a.k.a.\ fractional-linear) transformation is the map
\[
z \;\mapsto\; \frac{a z + b}{c z + d}, \qquad ad - bc \neq 0.
\]
It generalises the familiar affine form ($c=0$ gives $az+b$) by introducing an additional denominator term.
Each transform is represented (up to multiplication by a nonzero scalar) by a $2\times2$ matrix
$M=\begin{psmallmatrix}a & b \\ c & d\end{psmallmatrix}$ and acts on a scalar $z$ via the linear--fractional rule
\[
M(z) \;\coloneqq\; \frac{a z + b}{c z + d}.
\]
Composing multiple transforms amounts to multiplying their representing matrices, so associativity and invertibility (whenever $\det M\neq 0$) follow directly.
Geometrically, M\"{o}bius transforms can be viewed as compositions of translations, rotations/scalings, and a crucial nonlinear ingredient---\emph{inversion}---which affine transformations cannot capture.
The denominator also provides a built-in \emph{self-normalisation}: as the state grows, it is automatically rescaled, producing saturation and stability effects reminiscent of gating or normalisation layers, but achieved directly through the update itself.

\subsection{Information Form of Gaussian Distributions and Information Filters.}
\label{sec:information-form}
Gaussians belong to the exponential family~\citep{nielsen2009statistical}, so the posterior can be written in \emph{canonical} (information) form with precision $\Lambda_t \!=\! \Sigma_t^{-1}$ and natural parameter $\eta_t \!=\! \Lambda_t \mu_t$.
Filtering recursions in this parameterisation are known as \emph{information filters (IF)}~\citep{anderson2005optimal,khan2005matrix}.
They are algebraically equivalent to the standard (moment) Kalman updates; only the representation of the Gaussian belief differs.
Crucially, the canonical posterior updates are \emph{linear--additive} in $(\eta,\Lambda)$, a property we leverage for scalable sequence modelling in \cref{sec:method}.

\subsection{Control vs Observation View Of Token Processing}
\label{app:control-obs-view}

State-space models can be understood from two complementary perspectives (\cref{fig:control-obs-view}):

\paragraph{Observation view (traditional SSM/Kalman filter).}
In classical Bayesian filtering, input tokens are treated as \emph{noisy observations} of a hidden state.
The model assumes: (i) a latent state $z_t$ evolves over time via learned dynamics, and (ii) each token $v_t$ is a noisy measurement generated from this state.
The goal is to \emph{infer} the hidden state from the observed tokens---asking ``what underlying state could have produced these observations?''

\paragraph{Control view (Mamba/S6).}
In contrast, Mamba treats input tokens as \emph{control signals} that directly drive the state.
There is no observation noise or inference step; tokens deterministically update the hidden state via $z_t = A_t z_{t-1} + B_t u_t$.
The state simply accumulates information from inputs---asking ``how should I update my memory given this new input?''

\paragraph{Why this matters.}
The observation view naturally provides uncertainty quantification: since tokens are noisy measurements, the model maintains a posterior \emph{distribution} over states rather than a point estimate.
This is the foundation of \our---we treat language modelling as Bayesian filtering where tokens inform our belief about an underlying semantic state.

\begin{figure}[h]
\centering
\begin{subfigure}[t]{0.48\textwidth}
\centering
\resizebox{\linewidth}{!}{\tikzBGSS}
\caption{Observation view: tokens $v_t$ are noisy observations of latent state $z_t$. Blue arrows show inference direction.}
\end{subfigure}
\hfill
\begin{subfigure}[t]{0.48\textwidth}
\centering
\resizebox{\linewidth}{!}{\tikzMambaSSM}
\caption{Control view: tokens $u_t$ directly drive state updates by forward simulation rather than posterior inference.}
\end{subfigure}
\caption{Two views of state-space sequence models. \textbf{Left:} Bayesian/observation view - tokens are measurements, inference recovers hidden state. \textbf{Right:} Control view (Mamba) - tokens are inputs that deterministically update state.}
\label{fig:control-obs-view}
\end{figure}

\subsection{Kalman and Information Filters}
\label{app:primer-nlp-kf}

\subsubsection{Bayesian filtering as posterior inference.}

Bayesian filters (Kalman, information-filter variants) are \emph{inference schemes}---procedures that compute
posterior beliefs about latent states from noisy data.
The Kalman filter assumes a Gaussian linear state-space model:
the latent evolves via a (possibly time-varying) linear transition \(A_t\) with unmodelled variability captured by process noise \(P_t\),
and each observed token feature \(v_t\) is generated by an observation model \(C_t\) with observation noise covariance \(\Sigma^{\mathrm{obs}}_t\).
The filter then computes \(p(z_t\mid v_{1:t})\) sequentially by combining the dynamics prior with token evidence.

In this linear--Gaussian setting, inference has closed-form solutions:
in the \emph{moment parametrisation} \((\mu_t,\Sigma_t)\) the scheme is called the \textbf{Kalman filter},
while in the \emph{canonical parametrisation} \((\eta_t,\Lambda_t)\) it is called the \textbf{information filter}.
Both are algebraically equivalent, differing only in how Gaussian beliefs are represented.

\vspace{4pt}
\begin{table}[h]
\centering
\caption{Two equivalent views at time \(t\): moment (KF) vs.\ canonical (IF). Superscript ``prior'' denotes the predictive prior; unadorned symbols denote the posterior.}
\label{tab:kf-if-side-by-side}
\setlength{\tabcolsep}{6pt}
\renewcommand{\arraystretch}{1.35}
\small
\begin{tabular}{p{0.31\textwidth} p{0.31\textwidth} p{0.31\textwidth}}
\toprule
\textbf{Belief / Parameters}
& \textbf{Kalman Filter (moment form)}
& \textbf{Information Filter (canonical form)}\\
\midrule
\textbf{Prior / Prediction} \newline
$\displaystyle z_t \mid v_{1:t-1} \sim \mathcal N(\mu_t^{\mathrm{prior}},\,\Sigma_t^{\mathrm{prior}})\quad$
&
\(\mu_t^{\mathrm{prior}} = A_t \mu_{t-1}\) \newline
\(\Sigma_t^{\mathrm{prior}} = A_t \Sigma_{t-1} A_t^\top + P_t\)
&
\(\displaystyle \eta_t^{\mathrm{prior}} = \Lambda_t^{\mathrm{prior}} A_t \Lambda_{t-1}^{-1} \eta_{t-1}\)
\newline
\(\displaystyle \Lambda_t^{\mathrm{prior}} = \big(A_t \Lambda_{t-1}^{-1} A_t^\top + P_t\big)^{-1}\)

\\
\midrule
\textbf{Posterior / Update} \newline
$\displaystyle z_t \mid v_{1:t} \sim \mathcal N(\mu_t,\,\Sigma_t)\quad$
&
\(\mu_t = \mu_t^{\mathrm{prior}} + K_t\big(v_t - C_t \mu_t^{\mathrm{prior}}\big)\) \newline
\(\Sigma_t = (I - K_t C_t)\Sigma_t^{\mathrm{prior}} \)  \newline
\(K_t = \Sigma_t^{\mathrm{prior}} C_t^\top (C_t \Sigma_t^{\mathrm{prior}} C_t^\top + (\Lambda^{\mathrm{v}}_t)^{-1})^{-1}\)
&
\(\displaystyle \Lambda_t = \Lambda_t^{\mathrm{prior}} + C_t^\top \Lambda^{\mathrm{v}}_t C_t\) \newline
\(\displaystyle \eta_t = \eta_t^{\mathrm{prior}} + C_t^\top \Lambda^{\mathrm{v}}_t v_t\)
\\
\bottomrule
\end{tabular}
\end{table}

\paragraph{Intuition on noise.}
In our linear--Gaussian state-space model, the ``noise'' terms \(\varepsilon_t\) and \(\nu_t\) are not claims that language or the world is intrinsically random; they compactly represent \emph{uncertainty} and unmodeled effects (domain shift, annotation/embedding error, missing context).
Their covariances \(P_t\) and \(\Sigma^{\mathrm{obs}}_t\) calibrate trust in dynamics vs.\ data---larger \(P_t\) yields more adaptivity (faster forgetting), while larger \(\Sigma^{\mathrm{obs}}_t\) leans on the prior---producing principled ``gates'' via precision ratios.
In contrast, `noise' in diffusion models is an \emph{algorithmic corruption} deliberately added in a forward process to define a denoising/score-matching objective; its schedule (\(\beta_t\)) is chosen for optimisation/SNR shaping rather than for model measurement uncertainty.

\newpage

\clearpage

\section{Theorems and Proofs}
\label{app:theorems-proofs}
\ThMobius*
\begin{mdframed}
\begin{proof}[Proof of \Cref{Th:Mobius}]
Consider the 1D (or diagonal) linear--Gaussian state space model
\[
z_t = \bar{a}_t\,z_{t-1} + \varepsilon_t,\qquad \varepsilon_t \sim \mathcal N(0,\bar{p}_t),\qquad
v_t = k_t\,z_t + \nu_t,\qquad \nu_t \sim \mathcal N(0,(\Lambda_t^{\mathrm{v}})^{-1}),
\]
and define the posterior precision
$\lambda_{t} \coloneqq \mathrm{Var}(z_t\mid v_{1:t})^{-1}$.

\textbf{Step 1 (Information-form predict/update).}
\[
\text{(predict)}\quad
\lambda_{t}^{\mathrm{prior}} \;=\; \bigl(\bar{a}_t^2 \lambda_{t-1}^{-1} + \bar{p}_t\bigr)^{-1}
\;=\;
\frac{\lambda_{t-1}}{\bar{a}_t^2 + \bar{p}_t\,\lambda_{t-1}},
\qquad
\text{(update)}\quad
\lambda_t \;=\; \lambda_t^{\mathrm{prior}} + k_t^2 \Lambda_t^{\mathrm{v}}.
\]

\textbf{Step 2 (Single-step recursion and rearrangement).}
Eliminate the intermediate $\lambda_t^{\mathrm{prior}}$ to obtain
\[
\lambda_t
\;=\;
\frac{\lambda_{t-1}}{\bar{a}_t^2 + \bar{p}_t\,\lambda_{t-1}}
\;+\;
k_t^2 \Lambda_t^{\mathrm{v}}
\;=\;
\frac{(1 + \bar{p}_t\,k_t^2 \Lambda_t^{\mathrm{v}})\,\lambda_{t-1} + \bar{a}_t^2 k_t^2 \Lambda_t^{\mathrm{v}}}{\bar{p}_t\,\lambda_{t-1} + \bar{a}_t^2}.
\]
Hence $\lambda_t$ is a linear--fractional (M\"{o}bius) transform of $\lambda_{t-1}$,
\[
\lambda_t \;=\; M_t(\lambda_{t-1})
\;=\;
\frac{\alpha_t \lambda_{t-1} + \beta_t}{\gamma_t \lambda_{t-1} + \delta_t},
\qquad
M_t \;=\;
\begin{pmatrix}
\alpha_t & \beta_t \\ \gamma_t & \delta_t
\end{pmatrix}
\;=\;
\begin{pmatrix}
1 + \bar{p}_t\,\phi_t & \bar{a}_t^2 \phi_t \\
\bar{p}_t & \bar{a}_t^2
\end{pmatrix},
\]
where $\phi_t \coloneqq k_t^2 \Lambda_t^{\mathrm{v}}$. Thus, the precision recursion is a M\"{o}bius transformation with the stated matrix form.
The diagonal multivariate case follows by applying the same scalar derivation to each diagonal element.

\end{proof}
\end{mdframed}

\CorParallelScan*

\begin{mdframed}
\begin{proof}[Proof of \Cref{cor:parallel-scan}]
By \Cref{Th:Mobius}, each precision update is a M\"{o}bius transformation:
$\lambda_t = M_t(\lambda_{t-1})$. Since M\"{o}bius transformations compose via
matrix multiplication, the precision at time $t$ can be expressed as
\[
\lambda_t = \bigl(M_t \circ M_{t-1} \circ \cdots \circ M_1\bigr)(\lambda_0)
= M_{1:t}(\lambda_0),
\]
where the composition $M_{1:t} := \prod_{s=1}^t M_s$ is computed via
standard $2\times 2$ matrix multiplication. Since matrix multiplication is
associative, the product $M_{1:T}$ can be computed via a parallel prefix
scan~\citep{blelloch1990prefix} with $\mathcal{O}(T)$ work and $\mathcal{O}(\log T)$ depth on
$T$ processors. Once the prefix products $\{M_{1:t}\}_{t=1}^T$ are available,
each precision $\lambda_t = M_{1:t}(\lambda_0)$ is obtained in $\mathcal{O}(1)$
time, yielding the claimed complexity.
\end{proof}
\end{mdframed}

\begin{remark}[Practical Implementation]
In practice, the parallel prefix scan can be implemented efficiently on modern hardware
using frameworks such as JAX's \texttt{lax.associative\_scan} or PyTorch's parallel
primitives. The $\mathcal{O}(\log T)$ depth translates to logarithmic wall-clock time on
sufficiently parallel hardware, matching the computational efficiency of deterministic
SSMs like Mamba while maintaining probabilistic semantics.
\end{remark}

\newpage

\ThMeanAffine*
\begin{mdframed}
\begin{proof}[Proof of \Cref{thm:eta-affine}]
Write the information parameters as $\eta_t := \Lambda_t \mu_t$ and let
$\Lambda_t^{\mathrm{v}}$ denote the value precision. The information-form measurement update is
\[
\Lambda_t \;=\; \Lambda_t^{\mathrm{prior}} + k_t^2 \Lambda_t^{\mathrm{v}},
\qquad
\eta_t \;=\; \eta_t^{\mathrm{prior}} + k_t\,\Lambda_t^{\mathrm{v}}\,v_t,
\]
so the observation contribution is affine with coefficient $k_t\,\Lambda_t^{\mathrm{v}}$.

For the time-prediction of the information mean, use $\mu_t^{\mathrm{prior}} = \bar{a}_t\,\mu_{t-1}$ and
$\eta_t^{\mathrm{prior}} = \Lambda_t^{\mathrm{prior}} \mu_t^{\mathrm{prior}}$. Eliminating the means gives
\[
\eta_t^{\mathrm{prior}} \;=\; \Lambda_t^{\mathrm{prior}}\, \bar{a}_t\, \mu_{t-1}
\;=\; \Lambda_t^{\mathrm{prior}}\, \bar{a}_t\, \Lambda_{t-1}^{-1}\,\eta_{t-1}.
\]
Hence, given the (known) precision path $\{\Lambda_t, \Lambda_t^{\mathrm{prior}}\}$,
\[
\eta_t \;=\; \underbrace{\bigl(\Lambda_t^{\mathrm{prior}}\, \bar{a}_t\, \Lambda_{t-1}^{-1}\bigr)}_{\displaystyle f_t}\,\eta_{t-1}
\;+\; k_t\,\Lambda_t^{\mathrm{v}}\,v_t,
\qquad
\mu_t \;=\; \Lambda_t^{-1}\eta_t.
\]
Finally, substituting the information-form predict step
$\Lambda_t^{\mathrm{prior}} = \Lambda_{t-1}/(\bar{a}_t^2 + \bar{p}_t\,\Lambda_{t-1})$ (Step~1 of the proof of \Cref{Th:Mobius}) eliminates
$\Lambda_t^{\mathrm{prior}}$ from the forget gate,
\[
f_t \;=\; \Lambda_t^{\mathrm{prior}}\, \bar{a}_t\, \Lambda_{t-1}^{-1}
\;=\; \frac{\bar{a}_t}{\bar{a}_t^2 + \bar{p}_t\,\Lambda_{t-1}}
\;=\; \bigl(\bar{a}_t^2 + \bar{p}_t\,\Lambda_{t-1}\bigr)^{-1} \bar{a}_t,
\]
which is exactly the forget gate in the statement of \Cref{thm:eta-affine}. This proves the claimed affine form.
\end{proof}
\end{mdframed}

\CorMomentForm*

\begin{mdframed}
\begin{proof}[Proof of \Cref{cor:moment-form}]
By \Cref{thm:eta-affine}, $\eta_t = f_t\,\eta_{t-1} + k_t\,\Lambda_t^{\mathrm{v}}\,v_t$ with forget gate
$f_t = \lambda_t^{\mathrm{prior}}\,a_t\,\lambda_{t-1}^{-1}$. By Step~1 of the proof of \Cref{Th:Mobius}, the
predicted precision is
\[
\lambda_t^{\mathrm{prior}} \;=\; \frac{\lambda_{t-1}}{a_t^2 + p_t\,\lambda_{t-1}} \;=\; \rho_t\,\lambda_{t-1},
\qquad
\lambda_t \;=\; \lambda_t^{\mathrm{prior}} + k_t^2\,\Lambda_t^{\mathrm{v}}.
\]
Substituting $\eta_{t-1} = \lambda_{t-1}\,\mu_{t-1}$ and $\mu_t = \eta_t/\lambda_t$, and using
$f_t\,\lambda_{t-1} = \lambda_t^{\mathrm{prior}}\,a_t$,
\[
\mu_t \;=\; \frac{\lambda_t^{\mathrm{prior}}\,a_t\,\mu_{t-1} + k_t\,\Lambda_t^{\mathrm{v}}\,v_t}{\lambda_t}.
\]
Since $\lambda_t^{\mathrm{prior}} = \lambda_t - k_t^2\,\Lambda_t^{\mathrm{v}}$, we have
$\lambda_t^{\mathrm{prior}}/\lambda_t = 1 - k_t^2\,\Lambda_t^{\mathrm{v}}/\lambda_t$, hence
\[
\mu_t \;=\; a_t\Bigl(1 - \frac{k_t^2\,\Lambda_t^{\mathrm{v}}}{\lambda_t}\Bigr)\mu_{t-1} + \frac{k_t\,\Lambda_t^{\mathrm{v}}}{\lambda_t}\,v_t,
\]
the claimed gated recurrence. The diagonal multivariate case applies elementwise.
\end{proof}
\end{mdframed}

\begin{theorem}[Convolutional Form for Deterministic LTI Systems]
\label{thm:conv-form}
Under the conditions: (i) time-invariant dynamics, $\mathbf{a}_t \equiv \mathbf{a}$, $\mathbf{k}_t \equiv \mathbf{k}$; (ii) deterministic dynamics, process noise $\mathbf{p}_t = \mathbf{0}$; and (iii) zero initial conditions $\boldsymbol{\lambda}_0 = \boldsymbol{\eta}_0 = \mathbf{0}$, both the precision and information mean updates reduce to block-Toeplitz
convolutions computable in $\mathcal{O}(T \log T)$ time via FFT/NTT:

\paragraph{Precision updates:} $ \boldsymbol{\Lambda}_t = \sum_{s=0}^t \mathbf{a}^{-2(t-s)} \odot \mathbf{k}^2 \odot \boldsymbol{\Lambda}^{\mathrm{v}}_s $

\paragraph{Information mean updates:} $\boldsymbol{\eta}_t = \mathbf{k} \odot\,\sum_{s=0}^{t} \mathbf{a}^{-(t-s)} \odot\,\boldsymbol{\Lambda}_s^{\mathrm{v}} \odot\,\mathbf{v}_s,
\qquad
\boldsymbol{\mu}_t = \boldsymbol{\Lambda}_t^{-1} \odot \boldsymbol{\eta}_t $
\end{theorem}

\begin{mdframed}
\begin{proof}[Proof of \Cref{thm:conv-form}]
When $p_t = 0$ (deterministic dynamics), the M\"{o}bius transformation from
\Cref{Th:Mobius} simplifies. Specifically, with zero process noise, the
precision recursion becomes $\lambda_t = a^{-2}\lambda_{t-1} + k^2\Lambda_t^{\mathrm{v}}$.
For time-invariant $a$ and $k$, this unrolls to:
\[
\lambda_t = a^{-2t}\lambda_0 + k^2\sum_{s=0}^t a^{-2(t-s)} \Lambda^{\mathrm{v}}_s
\]
which is a discrete convolution with kernel $\kappa[n] = k^2 a^{-2n}$.

Similarly, from \Cref{thm:eta-affine}, when $p_t=0$ we have $f_t = a^{-1}$
constant, so the information mean recursion $\eta_t = a^{-1}\eta_{t-1} + k\Lambda_t^{\mathrm{v}}v_t$
unrolls to:
\[
\eta_t = k\,\sum_{s=0}^{t} a^{-(t-s)}\,\Lambda_s^{\mathrm{v}}\,v_s
\]
which is also a convolution with kernel $h[n] = k\,a^{-n}$.

Both convolutions have Toeplitz structure and can be computed via FFT in
$\mathcal{O}(\log T)$ parallel time using the convolution theorem.
\end{proof}
\end{mdframed}

\begin{remark}[Practical Note]
While theoretically interesting, the convolutional form is primarily relevant for
perfectly deterministic systems ($\mathbf{p}=\mathbf{0}$). In practice, \our uses the more general
parallel scan formulation (\Cref{cor:parallel-scan}) which handles stochastic
dynamics ($\mathbf{p}_t > \mathbf{0}$) efficiently.
\end{remark}
\newpage
\CorMeanParallelScan*

\begin{mdframed}
\begin{proof}[Proof of \Cref{cor:mean-parallel-scan}]
From \Cref{thm:eta-affine}, the information mean update has the affine form:
\[
\eta_t = f_t \eta_{t-1} + k_t \Lambda_t^{\mathrm{v}} v_t
\]
where $f_t = \frac{a_t}{a_t^2 + p_t \lambda_{t-1}}$ (which simplifies to $a^{-1}$ when $p_t = 0$).

Define the binary associative operator $\oplus$ on pairs $(f, b)$ by:
\[
(f_2, b_2) \oplus (f_1, b_1) = (f_2 f_1, f_2 b_1 + b_2)
\]
This is the standard associative operator for affine transformations.

The composed transformation at time $t$ is:
\[
(f_{1:t}, b_{1:t}) = (f_t, k_t \Lambda_t^{\mathrm{v}} v_t) \oplus (f_{t-1}, k_{t-1} \Lambda_{t-1}^{\mathrm{v}} v_{t-1}) \oplus \cdots \oplus (f_1, k_1 \Lambda_1^{\mathrm{v}} v_1)
\]
which gives $\eta_t = f_{1:t} \eta_0 + b_{1:t}$.

Since $\oplus$ is associative, the sequence of compositions can be computed via
a parallel prefix scan~\citep{blelloch1990prefix} with $\mathcal{O}(T)$ work and $\mathcal{O}(\log T)$ depth.
Once all information means $\{\eta_t\}_{t=1}^T$ are computed, the posterior means are
obtained by normalising with the precision path from \Cref{cor:parallel-scan}:
$\mu_t = \Lambda_t^{-1}\eta_t$.
\end{proof}
\end{mdframed}

\vfill
\begin{center}
    --appendices continue on next page--
\end{center}

\clearpage
\section{Additional Empirical Results}
\label{app:empirical}

\subsection{Ablation on the Importance of Process Noise Parameter}
\label{sec:gauss-variants}

We conduct an ablation that isolates the process-noise parameter by fixing $\mathbf{p}_t = 0$ (deterministic dynamics) while leaving the rest of \our unchanged --- in particular its \emph{adaptive}, input-dependent observation operator $\mathbf{k}_t$. Setting $\mathbf{p}_t = 0$ freezes the state-dependent factor $\boldsymbol{\rho}_t$ in the M\"{o}bius precision recurrence to a constant and collapses it to a fixed-gate linear update; the recurrence is still a parallel prefix scan with $\mathcal{O}(T)$ work and $\mathcal{O}(\log T)$ depth, but its gating is no longer history-dependent. (The stronger convolutional form of \Cref{thm:conv-form}, computable via FFT, requires \emph{both} $\mathbf{p}_t = 0$ \emph{and} a time-invariant observation operator $\mathbf{k}_t \equiv \mathbf{k}$; with the adaptive $\mathbf{k}_t$ retained here the ablated system is linear but not time-invariant.) This ablation investigates whether the nonlinear gating supplied by process noise translates into practical benefits.

\cref{tab:process-noise-ablation} compares full \our (learnable, time-invariant process noise $\mathbf{p}$) against a deterministic variant with process noise fixed to zero ($\mathbf{p}_t{=}0$, but retaining observation variance $(\Lambda^{\mathrm{v}})^{-1}$), which collapses the M\"{o}bius precision recurrence to a fixed-gate linear update. In the single-block MAD-Lab baseline setting the memorisation and recall tasks are essentially saturated for every architecture and show little cross-method variance (\cref{tab:accuracy-heatmap}); the skills that actually discriminate between models are Compression, Fuzzy Recall, and Selective Copy. Turning process noise off degrades exactly these three hardest tasks, while leaving memorisation, context recall, and noisy recall at ${\sim}100\%$: Selective Copy and Compression each fall by more than $10$ points ($-14.9$ and $-12.1$, respectively) and Fuzzy Recall by $2.5$, for a $4.7$-point average drop.

\begin{table}[h]
\caption{\textbf{Process-noise ablation on the MAD-Lab suite (accuracy \%, $\uparrow$).}
Full \our (learnable, time-invariant process noise $\mathbf{p}$) versus a deterministic
variant with $\mathbf{p}_t{=}0$, which collapses the M\"{o}bius precision recurrence to a
fixed-gate linear update. With-noise baselines are the single-block
MAD results of \cref{tab:accuracy-heatmap}. Removing process noise leaves memorisation and
exact recall unchanged but degrades the selection-heavy tasks (Selective Copy, Compression).}
\label{tab:process-noise-ablation}
\centering
\small
\setlength{\tabcolsep}{4pt}
\renewcommand{\arraystretch}{1.15}
\begin{tabular}{lccccccc}
\toprule
\textbf{\our{} variant} & \textbf{Compr.} & \textbf{Mem.} & \textbf{CR} & \textbf{NR} & \textbf{Fuzzy} & \textbf{Sel.\ Copy} & \textbf{Avg.} \\
\midrule
Learnable $\mathbf{p}$ (full)      & 85.03 & 98.87 & 99.95 & 99.93 & 45.70 & 90.67 & 86.69 \\
$\mathbf{p}_t{=}0$ (deterministic) & 72.91 & 100.00 & 100.00 & 100.00 & 43.21 & 75.73 & 81.98 \\
\midrule
$\Delta$ (zero $-$ full)           & $-12.12$ & $+1.13$ & $+0.05$ & $+0.07$ & $-2.49$ & $-14.94$ & $-4.71$ \\
\bottomrule
\end{tabular}
\end{table}

These drops move \our from best-in-class to near the bottom of the field on the very skill it is designed for: on Selective Copy the deterministic variant ($75.7$) trails every baseline except mLSTM (\cref{tab:accuracy-heatmap}), erasing \our's selection-under-uncertainty advantage. The effect is more pronounced at the larger pretraining scale of \cref{sec:pretraining}, with longer contexts and multiple stacked \our blocks: there, fixing $\mathbf{p}_t{=}0$ frequently destabilises training and diverges to NaNs, as the precision recursion degenerates into a linear scan in which the posterior precision $\boldsymbol{\lambda}_t$ grows unbounded and the nonlinear, history-dependent gating (\cref{eq:lambda_gate_view}) is lost; a non-zero process noise caps the accumulated precision and restores a fading, well-conditioned memory. A theoretical investigation of the stability properties of the system, particularly regarding Riccati controllability and observability, and the role of $\mathbf{p}_t$ in system stabilisation, represents an interesting direction for future work.

\subsection{Long-Context Associative Recall (Full Sweep)}
\label{sec:longcontext-mqar-results}
Multi-Query Associative Recall (MQAR)~\citep{arora2023zoology} is particularly demanding for fixed-size recurrent architectures, as it directly tests their capacity to store and retrieve multiple key--value associations --- a fundamental bottleneck often limited by state dimensionality. Following the Zoology benchmark~\citep{arora2023zoology}, we evaluate on a considerably harder configuration than typically studied: sequence length $T=2048$ with vocabulary size $V=256$ (matching or exceeding model dimensions; see \cref{tab:longcontext-mqar-params}).

\cref{fig:mqar-results} reveals distinct scaling behaviours across architectures. \textsc{KLA} consistently outperforms Mamba across all dimensions and substantially outperforms GLA, which fails to learn the task under this extreme setting. At lower dimensions ($d=64,128$), GDN performs best; its delta-rule mechanism is specifically designed for strong associative recall at limited capacity~\citep{yang2024gated}, and it operates with a slightly higher effective state size (see \cref{tab:mqar-model-configs}). At $d=256$, \textsc{KLA} outperforms all baselines and reaches near-perfect accuracy ($>95\%$).

We hypothesise that uncertainty-weighted updates help mitigate state ``saturation'' under heavy key--value load. Unlike purely additive accumulators, \textsc{KLA}'s fractional-linear (M\"{o}bius) precision updates (\cref{eq:lambda_gate_view}) adaptively down-weight observations based on posterior uncertainty: when the latent state becomes saturated with many associations, high-confidence historical information naturally suppresses unreliable new observations. This provides an implicit, learned compression mechanism, analogous to selective pattern storage in high-capacity associative memories~\citep{storkey1997hopfield}, without requiring explicit dimension scaling.\looseness-1

\subsection{Runtime Scaling}
\label{sec:fwd-scaling}

\paragraph{Setup.}
We benchmark forward-only runtimes for three \textsc{KLA} implementations:
a recurrent (time-stepped) Kalman update, the built-in Torch associative scan, and a custom Triton associative scan kernel.
All measurements use Torch~2.9.1 on an NVIDIA A100 (80GB), hidden size $960$, and float32 precision.
The forward-only pass uses a batch size of $1$, simulating test-time prompt processing for a 182 million parameter model---critical for deploying sequence models in production settings.

\paragraph{Results.}
\cref{fig:scaling-fwd} shows that the scan-based implementations scale efficiently with sequence length compared to the recurrent counterparts. These results confirm that information-form filtering can be implemented with the same scan-parallel profile as modern SSM/GLA-style mixers, making \textsc{KLA} practical for long-context prompt processing.

\begin{figure}[h]
    \centering
    \includegraphics[width=0.5\textwidth]{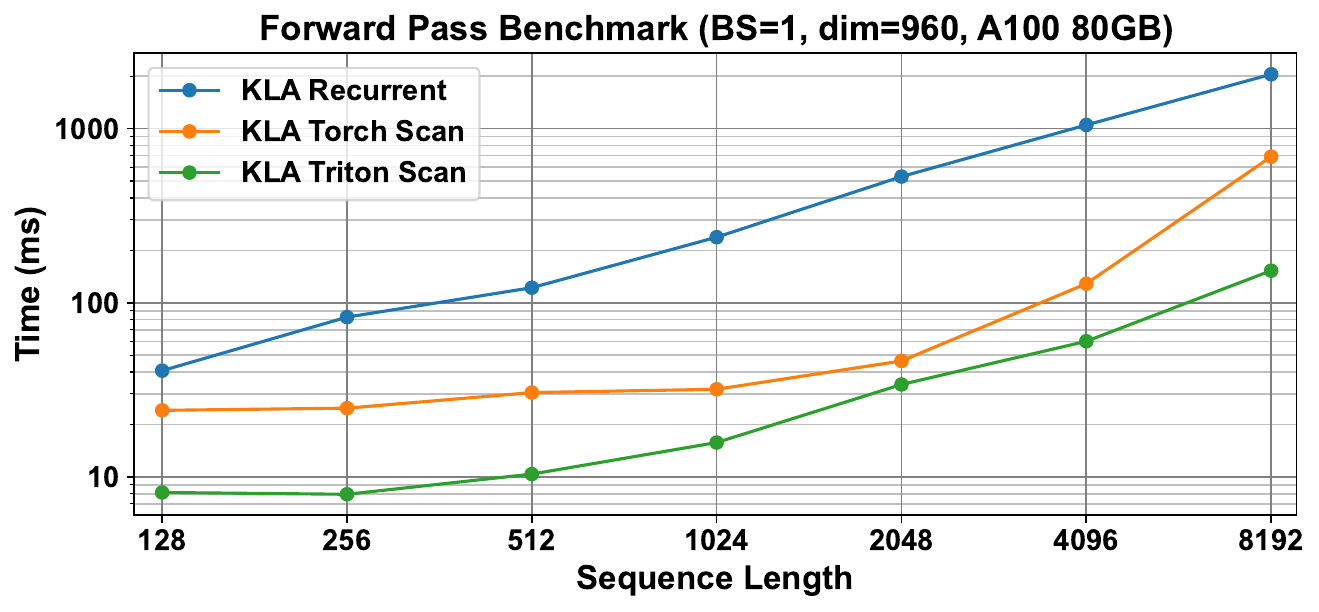}
    \caption{\textbf{Forward pass (prompt processing) runtime scaling.} Wall-clock runtime of \textsc{KLA} implementations across sequence lengths during forward-only pass. Torch Scan uses \texttt{torch.\_higher\_order\_ops.associative\_scan}; Triton Scan uses custom forward kernels.}
    \label{fig:scaling-fwd}
\end{figure}

\subsection{Equivalent Attention Matrix}
\label{sec:attn-matrix}

As shown in \Cref{thm:eta-affine}, the information-mean recurrence
$\eta_t = f_t\,\eta_{t-1} + k_t\,\Lambda_t^{\mathrm{v}}\,v_t$ (\Cref{eq:mu-update})
can be unrolled into a lower-triangular attention matrix~$\mathbf{W}$
whose entries are products of
{\color{forgetred}history-dependent forget gates~$f_s$},
keys~$k_j$, and
{\color{UncBlue}observation precisions~$\Lambda_j^{\mathrm{v}}$}.
Folding in the output readout query~$\mathbf{q}$ and posterior precision
scaling~${\color{UncBlue}\boldsymbol{\lambda}^{-1}}$ gives the full per-channel
sequence transformation
$\mathbf{M}_{\mathrm{seq}}
  = \mathrm{diag}(\mathbf{q}\odot{\color{UncBlue}\boldsymbol{\lambda}^{-1}})\,\mathbf{W}$,
so that
$\mathbf{y} = \mathbf{M}_{\mathrm{seq}}\,\mathbf{v}$ (plus init-state terms).
\cref{fig:attn-matrix-expansion} visualises this two-step structure.

\begin{figure}[h!]
  \centering
  \resizebox{\textwidth}{!}{\tikzBlockArchWithAttention}
  \caption{\textbf{Block architecture and equivalent attention form.}
  \emph{Left:} The block architecture follows the fused-MLP design of Mamba,
  with the Kalman Filter as a drop-in replacement for any SSM/Attention primitive.
  \emph{Right:} Unrolling the information-mean recurrence
  (\Cref{eq:mu-update}: $\eta_t = f_t\,\eta_{t-1} + k_t\,\Lambda_t^{\mathrm{v}}\,v_t$)
  yields a lower-triangular matrix~$\mathbf{W}$
  whose entries are products of
  {\color{forgetred}history-dependent forget gates $f_s$},
  keys $k_j$, and
  {\color{UncBlue}observation precisions $\Lambda_j^{\mathrm{v}}$}.
  Applying the output readout query~$q_i$ and
  {\color{UncBlue}posterior precision $\lambda_i^{-1}$} gives the full
  sequence transformation $\mathbf{M}_{\mathrm{seq}}
    =\mathrm{diag}(\mathbf{q}\odot{\color{UncBlue}\boldsymbol{\lambda}^{-1}})\,\mathbf{W}$%
  ---the precision terms being unique to KLA and absent in standard SSMs/GLA.
  }
  \label{fig:attn-matrix-expansion}
\end{figure}

\subsection{Kalman Attention Map Visualisation}
\label{sec:attention-viz}

Unrolling the information-mean recurrence~\eqref{eq:mu-update} yields the
lower-triangular matrix~$\mathbf{W}$ described in \cref{sec:attn-matrix};
folding in the readout query~$\mathbf{q}_t$ and precision
scaling~$\boldsymbol{\lambda}^{-1}$ gives the per-channel sequence transformation
$\mathbf{M}_{\mathrm{seq}}
  = \mathrm{diag}(\mathbf{q}\odot\boldsymbol{\lambda}^{-1})\,\mathbf{W}$
(cf.\ \cref{fig:attn-matrix-expansion}).
Because $\mathbf{M}_{\mathrm{seq}}$ is lower-triangular, it has the same causal
structure as a standard attention matrix; we therefore call it the
\emph{Kalman Attention Matrix}.

\Cref{fig:attn-sc,fig:attn-recall,fig:attn-mem} visualise the Kalman attention
patterns learned by \our on three MAD-Lab tasks, each showing four randomly chosen
channels.
The attention weight at position $(t,s)$ reflects how much the posterior mean at
step $t$ is influenced by the observation at step $s$, weighted by the learned
precision ratio.

\begin{figure}[h!]
  \centering
  \includegraphics[width=0.68\linewidth]{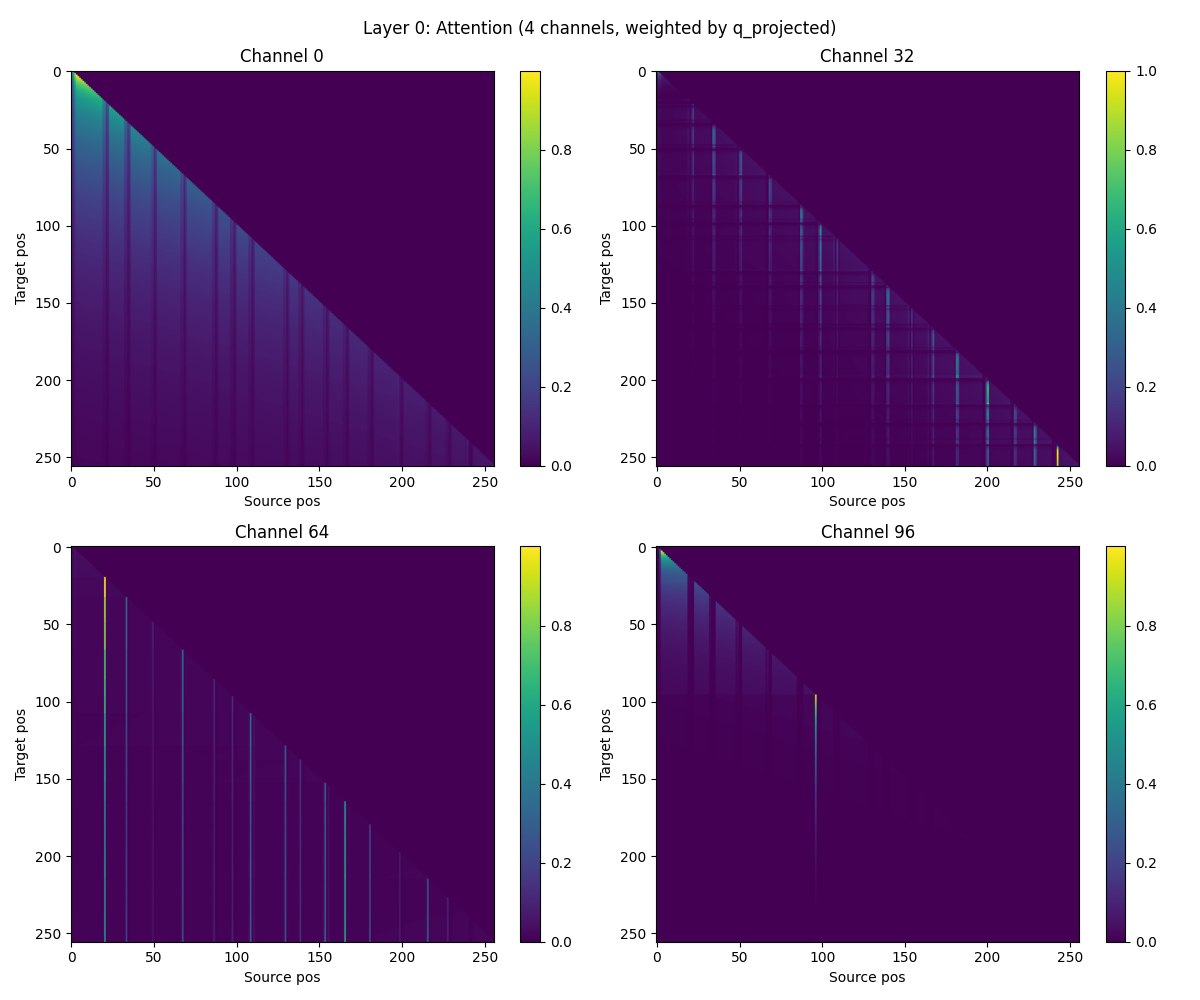}
  \caption{Attention maps for \textbf{Selective Copying} (sequence length 256). The model learns sparse, intermittent vertical bands, attending strongly to a small set of task-relevant positions (mostly copy positions). They either activate or suppress these relevant positions. Differences across channels indicate channel-level specialisation in which positions are retrieved.}
  \label{fig:attn-sc}
\end{figure}
\nopagebreak
\begin{figure}[h!]
  \centering
  \includegraphics[width=0.68\linewidth]{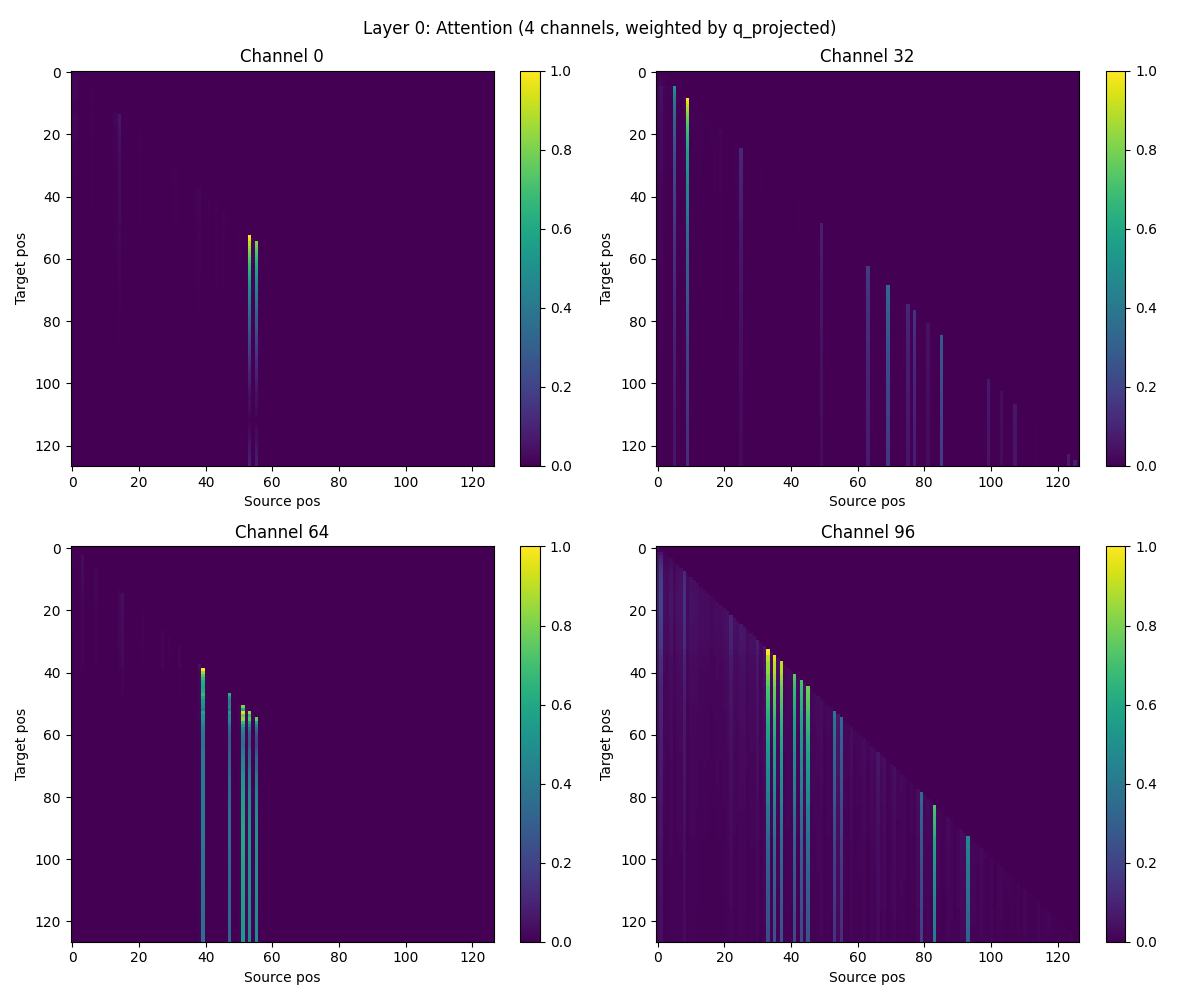}
  \caption{Attention maps for \textbf{In-Context Recall} (sequence length 128). The patterns are sparse and ``pointer-like'': within each channel, attention concentrates on a small set of source positions (vertical bands) across many target steps, with a mostly low-activation background. Channels show specialisation trends.}
  \label{fig:attn-recall}
\end{figure}
\nopagebreak
\begin{figure}[h!]
  \centering
  \includegraphics[width=0.68\linewidth]{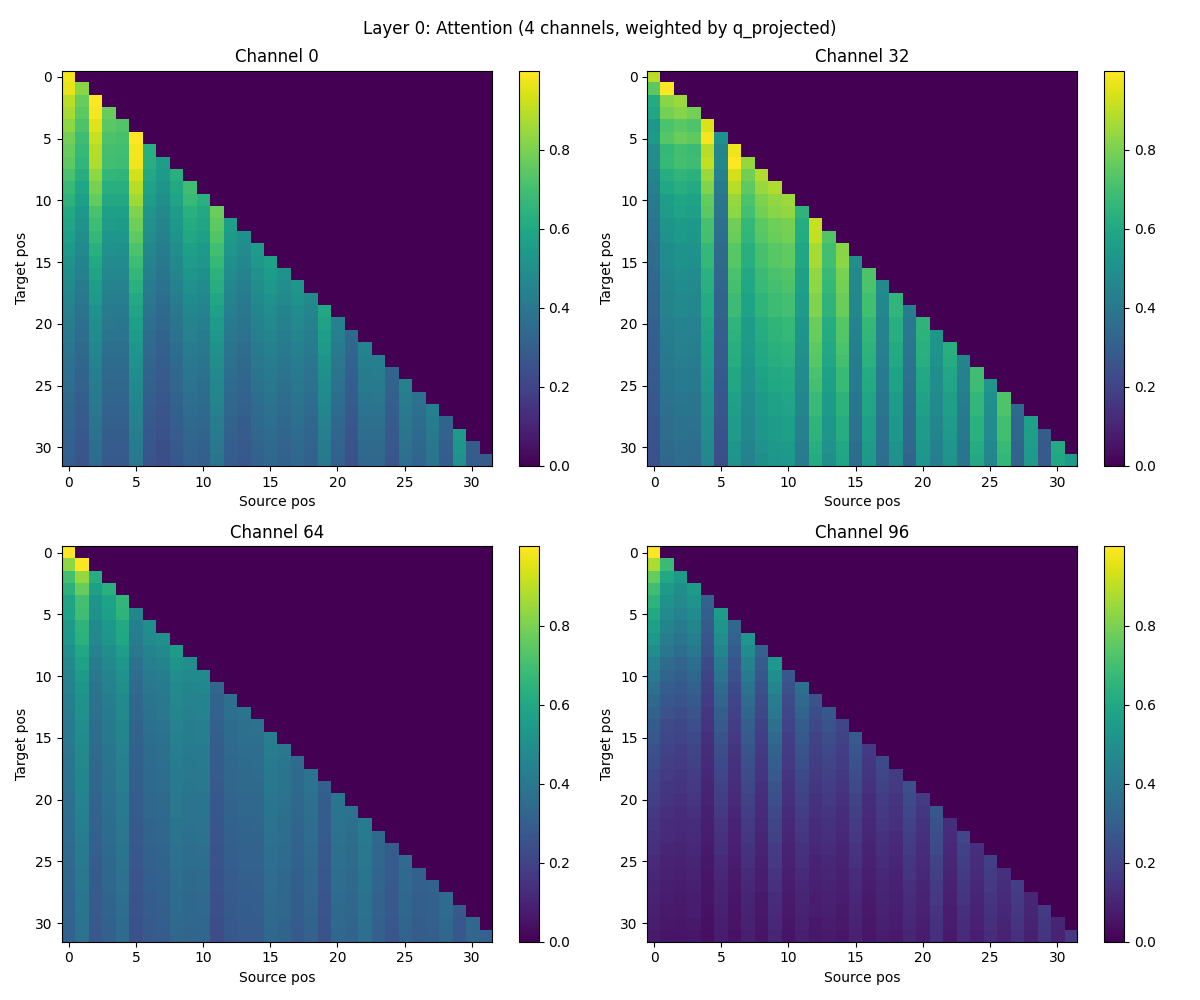}
  \caption{Attention maps for \textbf{Memorization} (sequence length 32). The smoother, more gradual decay reflects the task's requirement to maintain information uniformly across the sequence. All channels show similar patterns, indicating that memorization benefits from redundant, distributed storage rather than selective attention.}
  \label{fig:attn-mem}
\end{figure}

These visualisations demonstrate that \our learns task-appropriate attention patterns: selective retrieval for copying tasks, channel specialisation for recall, and uniform retention for memorization---all while maintaining the causal structure inherent to the Kalman filter formulation.

\clearpage

\section{Datasets and Benchmarks Used}
\label{app:datasets}
\subsection{MAD LM Suite}
To evaluate \our as a language-modelling primitive, we use a standardised suite of six discrete token manipulation tasks. We adopt the MAD framework~\citep{poli2024mad}, which consolidates synthetic tasks from prior work into a common specification of data generation, splits, and evaluation. Each task probes a distinct capability (associative retrieval, span compositionality, noise robustness, ordered copying, single-token information aggregation, and parametric memory), allowing us to disentangle model strengths and weaknesses across language skills (see \cref{tab:skill-mapping-grouped}).

\begin{table*}[t]
\centering
\small
\caption{Mapping from language skills/capabilities to the synthetic probe tasks. Green \gmark~indicates the primary probes/tasks for these skills.}
\label{tab:skill-mapping-grouped}
\setlength{\tabcolsep}{5pt}
\renewcommand{\arraystretch}{1.12}

\begin{tabularx}{\textwidth}{@{}Xcccccc@{}}
\toprule
\textbf{Language Skill / Capability} &
\makecell[c]{\textbf{In-context}\\\textbf{(MQAR)}} &
\makecell[c]{\textbf{Fuzzy}\\\textbf{Recall}} &
\makecell[c]{\textbf{Noisy}\\\textbf{Recall}} &
\makecell[c]{\textbf{Selective}\\\textbf{Copy}} &
\makecell[c]{\textbf{Com-}\\\textbf{pression}} &
\makecell[c]{\textbf{Memo-}\\\textbf{rization}} \\
\midrule

\makecell[l]{\textbf{Prompt-level associative retrieval \& in-context learning}\\\emph{(bindings, multi-query, remapping, basic composition)}}
& \gmark & \gmark & \gmark & \gmark & \rmark & \rmark \\

\makecell[l]{\textbf{Selective attention \& noise robustness}\\\emph{(noise filtering, long-context stability, calibration)}}
& \rmark & \rmark & \gmark & \gmark & \rmark & \rmark \\

\makecell[l]{\textbf{Ordered copying \& pointer-like control flow}\\\emph{(order-sensitive working memory)}}
& \rmark & \rmark & \rmark & \gmark & \rmark & \rmark \\

\makecell[l]{\textbf{Information aggregation/bottlenecking}\\\emph{(token concatenation for downstream decoding)}}
& \rmark & \rmark & \rmark & \rmark & \gmark & \rmark \\

\makecell[l]{\textbf{Parametric memory of stable facts}\\\emph{(weight-based knowledge)}}
& \rmark & \rmark & \rmark & \rmark & \rmark & \gmark \\

\bottomrule
\end{tabularx}

\end{table*}

We use the baseline configuration from the MAD framework for all six tasks. \cref{tab:mad-configs} summarises the key parameters for each task.

\begin{table}[h!]
\centering
\caption{Baseline configurations for the six MAD tasks used in our experiments.}
\label{tab:mad-configs}
\scriptsize
\setlength{\tabcolsep}{4pt}
\renewcommand{\arraystretch}{1.12}
\begin{tabular}{@{}lcccccc@{}}
\toprule
\textbf{Parameter} & \textbf{In-context} & \textbf{Fuzzy} & \textbf{Noisy} & \textbf{Selective} & \textbf{Compression} & \textbf{Memorization} \\
& \textbf{(MQAR)} & \textbf{Recall} & \textbf{Recall} & \textbf{Copy} & & \\
\midrule
Vocab size & 16 & 16 & 16 & 16 & 16 & 256 \\
Training seqs & 12,800 & 12,800 & 12,800 & 12,800 & 12,800 & 256 \\
Sequence length & 128 & 128 & 128 & 256 & 32 & 32 \\
Noise tokens (\%) & -- & -- & 20\% & -- & -- & -- \\
Key/value split & 8/8 & 8/8 & 8/8 & -- & -- & 128/128 \\
\bottomrule
\end{tabular}
\end{table}
\noindent\textbf{Additional task-specific settings:} Selective Copy (SC) uses \texttt{num\_copy}=16; In-Context Recall (CR) and Fuzzy Recall (FR) use \texttt{multi\_query}=True; FR uses motif sizes $k = v = 3$; Noisy Recall (NR) uses \texttt{noise\_vocab}=16 with noise fraction 0.2.
\subsubsection{Task Descriptions}
\label{sec:mad-descriptions}

\paragraph{In-context recall (MQAR).}
Sequences consist of key-value pairs with separate vocabularies. The model must predict values for keys that appeared earlier in the sequence. Key-value mappings are randomly shuffled between sequences, forcing the model to learn in-context rather than memorising fixed associations.

\paragraph{Fuzzy in-context recall.}
An extension of in-context recall where keys and values are represented by variable-length spans (1--3 tokens). This tests the model's ability to handle compositional keys and maintain associations across multi-token representations.

\paragraph{Noisy in-context recall.}
Similar to in-context recall, but with 20\% of tokens from a separate noise vocabulary randomly inserted. This evaluates the model's robustness to irrelevant information and selective attention capabilities.

\paragraph{Selective copying.}
Sequences contain random tokens interspersed with special [blank] and [insert] tokens. The model must copy non-special tokens to [insert] positions in order, while learning to ignore tokens near [blank] markers. This tests selective memorisation and order-preserving working memory.

\paragraph{Compression.}
Sequences of random tokens ending with a compression token [c]. The model must compress all sequence information into the representation at the compression token position, such that a fixed two-layer MLP can reconstruct any input token given the compressed representation plus a positional encoding.

\paragraph{Memorization.}
A fixed key-value dictionary is used across all sequences. In each sequence, keys appear with their values masked by [insert] tokens. The model must learn the fixed mappings from the training data, as values never appear in the input during training.

\subsubsection{Standard Training Setup in MAD Framework}
\label{sec:mad-training}

All tasks use the same architectural scaffolding and training procedure. Training hyperparameters are detailed in \cref{app:hyperparams}.

For most tasks, the training objective is standard next-token prediction via log likelihood (cross-entropy loss). The \textbf{Compression} task uses a different setup: the model's output at the compression token position is passed through a separate two-layer MLP decoder (dimensions [240, 120]) along with a positional encoding to reconstruct each token in the sequence. During training, the model learns to compress the entire sequence into a single representation that enables accurate reconstruction.

All models use a single-layer architecture with a model dimension $d_{\text{model}} = 128$ and an effective state size of 2048.

\subsection{Long-Context MQAR}
\label{sec:longcontext-mqar}

\cref{tab:longcontext-mqar-params} lists the data configuration for the long-context MQAR benchmark task~\citep{arora2023zoology}.

\begin{table}[h]
\caption{\textbf{Long-Context MQAR data parameters.}}
\label{tab:longcontext-mqar-params}
\centering
\small
\setlength{\tabcolsep}{6pt}
\renewcommand{\arraystretch}{1.12}
\begin{tabular}{lcccc}
\toprule
\multicolumn{5}{c}{\textbf{Long-Context Multi-Query Associative Recall (MQAR)}} \\
\midrule
\textbf{Setting} & \textbf{Seq. length} & \textbf{Vocab size} & \textbf{Training seqs} & \textbf{Eval seqs} \\
\midrule
CR (hard) & 2,048 & 256 & 12,800 & 1,280 \\
\bottomrule
\end{tabular}
\vskip -0.1in
\end{table}

\newpage
\section{Implementation Details}
\label{app:hyperparams}

\subsection{Experimental Protocol}
\label{sec:exp-protocol}

\paragraph{Statistical significance.}
For the synthetic and state-tracking experiments we report results averaged over 5 random seeds to ensure statistical reliability, with all model comparisons using the same random seeds to ensure fairness; for these tasks we report the mean performance across seeds. For the A5 state tracking task, we consider a task solved if the model achieves $\geq 90\%$ accuracy in at least one of the five seeds. The language-modelling pretraining experiments of \cref{sec:pretraining} are the exception: following standard pretraining practice, each model is trained once owing to its substantially higher compute cost.

\paragraph{Open-source code.}
The complete implementation, including all custom kernels and experimental code, is publicly available at \href{https://github.com/vaisakh-shaj/kalman-linear-attention}{github.com/vaisakh-shaj/\allowbreak kalman-linear-attention} to facilitate reproducibility and future research.

\subsection{Training Hyperparameters}
\label{sec:training-hyperparams-sub}

We use default settings unless specifically stated otherwise. \cref{tab:training-hyperparams} lists the training configuration used across all experiments.

\begin{table}[h]
\caption{\textbf{Training Hyperparameters} (fixed across all models and tasks)}
\label{tab:training-hyperparams}
\begin{center}
\begin{small}
\begin{tabular}{ll}
\toprule
\textbf{Hyperparameter} & \textbf{Value} \\
\midrule
Optimizer & AdamW \\
Learning rate & $1 \times 10^{-3}$ \\
Learning rate schedule & None \\
Maximum epochs & 750 \\
Early stopping patience & 70 epochs \\
Gradient clipping & 5.0 \\
Weight decay & 0.0 \\
Precision & 32-bit float \\
\bottomrule
\end{tabular}
\end{small}
\end{center}
\vskip -0.1in
\end{table}

\noindent\textbf{\our-specific settings:} We use encoder MLP hidden dimension = 120, decoder MLP dimensions = [240, 120], process noise scale $\mathbf{p} = 0.01$ (initial value), discretisation timestep range $\Delta_t \in [0.001, 0.1]$, causal convolution kernel size = 4, and Monte Carlo samples = 10 (used for \our{} probabilistic decoding; see \cref{tab:accuracy-heatmap}).

\subsection{MAD-Lab Hyperparameters}
\label{sec:madlab-hyperparams}

To ensure fair comparison, all models are configured with equal effective state sizes by adjusting architecture-specific expansion parameters. For each model dimension, we match the total number of parameters used for state representation across architectures. Note that GatedDeltaNet does not allow exact state-size matching due to constraints on the minimum number of heads and the requirement that expanded key and value dimensions be multiples of 16 or 32. Hence, we assign GatedDeltaNet the next-largest feasible value compared to the other baselines. \cref{tab:model-dims} shows the dimensional configuration for each architecture.

\begin{table}[h]
\caption{\textbf{Model Dimensions} (configured for equal state size $S = 2048$)}
\label{tab:model-dims}
\begin{center}
\begin{small}
\begin{tabular}{lccccc}
\toprule
\textbf{Model} & \textbf{\# Blocks} & $d_{\text{model}}$ & $d_{\text{state}}$ & \textbf{Other} & $S$ \\
\midrule
\our (Ours) & 1 & 128 & 8 & expand=1 & 2048 \\
Mamba & 1 & 128 & 16 & expand=1 & 2048 \\
GatedDeltaNet & 1 & 128 & - & $e_k$=0.5, $e_v$=1.0, $H$=4 & 2048 \\
GLA & 1 & 128 & - & $e_k$=0.5, $e_v$=1.0, $H$=4 & 2048 \\
mLSTM & 1 & 128 & - & $H$=16, proj\_factor=2.0 & 2048 \\
\bottomrule
\end{tabular}
\end{small}
\end{center}
\vskip -0.1in
\end{table}

\noindent\textbf{Batch size:} We use a batch size of 172 for all MAD-Lab experiments.

\subsection{MQAR Hyperparameters}
\label{sec:mqar-hyperparams}

For the MQAR (hard) experiments, we evaluate across three model dimensions with equal state sizes per dimension. All models use 2 layers (repeated blocks). As in the MAD-Lab setup, GatedDeltaNet is given the nearest feasible configuration (see above). \cref{tab:mqar-model-configs} lists the configuration for each architecture at each dimension.

\begin{table}[h]
\caption{\textbf{MQAR Model Configurations} (state-matched per dimension)}
\label{tab:mqar-model-configs}
\begin{center}
\begin{small}
\begin{tabular}{lccccc}
\toprule
\textbf{Model} & \textbf{\# Blocks} & $d_{\text{model}}$ & $d_{\text{state}}$ & \textbf{Other} & \textbf{State Size} \\
\midrule
\multicolumn{6}{c}{$d=64$ (batch size = 64)} \\
\midrule
\our & 2 & 64 & 16 & expand=1 & 2,048 \\
Mamba & 2 & 64 & 16 & expand=2 & 2,048 \\
GatedDeltaNet & 2 & 64 & - & $e_k$=0.75, $e_v$=1.5, $H$=2 & 2,304 \\
GLA & 2 & 64 & - & $e_k$=1.0, $e_v$=1.0, $H$=4 & 2,048 \\
\midrule
\multicolumn{6}{c}{$d=128$ (batch size = 32)} \\
\midrule
\our & 2 & 128 & 16 & expand=1 & 4,096 \\
Mamba & 2 & 128 & 16 & expand=2 & 4,096 \\
GatedDeltaNet & 2 & 128 & - & $e_k$=0.75, $e_v$=1.5, $H$=4 & 4,608 \\
GLA & 2 & 128 & - & $e_k$=0.5, $e_v$=1.0, $H$=4 & 4,096 \\
\midrule
\multicolumn{6}{c}{$d=256$ (batch size = 16)} \\
\midrule
\our & 2 & 256 & 16 & expand=1 & 8,192 \\
Mamba & 2 & 256 & 16 & expand=2 & 8,192 \\
GatedDeltaNet & 2 & 256 & - & $e_k$=0.75, $e_v$=1.5, $H$=8 & 9,216 \\
GLA & 2 & 256 & - & $e_k$=0.25, $e_v$=1.0, $H$=4 & 8,192 \\
\bottomrule
\end{tabular}
\end{small}
\end{center}
\vskip -0.1in
\end{table}

\subsection{A5 State Tracking Hyperparameters}
\label{sec:a5-hyperparams}

For the A5 state tracking experiments (\cref{sec:a5-expressivity}), we use $d_{\text{model}} = 1024$ and $d_{\text{state}} = 16$ for \our. We train for up to 500 epochs with early stopping (patience 50), and a learning rate of $3 \times 10^{-4}$, and we evaluate over 5 seeds. Values for the baselines are reported as in~\citet{merrill2024illusion}.

\subsection{Language-Modelling Pretraining and Evaluation Details}
\label{app:lm-details}

This appendix gives the pretraining and evaluation details for the
language-modelling experiments of \cref{sec:pretraining} (results in
\cref{tab:lm-main}). Our pretraining pipeline - data loading, the AdamW
parameter grouping, the $\mu$P-style learning-rate scaling, and the trapezoidal
learning-rate schedule - follows the \texttt{nanochat} recipe~\citep{nanochat},
which we adapt to our two academic compute budgets and to \our's state-space
parameters.

\paragraph{Compute.} All language-modelling models were trained on NVIDIA GH200
Grace Hopper Superchips on the Isambard-AI National AI Research Resource
(AIRR)~\citep{mcintosh2024isambard}.

\paragraph{Pretraining corpus and objective.} All models are pretrained on
FineWeb-Edu~\citep{penedo2024fineweb} with the standard next-token-prediction
(cross-entropy) objective in single precision (\texttt{fp32}); only the
token-embedding table is stored in \texttt{bfloat16} for memory, while all
computation - including the parallel scan - runs in \texttt{fp32}. We use the \texttt{nanochat} BPE
tokenizer~\citep{nanochat} with a vocabulary of $32{,}768$ tokens, trained on the
pretraining corpus. We train at two academic scales: a $45$M-parameter model on
$1.8$B tokens and a $180$M-parameter model on $10.9$B tokens, in both cases
following an approximately Chinchilla-style token budget~\citep{hoffmann2022training}.

\paragraph{Models and baselines.} Both the $45$M and $180$M models use a
$12$-layer backbone, and all architectures are parameter-matched at each scale.
We compare four standalone mixers - GPT (softmax attention), Mamba, Gated
DeltaNet (GDN), and \our - and the hybrid GPT$+$\our, in which \emph{only the
final attention layer} of the GPT backbone is replaced by a \our block
(\cref{sec:pretraining}). Baseline state/expansion settings are matched following
the conventions in \cref{sec:madlab-hyperparams}; we report the \our-specific
per-scale configuration in \cref{tab:lm-pretrain-hparams}.

\paragraph{Optimisation.} Following \texttt{nanochat}~\citep{nanochat}, all
parameters are optimised with AdamW ($\beta=(0.8,0.95)$, $\epsilon=10^{-10}$) and
gradient clipping $3.0$ (\cref{sec:ablations}); the Muon optimiser is disabled.
Parameters are partitioned into groups - token embeddings, the unembedding (LM
head), $2$-D hidden weights, $1$-D parameters, and the state-space parameters
$(\lambda,\Delta,\mathbf{q},A)$ - and the state-space group is trained at
$0.1\times$ the hidden learning rate with zero weight decay, reflecting its higher
sensitivity inside the recurrent scan. Base learning rates are tuned at a
reference width $d_{\text{model}}=768$ and token batch size $B=2^{19}$, and
transferred to other widths and batch sizes by the $\mu$P-style
factor $\sqrt{768/d_{\text{model}}}$~\citep{yang2021tuning} and the Adam
batch-scaling factor $\sqrt{B/2^{19}}$. The learning-rate schedule is constant
followed by a linear warmdown over the final $40\%$ of training, with no warmup.
The remaining settings are summarised in \cref{tab:lm-pretrain-hparams}.

\begin{table}[h]
\caption{\textbf{\our{} language-modelling pretraining configurations.} We report
\our-specific settings only; both scales use a $12$-layer \our{} backbone, and
baseline configurations are state-matched following \cref{sec:madlab-hyperparams}.
Following \texttt{nanochat}~\citep{nanochat}, all parameters are optimised with
AdamW ($\beta=(0.8,0.95)$, $\epsilon=10^{-10}$) and Muon is disabled. The
per-group AdamW learning rate is the base rate scaled internally by
$\sqrt{768/d_{\text{model}}}\cdot\sqrt{B/2^{19}}$ (with $B$ the token batch size);
the state-space parameters $(\lambda,\Delta,\mathbf{q},A)$ use an additional
$0.1\times$ multiplier, and token embeddings use a separate base rate of $0.3$
(scaled by the same factor).}
\label{tab:lm-pretrain-hparams}
\centering
\small
\setlength{\tabcolsep}{8pt}
\renewcommand{\arraystretch}{1.12}
\begin{tabular}{lcc}
\toprule
\textbf{\our{} hyperparameter} & \textbf{45M} & \textbf{180M} \\
\midrule
Layers (depth)                       & 12 & 12 \\
$d_{\text{model}}$                   & 496 & 1360 \\
$d_{\text{state}}$                   & 32 & 32 \\
Causal conv.\ kernel size            & 4 & 4 \\
Process-noise init $\mathbf{p}$      & 0.01 & 0.01 \\
Discretisation range $\Delta_t$      & $[0.001,\,0.1]$ & $[0.001,\,0.1]$ \\
Context length                       & 2048 & 2048 \\
Training tokens                      & $1.8$B & $10.9$B \\
Tokenizer vocabulary                 & $32{,}768$ & $32{,}768$ \\
\midrule
Optimiser                            & AdamW & AdamW \\
Base learning rate (hidden)          & $4\times10^{-3}$ & $4\times10^{-3}$ \\
SSM learning-rate multiplier         & $0.1$ & $0.1$ \\
Effective peak LR (hidden)           & $2.5\times10^{-3}$ & $3.0\times10^{-3}$ \\
Effective peak LR (SSM)              & $2.5\times10^{-4}$ & $3.0\times10^{-4}$ \\
Embedding learning rate (base)       & $0.3$ & $0.3$ \\
LR schedule                          & constant then linear decay & constant then linear decay \\
Warmup steps                         & 0 & 0 \\
Warmdown fraction                    & 0.4 & 0.4 \\
Weight decay ($2$-D hidden/lm\_head) & 0.1 & 0.1 \\
Weight decay (embed/$1$-D/SSM)       & 0.0 & 0.0 \\
Global batch size (tokens)           & $131{,}072$ & $524{,}288$ \\
Gradient clipping                    & 3.0 & 3.0 \\
Precision                            & fp32 & fp32 \\
\bottomrule
\end{tabular}
\end{table}

\paragraph{Zero-shot evaluation.} We report zero-shot performance on eight commonsense benchmarks:
LAMBADA (OpenAI split)~\citep{paperno2016lambada},
HellaSwag~\citep{zellers2019hellaswag}, PIQA~\citep{bisk2020piqa},
ARC-Easy and ARC-Challenge~\citep{clark2018think},
WinoGrande~\citep{sakaguchi2021winogrande},
OpenBookQA~\citep{mihaylov2018can}, and BoolQ~\citep{clark2019boolq}
(\cref{tab:lm-main}). We report accuracy (acc) for all tasks except HellaSwag and
ARC-Challenge, for which we report length-normalised accuracy (acc$_n$).

\end{document}

%% file: math_commands.tex
\usepackage{amsmath,amsfonts,bm}

\def\eqref#1{equation~\ref{#1}}

\def\1{\bm{1}}

\DeclareMathAlphabet{\mathsfit}{\encodingdefault}{\sfdefault}{m}{sl}
\SetMathAlphabet{\mathsfit}{bold}{\encodingdefault}{\sfdefault}{bx}{n}



%% file: Figures/tikz_plots.tex
\newcommand{\tikzBGSS}{
\begin{tikzpicture}[very thick, node distance=2.2cm, >=latex]
  \node[latent, fill, minimum size=38pt] (z0) {$z_{t}$};
  \node[latent, right=of z0, minimum size=38pt, xshift=2.2cm] (z1) {$z_{t+1}$};

  \node[obs, below=1.4cm of z0, fill, minimum size=38pt] (o0) {$v_{t}$};
  \node[obs, below=1.4cm of z1, fill, minimum size=38pt] (o1) {$v_{t+1}$};

  \coordinate[left=1.2cm of z0] (cstart);
  \coordinate[right=1.2cm of z1] (cend);

  \edge {z0} {z1}

  \edge {z0} {o0}
  \edge {z1} {o1}

  \draw[dashed, ->, >=stealth, color=blue!60, thick, bend left=35] (o0) to (z0);
  \draw[dashed, ->, >=stealth, color=blue!60, thick, bend left=35] (o1) to (z1);

  \draw[dashed] (cstart) -- (z0);
  \draw[dashed] (z1) -- (cend);

  \draw[draw=none] (z0) -- node[midway, above=0pt, font=\normalsize] {$z_{t+1} \mid z_t \sim \mathcal{N}({\color{black}\bar{a}_t} z_t, {\color{blue}\bar{p}_t})$} (z1);

  \draw[draw=none] (z0) -- node[midway, right=2pt, font=\normalsize] {$v_t \mid z_t \sim \mathcal{N}({\color{black}k_t} z_t, {\color{blue}(\Lambda_t^{v})^{-1}})$} (o0);

  \draw[draw=none] (o0) -- node[midway, left=8pt, font=\normalsize, color=blue!70] {\textbf{inference}} (z0);

\end{tikzpicture}
}

\newcommand{\tikzMambaSSM}{
\begin{tikzpicture}[very thick, node distance=2.2cm, >=latex]
  \node[obs, fill, minimum size=38pt] (u0) {$u_{t-1}$};
  \node[obs, right=of u0, minimum size=38pt, xshift=2.2cm] (u1) {$u_{t}$};

  \node[latent, below=1.4cm of u0, fill, minimum size=38pt] (z0) {$z_{t-1}$};
  \node[latent, below=1.4cm of u1, fill, minimum size=38pt] (z1) {$z_{t}$};

  \coordinate[left=1.2cm of u0] (cstart_u);
  \coordinate[right=1.2cm of u1] (cend_u);
  \coordinate[left=1.2cm of z0] (cstart_z);
  \coordinate[right=1.2cm of z1] (cend_z);

  \edge {z0} {z1}

  \edge {u0} {z0}
  \edge {u1} {z1}

  \draw[dashed] (cstart_z) -- (z0);
  \draw[dashed] (z1) -- (cend_z);

  \draw[draw=none] (z0) -- node[midway, below=0pt, font=\normalsize] {$z_t = {\color{black}\bar{A}_t} z_{t-1} + {\color{black}\bar{B}_t} u_t$} (z1);

  \draw[draw=none] (u1) -- node[midway, right=2pt, font=\normalsize] {\textbf{control inputs}} (z1);
\end{tikzpicture}
}

\newcommand{\tikzOUHybridTall}{
\begin{tikzpicture}[
    every node/.style={font=\small},
    obs/.style={circle, draw, thick, fill=gray!40, minimum size=22pt, inner sep=1pt},
    latent/.style={circle, draw, thick, fill=white, minimum size=22pt, inner sep=1pt}
]

\begin{scope}[local bounding box=timeline]
  \draw[->, thick] (0,0) -- (7.5,0) node[right] {$t$};
  \draw[->, thick] (0,-0.3) -- (0,2.2) node[above] {$x(t)$};

  \fill[blue!8, opacity=0.6] plot[smooth, tension=0.7]
    coordinates {(0.3,2.0) (1.2,1.5) (2.2,1.0) (3.2,1.6) (4.2,1.3) (5.2,1.1) (6.2,1.5) (7.0,1.2)}
    -- plot[smooth, tension=0.7]
    coordinates {(7.0,0.8) (6.2,1.1) (5.2,0.7) (4.2,0.9) (3.2,1.2) (2.2,0.6) (1.2,1.1) (0.3,1.6)}
    -- cycle;

  \draw[thick, blue!70!black] plot[smooth, tension=0.7]
    coordinates {(0.3,1.8) (1.2,1.3) (2.2,0.8) (3.2,1.4) (4.2,1.1) (5.2,0.9) (6.2,1.3) (7.0,1.0)};

  \draw[dashed, gray] (0,1) -- (7.5,1) node[right, font=\scriptsize] {$\mu$};

  \fill[red!80] (1.2,1.3) circle (2.5pt);
  \fill[red!80] (3.2,1.4) circle (2.5pt);
  \fill[red!80] (5.2,0.9) circle (2.5pt);

  \node[below=1pt, font=\scriptsize] at (1.2,0) {$t_{k-1}$};
  \node[below=1pt, font=\scriptsize] at (3.2,0) {$t_k$};
  \node[below=1pt, font=\scriptsize] at (5.2,0) {$t_{k+1}$};

  \draw[thick] (1.2,0.05) -- (1.2,-0.05);
  \draw[thick] (3.2,0.05) -- (3.2,-0.05);
  \draw[thick] (5.2,0.05) -- (5.2,-0.05);

  \node[font=\scriptsize, anchor=north west] at (3.5,1.8) {$\mathrm{d}x = -a(x-\mu)\mathrm{d}t + \sigma\mathrm{d}W$};

  \node[font=\normalsize\itshape, anchor=north west] at (0.1,2.4) {Continuous OU Process (Prior Dynamics)};
\end{scope}

\begin{scope}[yshift=-2.5cm, local bounding box=gm]
  \node[latent] (z0) at (1.2,0) {$z_{k-1}$};
  \node[latent] (z1) at (3.2,0) {$z_k$};
  \node[latent] (z2) at (5.2,0) {$z_{k+1}$};

  \node[obs] (v0) at (1.2,-1.3) {$v_{k-1}$};
  \node[obs] (v1) at (3.2,-1.3) {$v_k$};
  \node[obs] (v2) at (5.2,-1.3) {$v_{k+1}$};

  \draw[->, thick] (z0) -- (z1);
  \draw[->, thick] (z1) -- (z2);

  \draw[->, thick] (z0) -- (v0);
  \draw[->, thick] (z1) -- (v1);
  \draw[->, thick] (z2) -- (v2);

  \node[font=\normalsize] at (0.3,0) {$\cdots$};
  \node[font=\normalsize] at (6.1,0) {$\cdots$};

  \node[above=3pt, font=\scriptsize] at (2.2,0) {$\mathcal{N}(\bar{a}z, {\color{red!80}\bar{p}})$};

  \node[font=\normalsize\itshape, anchor=north west] at (0.1,1.1) {Discretized Gaussian SSM};

  \node[font=\scriptsize, align=left, anchor=west] at (5.3,-0.6) {
    $\bar{a} = e^{-a\Delta t}$\\[2pt]
    ${\color{red!80}\bar{p}} = \tfrac{p^2}{2a}(1-e^{-2a\Delta t})$
  };
\end{scope}

\draw[dotted, gray, thick] (1.2,-0.1) -- (1.2,-2.25);
\draw[dotted, gray, thick] (3.2,-0.1) -- (3.2,-2.25);
\draw[dotted, gray, thick] (5.2,-0.1) -- (5.2,-2.25);

\draw[->, thick, >=stealth, gray!70] (6.5,-0.5) -- node[right, font=\footnotesize] {discretize} (6.5,-2.0);

\begin{scope}[yshift=-6.2cm, local bounding box=mobius]
  \node[font=\normalsize\itshape, anchor=north west] at (0.1,0.95) {Parallel Inference};

  \node[draw, thick, fill=purple!8, minimum width=0.65cm, minimum height=0.65cm,
        font=\scriptsize] (M1) at (1.2,0) {$M_1$};
  \node[font=\small] at (1.9,0) {$\circ$};
  \node[draw, thick, fill=purple!8, minimum width=0.65cm, minimum height=0.65cm,
        font=\scriptsize] (M2) at (2.6,0) {$M_2$};
  \node[font=\small] at (3.3,0) {$\circ$};
  \node[draw, thick, fill=purple!8, minimum width=0.65cm, minimum height=0.65cm,
        font=\scriptsize] (M3) at (4.0,0) {$M_3$};
  \node[font=\small] at (4.7,0) {$\cdots$};
  \node[draw, thick, fill=purple!8, minimum width=0.65cm, minimum height=0.65cm,
        font=\scriptsize] (MT) at (5.4,0) {$M_T$};

  \node[draw, thick, fill=green!10, rounded corners=3pt,
        minimum width=3cm, minimum height=0.5cm,
        font=\scriptsize] (scan) at (3.3,-0.7) {M\"obius Assoc Scan};

  \node[below=0.35cm of scan, font=\scriptsize] (output) {$(\eta_t, \lambda_t)$ transforms via $M_t$};

  \node[font=\normalsize] at (0.3,0) {$\cdots$};
\end{scope}

\draw[->, thick, >=stealth, gray!70] (3.2,-4.1) -- node[right, font=\footnotesize] {information form} (3.2,-5.6);

\end{tikzpicture}
}

\newcommand{\tikzBlockArchWithAttention}{%
\begin{tikzpicture}[remember picture]

\node[inner sep=0pt, anchor=east] (img)
  {\includegraphics[height=7cm]{Figures/middle_diag.pdf}};

\coordinate (kf-top-left) at ($(img.south west)+(-0.15, 4.95)$);
\coordinate (kf-bot-left) at ($(img.south west)+(-0.15, 3.95)$);

\begin{scope}[shift={($(img.west)+(-10.5, 0.5)$)}, local bounding box=matrixarea]

  \node[anchor=north west, font=\small, align=left] (matW) at (0, 2.5) {%
    $\mathbf{W} = \begin{pmatrix}
      k_1 {\color{UncBlue}\Lambda_1^v}
        & 0
        & \cdots
        & 0 \\[4pt]
      {\color{forgetred}f_2}\, k_1 {\color{UncBlue}\Lambda_1^v}
        & k_2 {\color{UncBlue}\Lambda_2^v}
        & \cdots
        & 0 \\[4pt]
      \vdots & \vdots & \ddots & \vdots \\[4pt]
      {\color{forgetred}\textstyle\prod_{s=2}^{T} f_s}\, k_1 {\color{UncBlue}\Lambda_1^v}
        & {\color{forgetred}\textstyle\prod_{s=3}^{T} f_s}\, k_2 {\color{UncBlue}\Lambda_2^v}
        & \cdots
        & k_T {\color{UncBlue}\Lambda_T^v}
    \end{pmatrix}$%
  };

\end{scope}

\begin{scope}[shift={($(matrixarea.south west)+(0, -0.3)$)}, local bounding box=labelarea]

  \node[anchor=north west, font=\small, align=left] (geneq) at (0, 0) {%
    $W_{ij} = \begin{cases}
      \displaystyle {\color{forgetred}\prod_{s=j+1}^{i} f_s} \;\cdot\; k_j \,{\color{UncBlue}\Lambda_j^v}
        & \text{if } i \geq j \\[6pt]
      0 & \text{otherwise}
    \end{cases}$%
  };

  \node[anchor=north west, font=\small, inner sep=0pt] (mseq) at (0, -1.55) {%
    $\mathbf{M}_{\mathrm{seq}}
      \;=\;
      \mathrm{diag}\!\bigl(
        \mathbf{q} \odot {\color{UncBlue}\boldsymbol{\lambda}^{-1}}
      \bigr)
      \,\mathbf{W}$%
    \qquad\scriptsize\textrm{(output readout \& precision scaling)}%
  };

  \node[anchor=north west, font=\small, inner sep=0pt] (veceq) at (0, -2.15) {%
    $\mathbf{y} \;=\;
      \underbrace{%
        \mathbf{M}_{\mathrm{seq}}\,\mathbf{v}
      }_{\text{\scriptsize attention}}
      +\;
      \underbrace{%
        (\mathbf{q} \odot {\color{UncBlue}\boldsymbol{\lambda}^{-1}})
        \odot \mathbf{b} \odot \boldsymbol{\eta}_0%
      }_{\text{\scriptsize init-state}}$%
  };

  \node[anchor=north west, font=\scriptsize, align=left] (legend) at (0, -3.45) {%
    {\color{forgetred}\rule{0.7em}{0.7em}} forget gates $f_s$ \quad
    {\color{UncBlue}\rule{0.7em}{0.7em}} uncertainty $\Lambda^v\!,\;\lambda$ {\small(KLA-only)} \quad
    {\rule{0.7em}{0.7em}} shared with SSM/GLA%
  };

\end{scope}

\coordinate (cone-tr) at (matrixarea.north east);
\coordinate (cone-br) at (matrixarea.south east);

\fill[cyan!4]
    (kf-top-left) -- (cone-tr) -- (cone-br) -- (kf-bot-left) -- cycle;

\draw[dashed, thick, gray!50] (kf-top-left) -- (cone-tr);
\draw[dashed, thick, gray!50] (kf-bot-left) -- (cone-br);

\node[font=\scriptsize, gray!70, rotate=8]
  at ($(kf-bot-left)!0.45!(cone-tr)+(0, -0.15)$) {\textit{attention view}};

\end{tikzpicture}%
}